\documentclass{article}
\pdfoutput=1

\usepackage{graphicx,dsfont,subfig}

\usepackage{multirow}
\usepackage{authblk}
\usepackage{courier}
\usepackage{fullpage}
\usepackage{url}
\usepackage{verbatim} 
\usepackage{datetime}
\usepackage{mathtools}
\usepackage{xcolor}
\usepackage{bm}
\usepackage{algorithm}
\usepackage[]{algpseudocode}
\algdef{SE}[DOWHILE]{Do}{doWhile}{\algorithmicdo}[1]{\algorithmicwhile\ #1}%
\usepackage{booktabs}
\usepackage[utf8]{inputenc}
\usepackage{Definitions}
\usepackage{paralist, tabularx}
\usepackage[numbers,sort&compress,square]{natbib}
\usepackage{stackengine}

\newcommand{\nn}{\nonumber}
\graphicspath{{./FIG/}}

\begin{document}

\setlength{\abovedisplayskip}{2pt}
\setlength{\belowdisplayskip}{2pt}

\newcommand\blfootnote[1]{%
  \begingroup
  \renewcommand\thefootnote{}\footnote{#1}%
  \addtocounter{footnote}{-1}%
  \endgroup
}

\title{Regression Under Human Assistance$^*$}

\author{{Abir De$^{1}$,}
{Nastaran Okati$^{2}$,}
{Paramita Koley$^{3}$,}\\
{Niloy Ganguly$^{3}$} and 
{Manuel Gomez Rodriguez$^{2}$}}
\affil{$^{1}$IIT Bombay, abir@cse.iitb.ac.in\\
$^2$Max Planck Institute for Software Systems, \{nastaran, manuelgr\}@mpi-sws.org \\
$^{3}$IIT Kharagpur, paramita.koley@iitkgp.ac.in, niloy@cse.iitkgp.ac.in}

\date{}

\clubpenalty=10000
\widowpenalty = 10000

\maketitle
\blfootnote{$^{*}$Preliminary version of this work appeared in~\citet{de2020aaai}. Paramita Koley contributed to this work during her internship at the 
Max Planck Institute for Software Systems.} 

\begin{abstract}
Decisions are increasingly taken by both humans and machine learning models.
However, machine learning models are currently trained for full automation---they are not aware that some of the decisions may still be 
taken by humans.
In this paper, we take a first step towards the development of machine learning models that are optimized to operate under different 
automation levels.
%
%
More specifically, we first introduce the problem of ridge regression under human assistance and show that it is NP-hard.
Then, we derive an alternative representation of the corresponding objective function as a diffe\-rence of nondecreasing 
submodular functions. 
Building on this representation, we further show that the objective is nondecreasing and satisfies $\alpha$-submodularity, a recently introduced 
notion of approximate submodularity. 
These pro\-per\-ties allow a simple and efficient greedy algorithm to enjoy approxi\-ma\-tion guarantees at solving the problem.
Ex\-pe\-ri\-ments on synthetic and real-world data from two important applications---medical diagnosis and content moderation---demonstrate that
our algorithm outsources to humans those samples in which the prediction error of the ridge regression model would have been the highest if it 
had to make a prediction, it outperforms several competitive baselines, and its performance is robust with respect to several design choices and 
hyperparameters used in the experiments.

\end{abstract}

\section{Introduction}
\label{sec:introduction}
In a wide range of critical applications, societies rely on the judgement of human experts to take consequential decisions---decisions which 
have significant consequences.
Unfortunately, the timeliness and quality of the de\-cisions are often compromised due to the large number of decisions to be taken and the 
shortage of human experts.
For example, in certain medical specialties, patients in most countries need to wait for months to be diagnosed by a specialist. 
In content moderation, online publishers often stop hosting comments sections because their staff is unable to moderate the 
myriad of comments they receive.
In software development, bugs may be sometimes overlooked by software developers who spend long hours on code reviews for large 
software projects.

In this context, there is a widespread discussion on the possibility of letting machine learning models take decisions in these high-stake tasks, 
where they have matched, or even surpassed, the average performance of human experts~\cite{cheng2015antisocial,pradel2018deepbugs,topol2019high}.
Currently, these mo\-dels are mostly trained for full automation---they assume they will take \emph{all} the decisions. However, their decisions are still 
worse than those by human experts on some instances, where they make far more errors than average~\cite{raghu2019algorithmic}.
%
%
%
%
Motivated by this observation, our goal is to develop machine learning models that are optimized to operate under different automation levels---mo\-dels that are optimized 
to take decisions for a given fraction of the instances and leave the remaining ones to humans.

%
%
%
To this end, we focus on the specific problem of ridge regression and introduce a novel formulation that allows for different automation levels.
Based on this problem formulation, we make the following~contributions:
\begin{itemize}
\item[I.] We show that the problem is NP-hard. This is due to its combinatorial nature---for each potential meta-decision about which instances the machine 
will decide upon, there is an optimal set of parameters for the regression model, however, the meta-decision is also something we seek to optimize.

\item[II.] We derive an alternative representation of the objective function as a difference of nondecreasing submodular functions. This representation 
enables us to use a recent iterative algorithm~\cite{iyer2012algorithms} to solve the problem, however, this algorithm does not enjoy approximation
guarantees.

\item[III.] Building on the above representation, we further show that the objective function is nondecreasing and satisfies $\alpha$-submodularity, a notion of approximate submodularity~\cite{gatmiry2019}.
These properties allow a simple and efficient greedy algorithm (refer to Algorithm~\ref{alg:greedy}) to enjoy approximation guarantees.

\item[IV.] We design a practical framework that uses the solution to the ridge regression problem provided by Algorithm~\ref{alg:greedy} and supervised learning 
to perform predictions about unseen samples at test time (refer to Algorithm~\ref{alg:framework}).
%


\end{itemize}
%
%
%

Finally, we experiment with synthetic and real-world data from two important practical appli\-ca\-tions---medical diagnosis and content moderation. 
Our results show that our algorithm outsources to humans those samples in which the machine error would have been the highest if it had to predict their response 
variables.
This suggests that our algorithm has the ability to learn the underlying relationship between a given sample and its corresponding human and machine error.
Moreover, the results also show that the greedy algorithm outperforms several competitive algorithms, including the iterative algorithm for maximization 
of a difference of submodular functions mentioned above, and its performance is robust with respect to several design choices and hyperparameters 
used in the experiments.
%
%
To facilitate research in this area, we are releasing an open source implementation of our method\footnote{\url{https://github.com/Networks-Learning/regression-under-assistance}}.

Before we proceed further, we would like to acknowledge that our contributions are just a first step towards designing machine learning mo\-dels that are optimized to 
operate under different automation levels. It would be very interesting to extend our work to more sophisticated machine learning models and other 
machine learning tasks (\eg, classification).

\section{Related Work}
\label{sec:related}
The work most closely related to ours is by~\citet{raghu2019algorithmic}, in which a classifier can outsource samples to humans. 
However, in contrast to our work, their classifier is trained to predict the labels of all samples in the training set, as in full automation, and the 
proposed algorithm does not enjoy theoretical guarantees. 
As a result, a natural extension of their algorithm to ridge re\-gre\-ssion achieves a significantly lower performance than ours, as shown in 
Figure~\ref{fig:quantrealone}.

There is a rapidly increasing line of work devoted to designing machine learning models that are able to defer decisions~\cite{bartlett2008classification,cortes2016learning,geifman2018bias,ramaswamy2018consistent,geifman2019selectivenet,raghu2019direct,thulasidasan2019combating,liu2019deep, meresht2020learning, ziyin2020learning}. 
Most previous work focuses on supervised learning and design classifiers that learn to defer either by considering the defer action as an additional label value or by training 
an independent classifier to decide about deferred decisions.
However, there are two fundamental differences between this work and ours. 
First, they do not consider there is a human decision maker, with a human error model, who takes a decision whenever the classifiers defer it.
Second, the classifiers are trained to predict the labels of all samples in the training set, as in full automation.
A very recent notable exception is by~\citet{meresht2020learning}, who consider there is a human decision maker, however, they tackle the problem in a reinforcement learning setting. As a result, their formulation and 
assumptions are fundamentally different and their technical contributions are orthogonal to ours.

Our work is also related to active learning~\cite{cohn1995active,hoi2006batch,sugiyama2006active,willett2006faster,guo2008discriminative,sabato2014active,chen2017active,hashemi2019submodular},
robust linear regression~\cite{wright2010dense,tsakonas2014convergence,bhatia2017consistent,suggala2019adaptive} and robust logistic regression~\cite{feng2014robust}.
In active learning, the goal is to determine which subset of training samples one should label so that a supervised machine learning model, trained on these samples, generalizes well across the entire feature space during test. In other words, the model needs to predict well \emph{any} sample during test time.
In contrast, our trained model only needs to accurately predict samples which are close to the samples assigned to the machine during training time and rely on humans to predict the remaining samples.
In robust linear regression and robust logistic regression, the (implicit) assumption is that a constant fraction of the output variables are corrupted by an unbounded noise. Then, the goal is to find a consistent estimator of 
the model parameters which ignores the samples whose output variables are noisy.
In contrast, in our work, we do not assume any noise model for the output variables but rather a human error per sample and find a estimator of the model parameters that outsources some of the samples to humans.

Our work contributes to an extensive body of work on human-machine collaboration~\citep{macindoe2012pomcop,nikolaidis2015efficient, hadfield2016cooperative, nikolaidis2017mathematical, grover2018learning, wilson2018collaborative, haug2018teaching, tschiatschek2019learner, kamalaruban2019interactive,radanovic2019learning,ghosh2020towards}.
However, rather than developing algorithms that learn to distribute decisions between humans and machines, previous work has predominantly considered settings in which the machine and the human interact 
with each other.

Finally, our work also relates to a recent line of work that combines deep reinforcement learning with opponent modeling to robustly switch between multiple machine policies~\citep{everett2018learning, zheng2018deep}. However, this line 
of work focuses on reinforcement learning, rather than supervised learning, and does not consider the existence of 
 a human policy.

%
%

\section{Problem Formulation}
\label{sec:problem}
In this section, we formally state the problem of ridge regression under human assistance, where some of the predictions can 
be outsourced to humans.

Given a set of training samples $\{ (\xb_i, y_i) \}_{i \in \Vcal}$ and a human error per sample $c(\xb_i, y_i)$, we can outsource 
a subset $\Scal \subseteq \Vcal$ of the training samples to humans, with $| \Scal | \leq n$.
Then, ridge regression under human assis\-tance seeks to minimize the overall training error, including the outsourced samples, 
\ie,
%
%
\begin{equation} \label{eq:optimization-problem}
\underset{\wb,\, \Scal}{\text{minimize}} \quad \ell(\wb, \Scal) \ \text{subject to} \quad |\Scal| \le n,
\end{equation}
with
\begin{equation} \nonumber
 \ell(\wb, \Scal) = \sum_{i \in \Scal} c(\xb_i,y_i)+\sum_{j \in \cs}\left[ (y_j-\xb_j ^\top\wb)^2 +\lambda || \wb ||_2 ^2 \right],
\end{equation}
where the first term accounts for the human error, 
the second term accounts the machine error,
and $\lambda$ is a given regularization parameter for the machine. 
%

Moreover, if we define $\yb=[y_1, y_2,\cdots, y_N]^{\top}$ and $\Xb=[\xb_1, \xb _2,\cdots,\xb _N]$, we can rewrite the above 
objective function as
\begin{equation*}
\ell(\wb,\Scal) =  \sum_{i \in \Scal}  c(\xb_i,y_i) + (\yb_{\cs} - \Xb_{\cs} ^{\top}\wb)^\top(\yb_{\cs}-\Xb_{\cs} ^\top\wb) + \lambda || \wb ||_2 ^2  \cdot |\cs|,
\end{equation*}
where $\yb_{\cs}$ is the subvector of $\yb$ indexed by $\cs$ and $\Xb_{\cs}$ is the submatrix formed by columns of $\Xb$ that are 
indexed by $\cs$.
Then, whenever $\Scal \subset \Vcal$, it readily follows that the optimal parameter $\wb^{*} = \wb^{*}(\Scal)$ is given by
\begin{equation*}
\wb^{*}(\Scal) = \left(\lambda |\cs|\II+\Xb_{\cs} \Xb_{\cs} ^\top\right ) ^{-1} \Xb _{\cs} \yb _{\cs}.
\end{equation*}
If we plug in the above equation into Eq.~\ref{eq:optimization-problem}, we can rewrite the ridge regression problem under
human assistance as a set function maximization problem, \ie,
\begin{equation} \label{eq:optimization-problem-final}
\underset{\Scal}{\text{maximize}} \   -\log \ell(\wb^{*}(\Scal), \Scal) \ \ \text{subject to} \ |\Scal| \le n,
\end{equation}
where
\begin{equation}
\ell(\wb^{*}(\Scal), \Scal) = \left\{
\begin{array}{ll}
\sum_{i\in \Scal}c(\xb_i,y_i) + \yb_{\cs} ^\top\yb_{\cs}-\yb^\top_{\cs} \Xb_{\cs} ^\top\left(\lambda |\cs|\II+\Xb_{\cs} \Xb_{\cs} ^\top\right )^{-1}  \Xb _{\cs} \yb _{\cs} & \hspace*{-0.1cm}\text{if } \Scal \subset \Vcal, \\[0.2cm]
\sum_{i\in \Scal}c(\xb_i,y_i) & \hspace*{-0.1cm} \text{if } \Scal = \Vcal.
\end{array}
\right. \label{eq:ellLongDef}
\end{equation}
Unfortunately, due to its combinatorial nature, the above problem formulation is difficult to solve, as formalized by the following Theorem:
\begin{theorem}
The problem of ridge regression under human assistance defined in Eq.~\ref{eq:optimization-problem} is NP-hard.
\end{theorem}
\begin{proof}
Consider a particular instance of the problem with $c(\xb_i, y_i) = 0$ for all $i \in \Vcal$ and $\lambda = 0$. 
Moreover, assume the response variables $\yb$ are generated as follows:
\begin{equation}
\yb = \Xb^{\top} \wb^{*} + \bb^{*},
\end{equation}
where $\bb^{*}$ is a $n$-sparse vector which takes non-zero values on at most $n$ corrupted samples, and a 
zero elsewhere. 
Then, the problem can be just viewed as a robust least square regression (RLSR) problem~\cite{studer2011recovery}, \ie,
\begin{equation}\nonumber
\underset{\wb,\, \Scal}{\text{minimize}} \ \sum_{i \in \Scal} (y_i-\xb_i ^\top\wb)^2 \ \ \text{subject to} \ |\Scal| = |\Vcal| - n,
\end{equation}
which has been shown to be NP-hard~\cite{bhatia2017consistent}. This concludes the proof.
\end{proof}
However, in the next section, we will show that, perhaps surprisingly, a simple greedy algorithm enjoys approximation guarantees. In the remainder of the paper, to ease the notation, we will use $\ell(\Scal) = \ell(\wb^{*}(\Scal), \Scal)$.

\section{An Algorithm With Approximation Guarantees}
\label{sec:properties}
In this section, we first show that the objective function in Eq.~\ref{eq:optimization-problem-final} can be represented as a difference 
of nondecreasing submodular functions.
Then, we build on this representation to show that the objective function is nondecreasing and satisfies $\alpha$-submodularity~\cite{gatmiry2019},
a recently introduced notion of approximate submodularity. 
Finally, we present an efficient greedy algorithm that, due to the $\alpha$-submodularity of the objective function, enjoys approximation guarantees.

\subsection{Difference of submodular functions}
We first start by rewriting the objective function $\log \ell(\Scal)$ using the following Lemma, which states a well-known property of the 
Schur complement of a block matrix:
\begin{lemma}\label{lem:schur}
Let $\Zb=\left[\begin{array}{cc} \Ab & \Bb\\ \Cb &\Db \end{array}\right]$. If $\Db$ is invertible, then 
$\text{det}(\Zb)=\text{det}(\Db)\cdot \text{det} (\Ab-\Bb \Db^{-1} \Cb)$.
\end{lemma}
More specifically, consider $\Ab = \sum_{i\in \Scal} c(\xb_i,y_i) + \yb_{\cs}^\top\yb_{\cs}$, $\Bb = \Cb^{\top} = \yb^\top_{\cs} \Xb^\top _{\cs}$ and 
$\Db = \lambda|\cs| \II+\Xb_{\cs} \Xb^\top_{\cs}$ in the above lemma. Then, for $\Scal \subset \Vcal$, it readily follows that: 
\begin{equation} \label{eq:difference-equation}
\log \ell(\Scal) = f(\Scal) - g(\Scal)
\end{equation}
where
  \begin{align*}
      f(\Scal) &= \log\det \left[\begin{array}{cc} \sum_{i\in \Scal} c(\xb_i,y_i) + \yb _{\cs} ^\top\yb_{\cs}   & \yb^\top_{\cs} \Xb^\top _{\cs} \\  \Xb_{\cs} \yb_{\cs} &   \lambda|\cs| \II+\Xb_{\cs} \Xb^\top_{\cs}  \end{array}\right] \\
       g(\Scal) &=  \log\det \left[ \lambda |\cs| \II + \Xb_{\cs} \Xb_{\cs}^{\top}  \right].
\end{align*} 
In the above, note that, for $\Scal = \Vcal$, the functions $f$ and $g$ are not defined. As it will become clearer later, for $\Scal = \Vcal$, it will be useful to 
define their values as follows:
\begin{align*}
f(\Vcal) &= \displaystyle{\min_{k_1,k_2 \in \Vcal}} \Big\{ f(\Vcal \cp k_1)+f (\Vcal \cp k_2) -f(\Vcal\cp \{k_1,k_2\})
g(\Vcal\cp k_1)+ g(\Vcal\cp k_2) -g(\Vcal\cp \{k_1,k_2\}) \\
& \qquad \qquad \qquad +  \log \sum_{i\in{\Vcal}} \hci  \Big\}, \\
g(\Vcal) &= f(\Vcal)-\log \sum_{i\in \Vcal} \hci,
\end{align*}
where note that these values also satisfy Eq.~\ref{eq:difference-equation}.
Next, we show that, under mild technical conditions, the above functions are nonincreasing and satisfy a natural diminishing property called 
submodularity\footnote{A set function $f(\cdot)$ is submodular iff it satisfies that $f(\Scal \cup \{k\}) - f(\Scal) \geq f(\Tcal \cup \{k\}) - f(\Tcal)$ for 
all $\Scal \subseteq \Tcal \subset \Vcal$ and $k \in \Vcal$, where $\Vcal$ is the ground set.}.

\begin{theorem}\label{thm:fsubdecreasing}
Assume $\ck \le \gamma y^2 _k$ and $\lambda\ge \frac{\gamma}{1-\gamma} \displaystyle{\max_{i\in\Vcal}} || \xb_i || _2 ^2$ with $0 \leq \gamma < 1$, 
then, $f$ and $g$ are nonincreasing and submodular.
\end{theorem}
\begin{proof}
We start by showing that $f$ is submodular, \ie, $f(\Scal\cup k)-f(\Scal) \ge f(\Tcal\cup k)-f(\Tcal)$ for all $\Scal \subseteq \Tcal \subset \Vcal$ and $k \in \Vcal$.
First, define 
\begin{equation*}
\Mb(\Scal)= \left[\hspace*{-0.1cm}\begin{array}{cc} \yb _{\cs} ^\top\yb_{\cs} +\sum_{i\in \Scal} c(\xb_i,y_i) 
& \hspace*{-0.4cm} \yb^\top_{\cs} \Xb^\top _{\cs} \\  \Xb_{\cs} \yb_{\cs} &  \hspace*{-0.4cm}  \lambda|\cs| \II +\Xb_{\cs} \Xb^\top_{\cs}  \end{array}\hspace*{-0.1cm}\right].
\end{equation*}
and observe that
\begin{equation*} 
\Mb(\Scal\cup k)
=\Mb(\Scal)-\left[\begin{array}{cc} y^2  _{k}   -c(\xb_k,y_k)  &      y_k  \xb^\top _k  \\  \xb_k y_k &  \lambda \II  +\xb_k  \xb^\top _k  
 \end{array}\right]
\end{equation*}
Then, it follows from Proposition~\ref{prop1} (refer to Appendix~\ref{app:auxiliary}) that $\Mb(\Scal)-\Mb(\Scal\cup k) \mge 0$.
%
%
Hence, we have a Cholesky decomposition $\Mb(\Scal)-\Mb(\Scal\cup k)=\Qb_k \Qb_k ^\top$.
%
%
Similarly, we have that $ \Mb(\Tcal\cup k)=\Mb(\Tcal)-\Qb_k \Qb^\top _k$, and hence
 \begin{equation}
\Mb(\Tcal)=\Mb(\Scal)-\sum_{i\in \Tcal \cp \Scal}\Qb_i \Qb^\top _i \label{eq:mstx}
 \end{equation}
 Now, for $\Tcal \cup k \subset \Vcal$, a few steps of calculation shows that:  $f(\Scal\cup k)-f(\Scal) - f(\Tcal\cup k) + f(\Tcal)$ equals to
 \begin{align*}
  \log  \frac{ \det (\II - \Qb^\top _k \Mb^{-1}(\Scal) \Qb_k  )   }{  \det(\II - \Qb^\top _k  \Mb^{-1}(\Tcal)  \Qb_k) }
 \end{align*}
%
%
Moreover,  Eq.~\ref{eq:mstx} indicates that $\Mb(\Scal)~\mge~\Mb(\Tcal)~\mge~0$.
%
 %
%
Therefore, $\Mb^{-1}(\Tcal)\mge \Mb^{-1}(\Scal)$ and hence 
\begin{equation*}
\II - \Qb^\top _k \Mb^{-1}(\Scal) \Qb_k  \mge \II - \Qb^\top _k \Mb^{-1}(\Tcal) \Qb_k.
\end{equation*} 
In addition, we also note that $\Mb(\Tcal)-\Qb_k \Qb^\top _k \mge 0$. This, together with  Lemma~\ref{lem:schur}, we have that 
%
%
$\II - \Qb^\top _k \Mb^{-1}(\Tcal) \Qb_k \mge 0$. Hence, due to Proposition~\ref{prop2} (refer to Appendix~\ref{app:auxiliary}), 
we have 
\begin{equation*}
\det (\II - \Qb^\top _k \Mb^{-1}(\Scal) \Qb_k  ) \ge  \det(\II -  \Qb^\top _k  \Mb^{-1}(\Tcal)  \Qb_k.
\end{equation*} 
Finally, for $\Tcal \cup k = \Vcal$, we have that 
\begin{equation}
  f(\Scal\cup {k_1})-f(\Scal) \ge f(\Vcal\cp k_2) -f(\Vcal\cp \{k_1,k_2\}) \ge f(\Vcal)-f(\Vcal\cp \{k_1\}), \label{eq:ffinalSubmod}
\end{equation}
where the first inequality follows from the proof of submodularity for $\Tcal \cup k \subset \Vcal$ and the second inequality comes from the 
definition of $f(\Scal)$ for $\Scal = \Vcal$. This concludes the proof of submodularity of $f$.

Next, we show that $f$ is nonincreasing.
First, recall that, for $|\Scal| < |\Vcal| - 1$, we have that
 \begin{align}
  f(\Scal\cup k)-f(\Scal)& =\log \frac{ \det (\Mb(\Scal) -\Qb_k \Qb^\top _k  )}{  \det ( \Mb(\Scal)) } 
  \end{align}
Then, note that $\Mb(\Scal) -\Qb_k \Qb^\top _k  \mle \Mb(\Scal)$ and $\Mb(\Scal) -\Qb_k \Qb^\top _k \mge \bm{0}$. 
Hence, using Proposition~\ref{prop2} (refer to Appendix~\ref{app:auxiliary}), it follows that 
\begin{equation*}
\det (\Mb(\Scal) -\Qb_k \Qb^\top _k  )\le \det (\Mb(\Scal)),
\end{equation*}
which proves $f$ is nonincreasing for $|\Scal|<|\Vcal|-1$. 
Finally, for $|\Scal|=|\Vcal|-1$, it readily follows from Eq.~\ref{eq:ffinalSubmod} that
\begin{align}
   f(\Vcal)-f(\Vcal\cp \{k_1\}) \le f(\Vcal\cp k_2) -f(\Vcal\cp \{k_1,k_2\})   \label{eq:ffinalmon}
\end{align}
Now $f(\Vcal\cp k_2) -f(\Vcal \cp \{k_1,k_2\}) \le 0 $ since we have proved that $f(\Scal)$ is nonincreasing 
for $|\Scal|<|\Vcal|-1$. This concludes the proof of monotonicity of $f$.

Proceeding similarly, it can be proven that $g$ is also nondecreasing and submodular.
\end{proof}
We would like to highlight that, in the above, the technical conditions have a natural interpretation. More specifically, the first condition 
is satisfied if the human error is not greater than a fraction $\sqrt{\gamma}$ of the true response variable and 
the second condition is satisfied if the regularization parameter is not \emph{too small}.

In our experiments, the above result will enable us to use a series of recent heuristic iterative algorithms for maximizing the 
difference of submodular functions~\cite{iyer2012algorithms} as baselines. However, these algorithms do not enjoy approxi\-ma\-tion 
guarantees---they only guarantee to monotonically reduce the objective function at every step.
%
%
 
\subsection{Monotonicity}
We first start by analyzing the monotonicity of $\log \ell(\Scal)$ whenever $\Scal = \Vcal \cp k$, for any $k \in \Vcal$ in
the following Lemma: 
\begin{lemma}\label{lem:F_F_basic}
Assume $\ck < \gamma y^2 _k$ and $\lambda >  \frac{\gamma}{1-\gamma} \max_{i\in\Vcal} ||\xb_i||_2 ^2$ with $0 \leq \gamma < 1$. Then, 
it holds that $\log \ell(\Vcal) - \log \ell(\Vcal \cp k) < 0$ for all $k \in \Vcal$.
\end{lemma}
\begin{proof}
By definition, we have that
\begin{align*}
\ell(\wb^{*}(\Vcal), \Vcal) &= \sum_{i\in S}c(\xb_i,y_i)+c(\xb_k,y_k)\\
\ell(\wb^{*}(\Vcal \cp k), \Vcal \cp k) &= y^2_k-y_k ^2 \xb_k ^\top  (\lambda \II +\xkk)^{-1}\xb_k+\sum_{i\in S}c(\xb_i,y_i) 
\end{align*}
Moreover, note that it is enough to prove that $\ell(\wb^{*}(\Vcal), \Vcal) -\ell(\Vcal \cp k), \Vcal \cp k) < 0$, without the
logarithms, to prove the result.
Then, we have that 
\begin{align*}
\ell(\wb^{*}(\Vcal), \Vcal) -\ell(\Vcal \cp k), \Vcal \cp k) &=\ck-y^2_k + y_k ^2 \xb_k ^\top  (\lambda \II +\xkk)^{-1}\xb_k \\
&\overset{(a)}= \ck-y_k ^2 +y^2 _k  \xb_k ^\top \left (\frac{1}{\lambda} \II -  \frac{1}{\lambda^2} \frac{\xkk}{1+\frac{\xb_k ^\top  \xb _k}{\lambda}}  \right )   \xb_k \\
&= \ck-y_k ^2 +y^2 _k  \frac{\xtk}{\lambda}  \left (1 -  \frac{\xtk} {\lambda+\xtk}  \right )   \\
&= \ck-y_k ^2 +y^2 _k   \frac{\xtk}{\lambda+\xtk}      \label{eq:midway} \\
&< y^2_k \left(  \frac{\xtk}{\lambda+\xtk} -(1-\gamma)\right)\\
& \overset{(b)}{<}  y^2_k \left(  \frac{\xtk}{ \frac{\gamma \xtk}{1-\gamma} +\xtk} -(1-\gamma)\right)\\
&=0,
\end{align*}
where equality $(a)$ follows from Lemma~\ref{lem:morrison} (refer to Appendix~\ref{app:auxiliary}) and inequality $(b)$ follows from the lower bound on 
$\lambda$.
\end{proof}
Then, building on the above lemma, we have the following Theorem, which shows that $\log \ell(\Scal)$ is a strictly 
nonincreasing function: 
%
%
%
\begin{theorem}\label{thm:decellgen}
Assume $\ck < \gamma y^2 _k$ and $\lambda >  \frac{\gamma}{1-\gamma} \max_{i\in\Vcal} ||\xb_i||_2 ^2$ with $0 \leq \gamma < 1$, then, 
the function $\log \ell(\Scal)$ is strictly nonincreasing, \ie,
\begin{equation*}
\log \ell(\Scal \cup k) - \log \ell(\Scal) < 0
\end{equation*}
for all $\Scal \in \Vcal$ and $k \in \Vcal$.
\end{theorem}
\begin{proof}
Define $\Lambdab_{0} = \lambda |\cs| \II +\XXS$, $\Lambdab_{1}= \lambda |\cs| \II +\XXS -\lambda \II -\xkk$ and $\Thetab = \lambda \II +\xkk$. Moreover,
note that
 \begin{equation*}
\Lambdab_1=\Lambdab_0-\Thetab \,\,\, \text{and} \,\,\,
\Lambdab^{-1} _1\overset{(a)}{=}  \Lambdab_0 ^{-1} +\underbrace{(\Lambdab_0\Thetab^{-1} \Lambdab_0-\Lambdab_0)}_{\text{Define as }\Omegab}{}^{-1}
 \end{equation*}
where equality $(a)$ follows from Proposition~\ref{prop:XX1} (refer to Appendix~\ref{app:auxiliary}).
Then, it follows that
 \begin{align*}
 \ell(\Scal \cup k) &=  \sum_{i\in S}\hci + \ck + \yys -y^2 _k -(\yxt-y_k \xb_k ^T) \Lambdab_1 ^{-1} (\xy -y_k \xb_k  )\\
  &\overset{(a)}{=}  \sum_{i\in S}\hci + \ck + \yys -y^2 _k  -\yxt  \lz^{-1} \xy \nonumber  -\yxt \Omegab^{-1} \xy   \\
  &\qquad\qquad+2  y _k \yxt \Lambdab_1 ^{-1} \xb_k ^\top -y_k ^2 \xb ^\top _k \Lambdab_1 ^{-1} \xb _k\\
  %
  &=  \ell(\Scal)+\ck - y^2 _k 
  - \left[ \begin{array}{cc} \yxt & y_k \xb^\top _k  \end{array}\right]
  \left[ \begin{array}{cc} \Omegab^{-1} &  -\lone^{-1} \\  -\lone^{-1}   &  \Lambdab_1 ^{-1} \Omegab \Lambdab_1 ^{-1}  \end{array}\right] 
  \left[ \begin{array}{c} \xy \\ y_k \xb_k  \end{array}\right]\nonumber\\
  &\qquad\qquad -y^2 _k \xb ^\top _k (\Lambdab_1 ^{-1} -\Lambdab_1 ^{-1} \Omegab \Lambdab_1 ^{-1})   \xb _k,
\end{align*}  
where equality $(a)$ follows from Proposition~\ref{prop:XX1} (refer to Appendix~\ref{app:auxiliary}). Finally, we can upper bound the right hand side of the above equation 
as follows:
\begin{align*}  
 \ell(\Scal \cup k) &\overset{(a)}{\leq} \ell(\Scal)+\ck - y^2 _k-y^2 _k \xb ^\top _k (\Lambdab_1 ^{-1} -\Lambdab_1 ^{-1} \Omegab \Lambdab_1 ^{-1})   \xb _k\\
  &\overset{(b)}{=} \ell(\Scal)+\ck - y^2 _k +y^2 _k \xb_k ^\top (\lambdab \II +\xkk)^{-1} \xb_k  \\ 
  &\overset{(c)}{=} \ell(\Scal)+ \ell(\Vcal) - \ell(\Vcal\cp k),
 \end{align*}
 where inequality $(a)$ uses that $\left[ \begin{array}{cc} \Omegab^{-1} &  -\lone^{-1} \\  -\lone^{-1}   &   \Lambdab_1 ^{-1} \Omegab \Lambdab_1 ^{-1}  \end{array}\right]\mge 0$,
 %
 equality $(b)$ follows from the following observation: 
 \begin{align*}
(\Lambdab_1 ^{-1} -\Lambdab_1 ^{-1} \Omegab \Lambdab_1 ^{-1}) 
&= (\Lambdab_0 ^{-1} + \Omegab^{-1})- (\Lambdab_0 ^{-1} + \Omegab^{-1}) \Omegab  (\Lambdab_0 ^{-1} + \Omegab^{-1}) = -\lz ^{-1} \Omegab \lz^{-1} -\lz^{-1}\\
&=-\Lambdab_0 ^{-1}  (\lz \Thetab^{-1} \lz -\lz) \lz^{-1} -\lz^{-1} =-\Thetab^{-1},
\end{align*}
and inequality $(c)$ follows from Lemma~\ref{lem:F_F_basic}.
\end{proof}
Finally, note that the above result does \emph{not} imply that the human error $\ck$ is always smaller than the machine error $(y_k -  \xb^{\top}_k \wb^{*}(k))^2$, 
where $\wb^{*}(k)$ is optimal parameter for $\Scal = \{ k \}$, as formalized by the following Proposition: 
%
%
%
%
\begin{proposition}\label{prop:F_F_strict}
Assume $\rho^2 y_k  ^2 < \ck < \gamma y^2 _k$ and $\frac{\gamma}{1-\gamma} \max_{i\in\Vcal} ||\xb_i||_2 ^2 < \lambda  <  \frac{\rho}{1-\rho}\max_{i\in\Vcal} ||\xb_i||_2 ^2$ with 
$ \gamma < \rho < \sqrt{\gamma}$ and $0 \leq \gamma < 1$, then, it holds that 
\begin{equation*}
\ck > (y_k-\xb_k ^\top \wb^{*}(k))^2.
\end{equation*}
\end{proposition}
\begin{proof}
First, we have that
\begin{align*}
 (y_k-\xb_k ^\top \wb^{*}(k))^2& =  (y_k-\xb_k ^\top \wb^{*}(k))^2 +\lambda || \wb^{*}(k)||^2  -\lambda   || \wb^{*}(k)||^2   \\
 & =    y^2_k -  y_k ^2 \xb_k ^\top  (\lambda \II +\xkk)^{-1}\xb_k -  \lambda y_k ^2 \xb^\top _k  (\lambda \II + \xkk)^{-2} \xb_k \\
 & = y_k ^2 - y^2 _k   \frac{\xtk}{\lambda+\xtk} -   \lambda y_k ^2  \xb^\top _k  (\lambda \II + \xkk)^{-2} \xb_k \\ 
  & \overset{(a)}{=}    \frac{\lambda y_k ^2   }{\lambda+\xtk} -   \lambda y_k ^2  \xb^\top _k   \left (\frac{1}{\lambda} \II -  \frac{1}{\lambda^2} \frac{\xkk}{1+\frac{\xb_k ^\top  \xb _k}{\lambda}}  \right )^2   \xb_k \\
  & =    \frac{\lambda y_k ^2   }{\lambda+\xtk} - \frac{y_k ^2}{\lambda}\xb^\top _k     \left ( \II -    \frac{\xkk}{ \lambda+ \xb_k ^\top  \xb _k }  \right )^2 \xb_k
  \end{align*}
  \begin{align*}
  & =    \frac{\lambda y_k ^2  }{\lambda+\xtk}  - \frac{y_k ^2}{\lambda}\xb^\top _k     \left ( \II -    2\frac{\xkk}{ \lambda+ \xb_k ^\top  \xb _k }  +  \frac{\xkk \xkk }{ (\lambda+ \xb_k ^\top  \xb _k )^2 }  \right ) \xb_k\\
    & =    \frac{\lambda y_k ^2   }{\lambda+\xtk}  - \frac{y_k ^2}{\lambda}\xtk    \left ( 1 -    2\frac{\xtk}{ \lambda+ \xb_k ^\top  \xb _k }  + \left( \frac{\xtk }{ \lambda+ \xb_k ^\top  \xb _k }  \right)^2  \right ) \\
        & =   \frac{\lambda y_k ^2   }{\lambda+\xtk}  -     \frac{ y_k ^2 \lambda \xtk} { (\lambda+ \xtk)^2} 
         = y^2_k \left(   \frac{\lambda}{\lambda+\xtk} \right)^2\\
         &\overset{(b)}{\le} \rho^2 y^2 _k,
\end{align*}
%
where equality $(a)$ follows from Lemma~\ref{lem:morrison} (refer to Appendix~\ref{app:auxiliary}) and inequality $(b)$ follows from the assumption 
$\lambda \le \frac{\rho}{1-\rho} \max_{i} ||\xb_i ^2|| _2 ^2$.
Then, since $\ck >\rho^2 y^2 _k $, we can conclude that $\ck> (y_k-\xb_k ^\top \wb^{*}(k))^2$.
\end{proof}

\subsection{$\alpha$-submodularity}
Given the above results, we are now ready to present and prove our main result, which characterizes the objective function
of the optimization problem defined in Eq.~\ref{eq:optimization-problem-final}:
%
%
%
\begin{theorem}\label{th:g_xi_submod}
Assume $\ck~<~\gamma~y^2 _k,~\lambda~>~\frac{\gamma}{1-\gamma}~\displaystyle{\max_{i\in\Vcal}}~||\xb_i||_2 ^2$ with $0 \leq \gamma < 1$, and 
$\sum_{i\in \Vcal} \hci \ge 1$\footnote{Note that we can always rescale the data to satisfy this last condition.}. 
Then, the function $-\log \ell(\Scal)$ is a nondecreasing $\alpha$-submodular function\footnote{A function $f(\cdot)$ is $\alpha$-submodular~\cite{gatmiry2019} iff it satisfies that $f(\Scal \cup \{k\}) - f(\Scal) \geq (1-\alpha) \left[ f(\Tcal \cup \{k\}) - f(\Tcal) \right]$ for all $\Scal \subseteq \Tcal \subset \Vcal$ and $k \in \Vcal$, where $\Vcal$ is the ground set and $\alpha$ is the generalized curvature~\cite{lehmann2006combinatorial,bogunovic2018robust,hashemi2019submodular}.} and the parameter $\alpha$ satisfies that
\begin{align} \label{eq:xistar}
 \alpha \le \alpha^* =1-\min \Big\{ &\frac{  (1-\kappa_\ell) \log \ell(\Vcal) }{ \max_{k_1,k_2} f(\Vcal\cp \{k_1,k_2\}) - f(\Vcal \cp \{k_1\})}, 
 \frac{  (1-\kappa_\ell) \log \ell(\Vcal) }{ \max_{k} \log \ell(\Vcal\cp k) - \log \ell(\Vcal)}\Big\}
\end{align}
%
%
\begin{equation*}
\text{with, }\kappa_{\ell}= \frac{\log \left[ \ell(\emptyset) - \min_k ( \ell(\Vcal \cp k) - \ell(\Vcal))\right] }{\log \ell(\emptyset)}
\end{equation*}
\end{theorem}
\begin{proof}
Using that $\sum_{i\in V} \hci >1$ and the function $\ell$ is nonincreasing, we can conclude 
that $1 < \ell(\Vcal) < \ell(\Scal)$.
Then, it readily follows from the proof of Theorem~\ref{thm:decellgen} that
\begin{align}
1< \ell(\Scal\cup k)< &\ell(\Scal)-(\ell(\Vcal\cp k)-\ell(\Vcal))
\end{align}
Hence we have,
\begin{align}
 \frac{\log \ell(\Scal\cup k)}{ \log \ell(\Scal)} &\le \frac{\log \Big(\ell(\Scal) -(\ell(\Vcal\cp k) -\ell(\Vcal))  \Big)}{ \log \ell(\Scal)} 
\overset{(a)}{\le} \left. \frac{\log \Big(\ell_{\max} -(\ell(\Vcal\cp k) -\ell(\Vcal))  \Big)}{ \log \ell(\Scal)}\right|_{\Scal^* = \argmax_{\Scal\subseteq \Vcal} \ell(\Scal) } \nonumber \\ 
&\overset{(b)}{=}  \frac{\log \Big(\ell(\emptyset) -(\ell(\Vcal\cp k) -\ell(\Vcal))  \Big)}{ \log \ell(\emptyset)} 
 \leq \kappa_{\ell} \label{eq:Fineq} 
\end{align}
where inequality $(a)$ follows from Proposition~\ref{prop:tx} (refer to Appendix~\ref{app:auxiliary}) and 
equality $(b)$ follows from Theorem~\ref{thm:decellgen}, which implies that $\ell_{\max}=\ell(\emptyset)$. 
%
%
%
Then, we have that
\begin{align}
1-\alpha& =\min_{k,\, \Scal\subseteq \Tcal \subseteq \Vcal} \frac{\log \ell(\Scal)-\log \ell(\Scal\cup k)}{\log \ell(\Tcal) - \log \ell(\Tcal\cup k)}  \nn\\ 
& \ge \min \Big\{ \min_{k, \Scal \subseteq \Tcal : |\Tcal| \le |\Vcal|-2}   \frac{\log \ell(\Scal)-\log \ell(\Scal\cup k)}{\log \ell(\Tcal)-\log \ell(\Tcal\cup k)}, 
\min_{\Scal,k}\frac{\log \ell(\Scal)-\log \ell(\Scal\cup k)}{\log \ell(\Vcal \cp k)-\log \ell(\Vcal)} \Big\} \label{eq:xitwob}
\end{align}
Next, we bound the first term as follows:
\begin{align*}
   \min_{k, \Scal \subseteq \Tcal : |\Tcal| \le |\Vcal|-2}   \frac{\log \ell(\Scal)-\log \ell(\Scal\cup k)}{\log \ell(\Tcal)-\log \ell(\Tcal\cup k)}
  &\ \overset{(a)}{\ge} \min_{k, S\subseteq T: |\Tcal| \le |\Vcal|-2}   \frac{(1-\kappa_{\ell})\log \ell(\Scal)}{\log \ell(\Tcal)-\log \ell(\Tcal\cup k)} \nn \\
 &\ \overset{(b)}{\ge} \min_{k,  |\Tcal| \le |\Vcal|-2}   \frac{(1-\kappa_{\ell}) \log \ell(\Vcal)}{\log \ell(\Tcal)-\log \ell(\Tcal\cup k)}\nn \\
&\ =  \min_{k,  |\Tcal| \le |\Vcal|-2}   \frac{(1-\kappa_{\ell}) \log \ell(\Vcal)}{f(\Tcal)-f(\Tcal\cup k) - (g(\Tcal)-g(\Tcal\cup k))}\nn\\
&\ \overset{(c)}{\ge} \min_{k,  |\Tcal| \le |\Vcal|-2}   \frac{(1-\kappa_{\ell})\log \ell(\Vcal)}{f(\Tcal)-f(\Tcal\cup k) }\nn\\
&\ \overset{(d)}{\ge}  \frac{(1-\kappa_l)\log \ell(V)}{\max_{k_1,k_2}(f(\Vcal\cp \{k_1,k_2\})-f(\Vcal\cp k_1)) },
 \end{align*}
where inequality (a) follows from Eq.~\ref{eq:Fineq}, inequality (b) follows from the monotonicity of $\log \ell(\Scal)$, and inequalities (c) and (d) follows from 
Theorem~\ref{thm:fsubdecreasing}.
%
%
Finally, we use the monotonicity of $\log \ell(\Scal)$ and Eq.~\ref{eq:Fineq} to bound the second term in Eq.~\ref{eq:xitwob} 
as follows:
\begin{align*}
\min_{\Scal,\, k}& \frac{\log \ell(\Scal)-\log \ell(\Scal\cup k)}{\log \ell(\Vcal\cp k)- \log \ell(\Vcal)}\nn \geq
\frac{(1-\kappa_{\ell}) \log \ell(\Vcal)}{\max_{k} \log \ell(\Vcal \cp k) -\log \ell(\Vcal)},
\end{align*}
which concludes the proof.
\end{proof}
\begin{algorithm}[t] 
\small
\renewcommand{\algorithmicrequire}{\textbf{Input:}}
\renewcommand{\algorithmicensure}{\textbf{Output:}}
\caption{Greedy algorithm}\label{alg:greedy}
\begin{algorithmic}[1]
\Require Ground set $\Vcal$, set of training samples $\{ (\xb_i, y_i) \}_{i \in \Vcal}$, parameters $n$ and $\lambda$.
\Ensure Solution $(\wb^{*}(\Scal), \Scal)$
\State $\Scal \leftarrow \varnothing$
\While{$|\Scal| < n$}
	\State \mbox{\% Find best sample}
	\State $k^* \leftarrow \text{argmax}_{k \in \Vcal \backslash \Scal} -\log \ell(\Scal \cup k) + \log \ell(\Scal)$ 
	\State \mbox{\% Sample is outsourced to humans}
	\State $\Scal \leftarrow \Scal \cup \{ k^* \}$
\EndWhile \\
\Return $(\wb^{*}(\Scal), \Scal)$
\end{algorithmic}
\end{algorithm}

\subsection{A greedy algorithm}
The greedy algorithm proceeds iteratively and, at each step, it assigns to the humans the sample $(\xb_k, y_k)$ that provides 
the highest marginal gain among the set of samples $\sc$ which are currently assigned to the machine. Algorithm 1 summarizes the greedy algorithm.

Since the objective function in Eq.~\ref{eq:optimization-problem-final} is $\alpha$-submodular, it readily follows from Theorem~9 in~\citet{gatmiry2019} 
that the above greedy algorithm enjoys an approximation guarantee. More specifically, we have the following Theorem:
\begin{theorem} \label{thm:approximation-guarantee}
The greedy algorithm returns a set $\Scal$ such that $-\log \ell(\Scal) \geq (1+1/(1-\alpha))^{-1}OPT$, where $OPT$ is the optimal value and $\alpha \leq \alpha^{*}$ with $\alpha^{*}$ defined 
in Eq.~\ref{eq:xistar}.
\end{theorem}
In the above, note that, due to Theorem~\ref{thm:decellgen}, the actual (regularized) loss function is strictly nonincreasing and thus the greedy algorithm 
always goes until $|S| = n$, however, the overall accuracy may be higher for some values of $|S| < n$ as shown in Figure~\ref{fig:realErrorWithC}. 
Next, we provide a formal analysis of the time complexity of Algorithm~\ref{alg:greedy} in our particular problem setting.
\begin{proposition}
%
The computational cost of Algorithm~\ref{alg:greedy} is $O(n d^2 |\Vcal|^2)$, where $n$ is the number of samples outsourced to humans, 
$d$ is the dimensionality of the feature vectors $\xb$ and $|\Vcal|$ is the number of training samples.
\end{proposition}
\begin{proof}
At each step of the $n$ steps of the greedy algorithm, we need to solve $|\Vcal|$ regularized least square problems. Solving each of these problems take 
each $O(|\Vcal| d^2)$ using Cholesky decomposition. Therefore, the overall complexity of the greedy algorithm is $O(n |\Vcal|^2 d^2)$.

\end{proof}

In the next section, we will demonstrate that, in addition to enjoying the above approximation guarantees, the greedy algorithm performs better in practice 
than several competitive baselines.

\section{Practical Deployment} 
\label{sec:implementation}
%
In practice, to deploy a model $(w^{*}(\Scal), \Scal)$ trained using a dataset $\{ (\xb_i, y_i) \}_{i \in \Vcal}$, we need to be able to decide whether to outsource 
any (unseen) test sample $\xb \neq \xb_i$ for all $i \in \Vcal$.
To this end, we train an additional model $h_{\theta}(\xb)$ to decide which samples to outsource to a human using the training set $\{ (\xb_i, d_i) \}_{i \in \Vcal}$, 
where $d_i = +1$ if $i \in \Scal$ and $d_i = -1$ otherwise. 
In our experiments, we consider three different types of models:

\vspace{1mm}
\noindent --- \emph{Nearest neighbours (NN)}: it outsources a sample $\xb$ to a human if the nearest neighbor in the set $\Vcal$ belongs to $\Scal$ and
and pass it on to the machine otherwise, \ie, $h_{\theta}(\xb) = +1$ if $\argmin_{i \in \Vcal} || \xb_i - \xb || \in \Scal$ and $h_{\theta}(\xb) = -1$ otherwise.

\noindent --- \emph{Logistic regression (LR)}: it outsources a sample $\xb$ to a human using a logistic regression classifier trained using maximum likelihood, \ie, 
$h_{\theta}(\xb) = +1$ if $\frac{1}{1 + \exp(\theta^{T} \xb)} > c$ and $h_{\theta}(\xb) = -1$ otherwise, where $c \in (0, 1)$.

\noindent --- \emph{Multilayer perceptron (MLP)}: it outsources a sample $\xb$ to a human using a multilayer perceptron, \ie, $h_\theta(\xb) = +1$ if $f_{\theta}(\xb) > c$ and 
$h_\theta(\xb) = -1$ otherwise, where $f_{\theta}(\xb)$ is a multilayer perceptron\footnote{\scriptsize https://scikit-learn.org/stable/modules/generated/sklearn.neural$\_$network.MLPClassifier.html.} and $c \in (0, 1)$.
\vspace{1mm}

\begin{algorithm}[t]
\small
\renewcommand{\algorithmicrequire}{\textbf{Input:}}
\renewcommand{\algorithmicensure}{\textbf{Output:}}
\caption{Practical framework for ridge regression under human assistance}\label{alg:framework}
\begin{algorithmic}[1]
\Require Training samples $\Dcal$, trained model $(\wb^{*}(\Scal), \Scal)$, and (unlabeled) test samples $\Xcal$.
\Ensure Set of predictions $\Dcal'$
\State $\Dcal' \leftarrow \emptyset$
\State $h_{\theta}(\xb) \leftarrow \textsc{TrainClassifier}(\Dcal, \Scal)$
\For{$\xb \in \Xcal$}
	\If{$h_{\theta}(\xb) > 0$}
		\State  $\hat{y} \leftarrow \textsc{Outsource}(\xb)$
	\Else
		\State $\hat{y} \leftarrow \xb^{T} \wb^{*}$
	\EndIf
	\State $\Dcal' \leftarrow \Dcal' \cup \{ (\xb, \hat{y}) \}$
\EndFor \\
\Return $\Scal$
\end{algorithmic}
\end{algorithm}

\noindent Algorithm~\ref{alg:framework} summarizes the above procedure. Within the algorithm, $\textsc{TrainClassifier}(\Dcal, \Scal)$ trains the
additional model $h_{\theta}(\xb)$ and $\textsc{Outsource}(\xb)$ outsources a sample $\xb$ to a human and returns the prediction made by the 
human.
Finally, note that, as long as the feature distribution does not change during test, the above pro\-ce\-du\-re guarantees that the fraction of samples outsourced to humans during 
training and test time will be similar.
More specifically, the following proposition formalizes this result for the case of nearest neighbours:
\begin{proposition}\label{prop:tt}
Let $\{(\xb_i, y_i)\}_{i\in \Vcal}$ be a set of training samples, $\{\xb' _j\}_{j\in \Vcal'}$ a set of (unlabeled) test samples, and $n$ and $n'$ the number of training and test 
samples outsourced to humans, respectively.
If $\xb_i, \xb_j \sim \PP(\xb)$ for all $i \in \Vcal, j \in \Vcal'$, then, it holds that $\EE[n'] / |\Vcal'|= n / |\Vcal|$.
\end{proposition}
\begin{proof}
Let the feature space be $\Fcal$. Moreover we denote that
$\Xcal=\{\xb_i\}_{i\in\Vcal}$ and $\Xcal'=\{\xb' _j\}_{j\in\Vcal'}$.
Then we denote that
 \begin{align}
H_{\xb_i}=\cap_{k\in\Vcal}\{\xb \in\Fcal | || \xb_i-\xb||\le || \xb_k-\xb||\}.
 \end{align}
%
 Hence, the set of test samples,  which are nearest to $\xb_i$, is denoted as $\Xcal'\cap H_{\xb_i}$.
 Since the features in $\Xcal$ and $\Xcal'$ are  i.i.d random variables, $|\Xcal'\cap H_{\xb_i}|$ are also i.i.d random variables
 for different realizations of $\Xcal$ and $\Xcal'$.  Let us define $\vartheta=\EE[|\Xcal'\cap H_{\xb_i}|]$.
 Hence we have,
$\EE[n'] =\sum_{i\in\Scal}\EE[|\Xcal'\cap H_{\xb_i}|] = n\vartheta$ and
$\EE[|\Vcal'|-n']  = (|\Vcal|-n)\vartheta$,
which leads to the required result.
\end{proof}
%

%
\section{Experiments on Synthetic Data}
\label{sec:synthetic}
%
In this section, we experiment with a variety of~syn\-the\-tic exam\-ples. First, we look into the solution $(\wb^{*}(\Scal^{*}), \Scal^{*})$ 
provided by the greedy algorithm.
%
%
Then, we compare the performance of the greedy algorithm with several competitive baselines.
Finally, we investigate how the performance of the greedy algorithm varies with respect to the amount of human error.
%
%
 \begin{figure*}[!t]
\centering
\hspace*{1.7cm}{\includegraphics[width=0.65\textwidth]{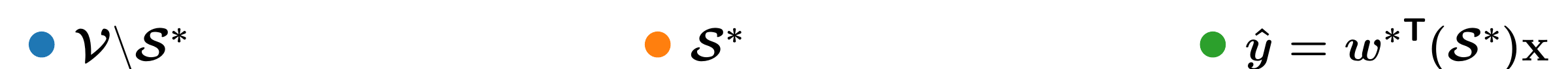}\label{fig:ltr}\vspace{0.1cm}}\\
\subfloat[$n=80$, Logistic]{\includegraphics[width=0.24\textwidth]{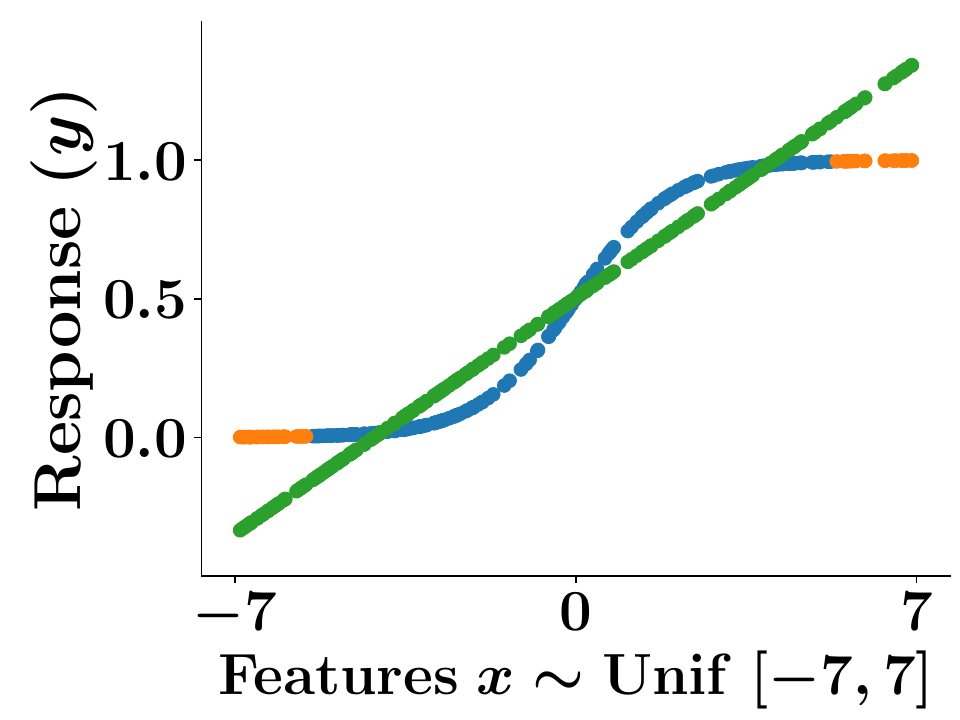}\label{fig:ltr}}\hspace*{0.2cm}
\subfloat[$n=160$, Logistic]{\includegraphics[width=0.24\textwidth]{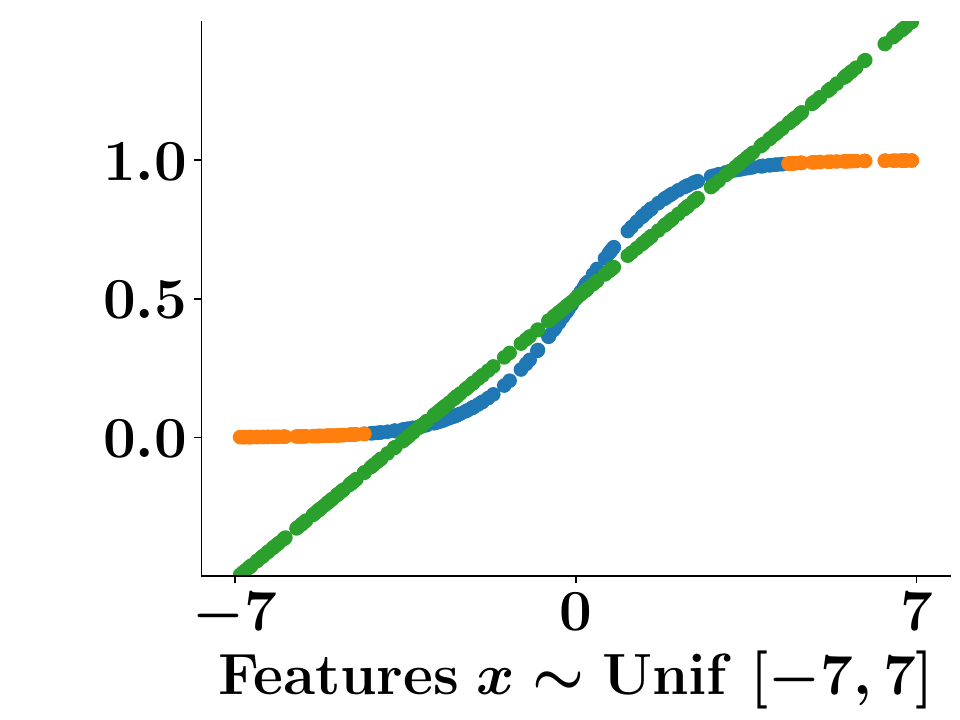}\label{fig:ltr}}\hspace*{0.2cm}
\subfloat[$n=240$, Logistic]{\includegraphics[width=0.24\textwidth]{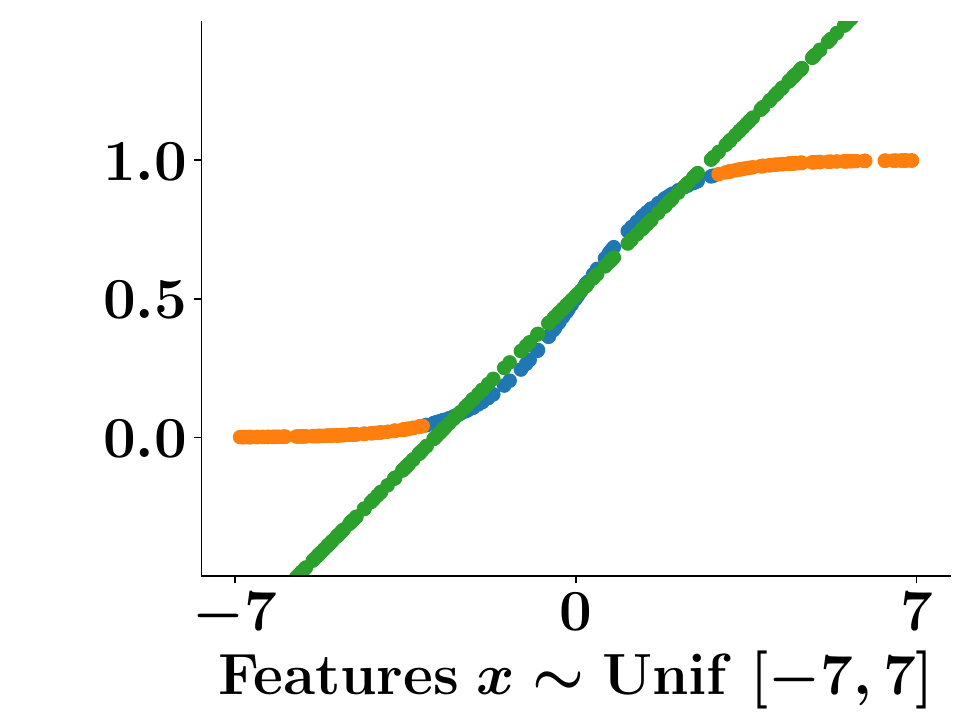}\label{fig:ltr}}\hspace*{0.2cm}
\subfloat[$n=320$, Logistic]{\includegraphics[width=0.24\textwidth]{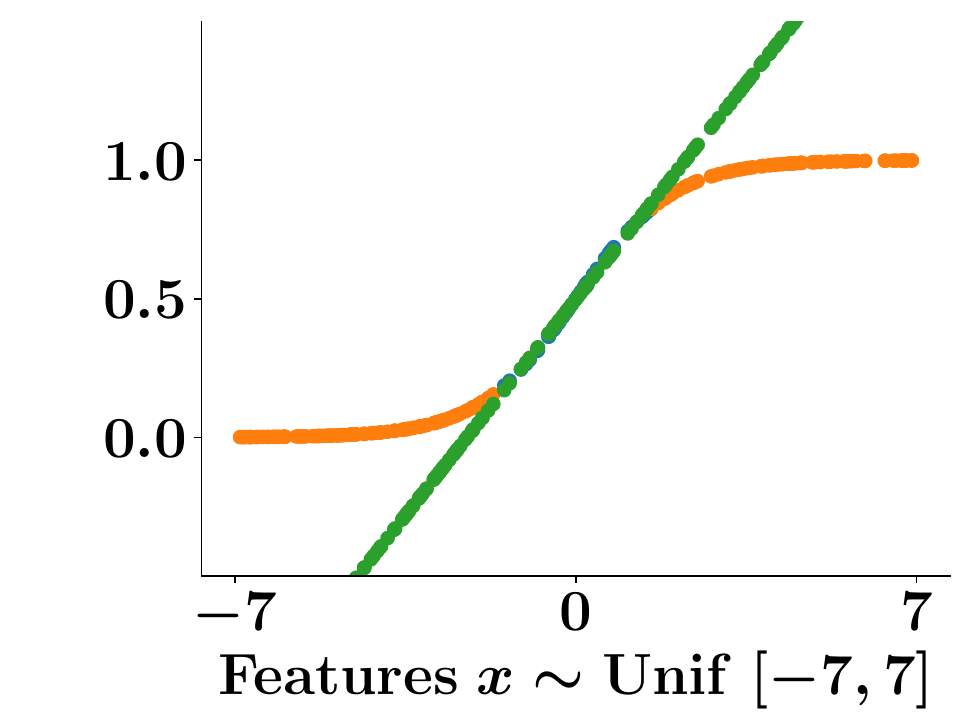}\label{fig:ltr}}\\
\subfloat[$n=80$, Gaussian]{\includegraphics[width=0.24\textwidth]{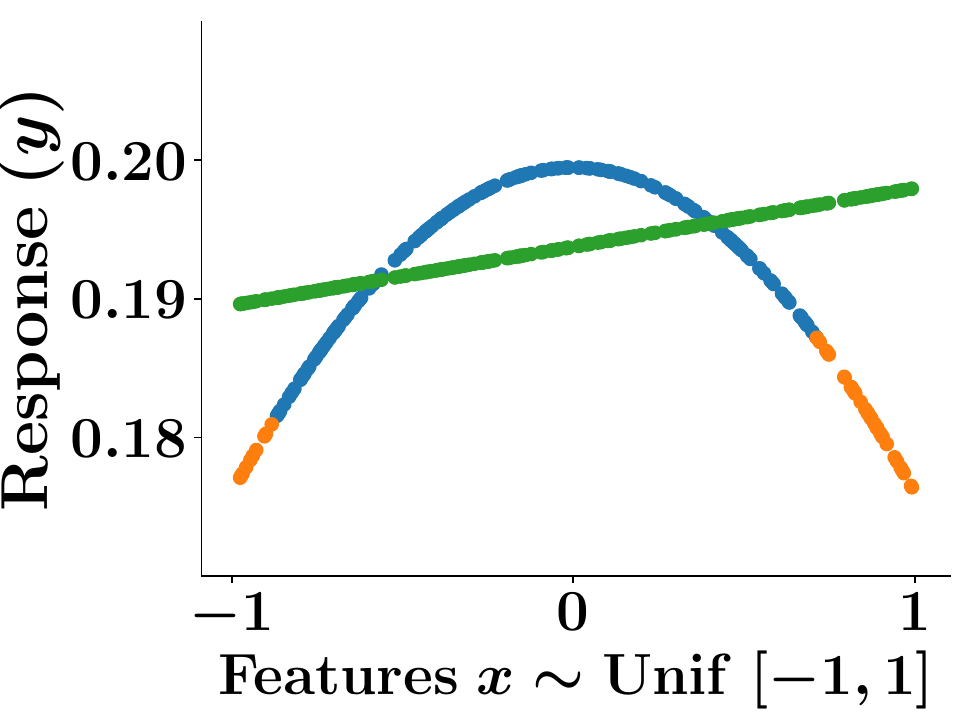}\label{fig:ltr}}\hspace*{0.2cm}
\subfloat[$n=160$, Gaussian]{\includegraphics[width=0.24\textwidth]{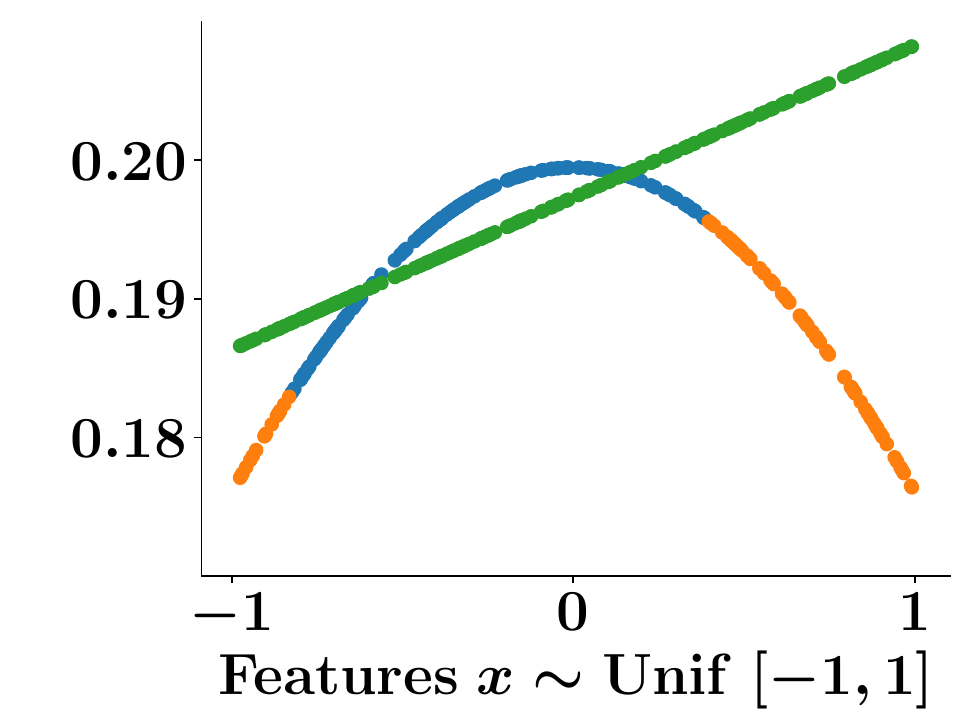}\label{fig:ltr}}\hspace*{0.2cm}
\subfloat[$n=240$, Gaussian]{\includegraphics[width=0.24\textwidth]{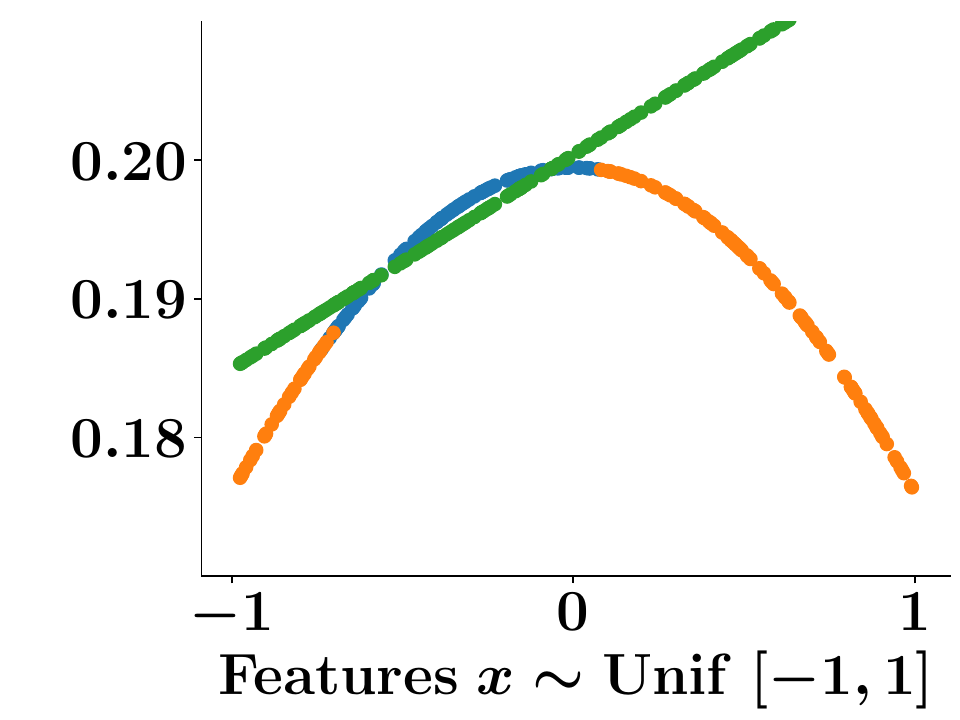}\label{fig:ltr}}\hspace*{0.2cm}
\subfloat[$n=320$, Gaussian]{\includegraphics[width=0.24\textwidth]{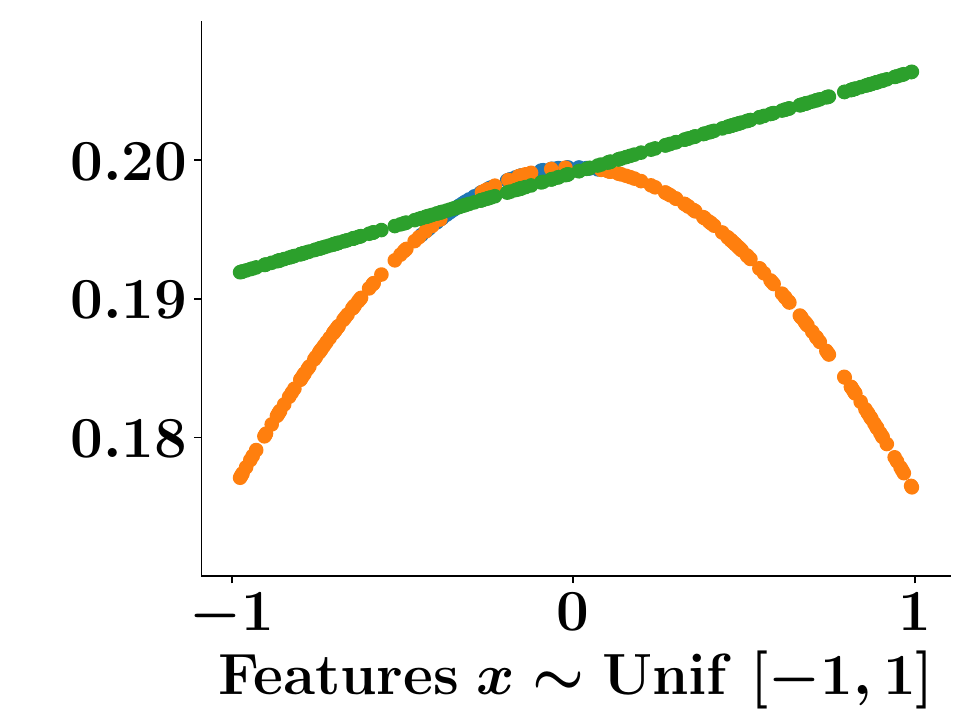}\label{fig:ltr}}
\caption{Solution $(\wb^{*}(\Scal^{*}), \Scal^{*})$ provided by our greedy algorithm for a gaussian and logistic response 
variable distribution and different number of outsourced samples $n$. In all cases, we used $d = 1$ and $\sigma_2=0.001$.
For the logistic distribution, as $n$ increases, the greedy algorithm let the machine to focus on the samples 
where the relationship between features and the response variables is more linear and outsource the remaining points to humans.
For the gaussian distribution, as $n$ increases, the greedy algorithm outsources samples on the tails of the 
distribution to humans.
}
 \label{fig:qualSyn}
\end{figure*}

\subsection{Experimental setup} For each sample $(\xb, y)$, we first generate each dimension of the feature vector $\xb \in \RR^{d}$ from a uniform distribution,
\ie, $x_i \sim U(-a,a)$, where $a$ is a given parameter, and then sample the response variable $y$ from a Gaussian distribution $\Ncal({\mathbf{1}^\top \xb}/d, \sigma_1 ^2)$ 
or a logistic distribution $1/(1+\exp(-\mathbf{1}^{\top} \xb/d))$. 

%
Moreover, we sample the associated human error from a Gaussian distribution, \ie, $c(\xb,y)\sim \Ncal(0,\sigma_2 ^2)$.
In each experiment, we use $|\Vcal| = 400$ training samples and we compare the performance of the greedy algorithm with four competitive baselines on
a held-out set of $100$ test samples:

\vspace{1mm}
\noindent --- An iterative heuristic algorithm (DS) for ma\-xi\-mi\-zing the difference of submodular functions by~\citet{iyer2012algorithms}.

\noindent --- A greedy algorithm (Distorted greedy) for maximizing $\gamma$-weakly submodular functions by~\citet{harshaw2019submodular}\footnote{Note 
that any $\alpha$-submodular function is $\gamma$-weakly submodular~\cite{gatmiry2019}.}.

\noindent --- A natural extension of the algorithm (Triage) by~\citet{raghu2019algorithmic}, originally developed for classification under human assistance,
which first solve the standard ridge regression problem for the entire training set and then outsources to humans the top $n$ samples sorted in decreasing 
order of the difference between machine and human error.
%


\noindent --- A natural extension of the iterative algorithm for robust least square regression by~\citet{bhatia2017consistent}, which includes L2 regularization (CRR). 
Within the algorithm, the $||\mathbf{b}||_0 = n$ samples identified as \emph{corrupt} are outsourced to humans.

\vspace{1mm}

 \begin{figure}[!t]
\centering
\hspace*{0.5cm}{\includegraphics[width=0.55\textwidth]{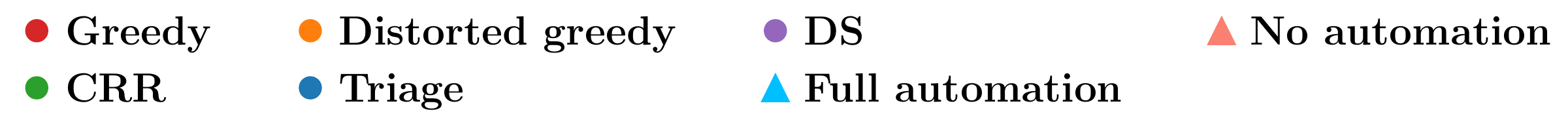}\label{fig:ltr}\vspace{-0.2cm}}\\
\subfloat[Gaussian]{\includegraphics[width=0.28\textwidth]{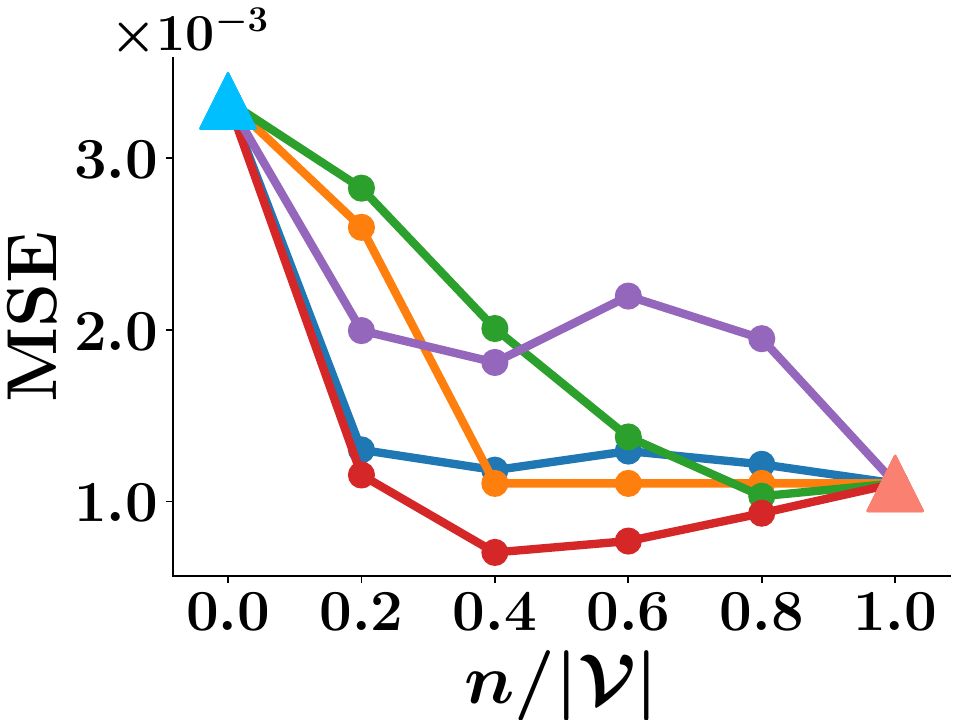}\label{fig:gtest}}\hspace*{0.5cm}
\subfloat[Logistic]{\includegraphics[width=0.28\textwidth]{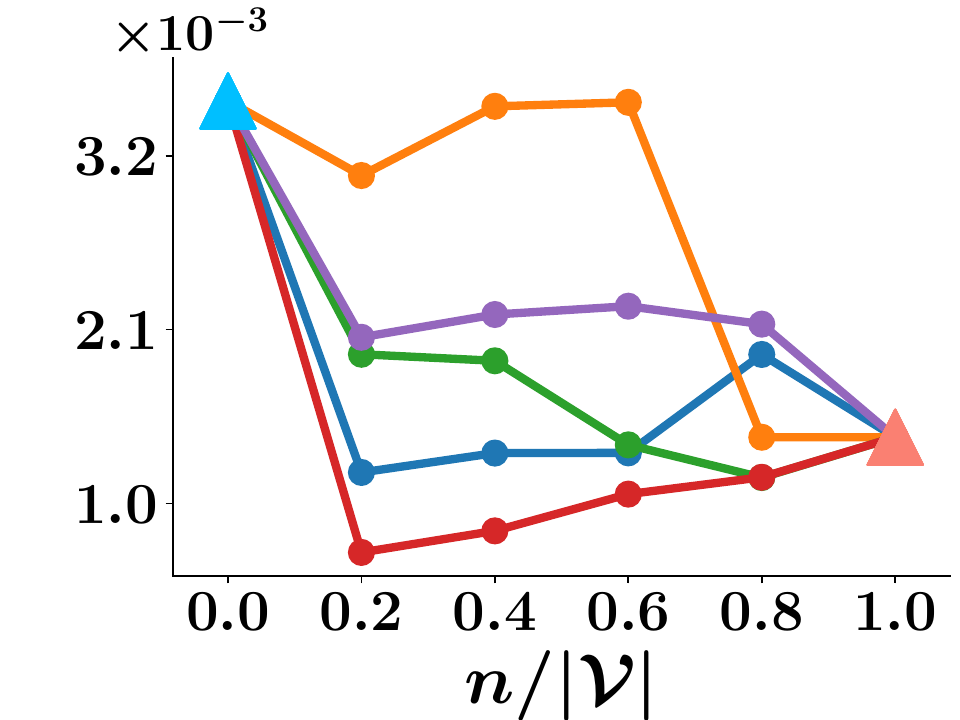}\label{fig:lte}} 
\caption{Mean squared error (MSE) against number of outsourced samples $n$ for the proposed greedy algorithm, DS~\cite{iyer2012algorithms}, 
distorted greedy~\cite{harshaw2019submodular}, Triage~\cite{raghu2019algorithmic} and CRR~\cite{bhatia2017consistent} on synthetic data. 
In all cases, we used the MLP model $h_{\theta}(\xb)$, $d=5$, $\sigma_2=10^{-3}$, $\lambda = 5 \cdot 10^{-3}$ and, for clarity, we explicitly 
highlight the performance under no automation and full automation.
The greedy algorithm consistently outperforms the baselines 
across the entire range of automation levels. For most automation levels, the competitive advantage provided by the greedy
algorithm is statistically significant (Welch'{}s t-test, $p$-value $= 10^{-3}$).}
\label{fig:quantSyn}
\end{figure} 
 \begin{figure}[!t]
\centering

\hspace*{0.5cm}{\includegraphics[width=0.55\textwidth]{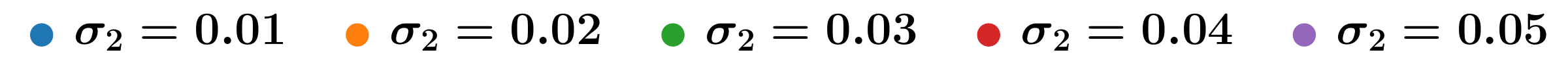}\label{fig:ltr}\vspace{-0.2cm}}\\
\subfloat[Gaussian]{\includegraphics[width=0.28\textwidth]{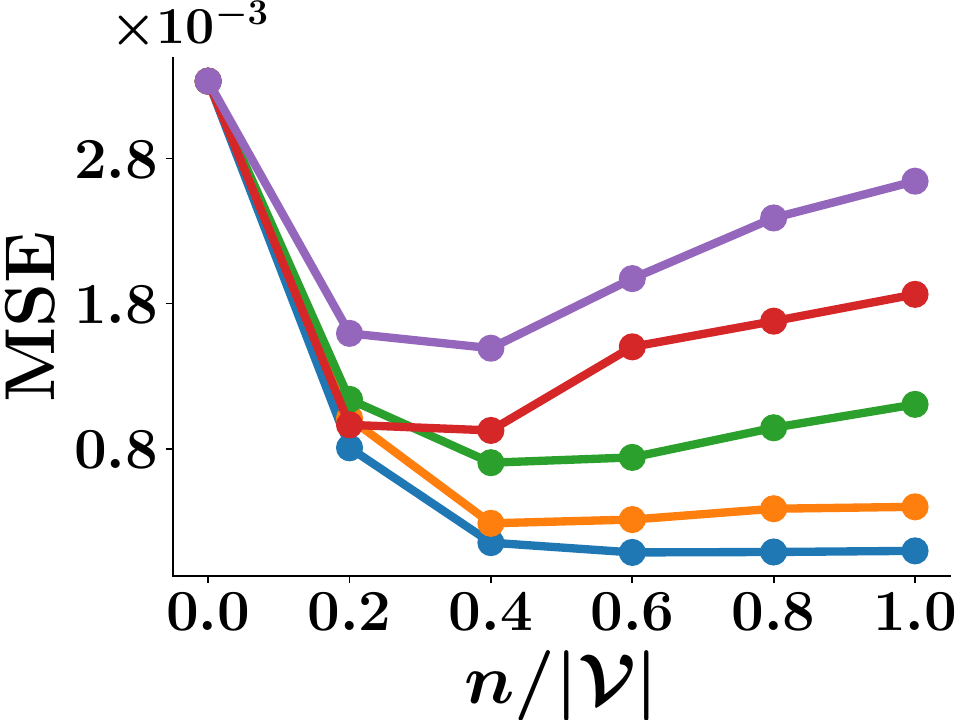}\label{fig:gtest}}\hspace*{0.5cm}
\subfloat[Logistic]{\includegraphics[width=0.28\textwidth]{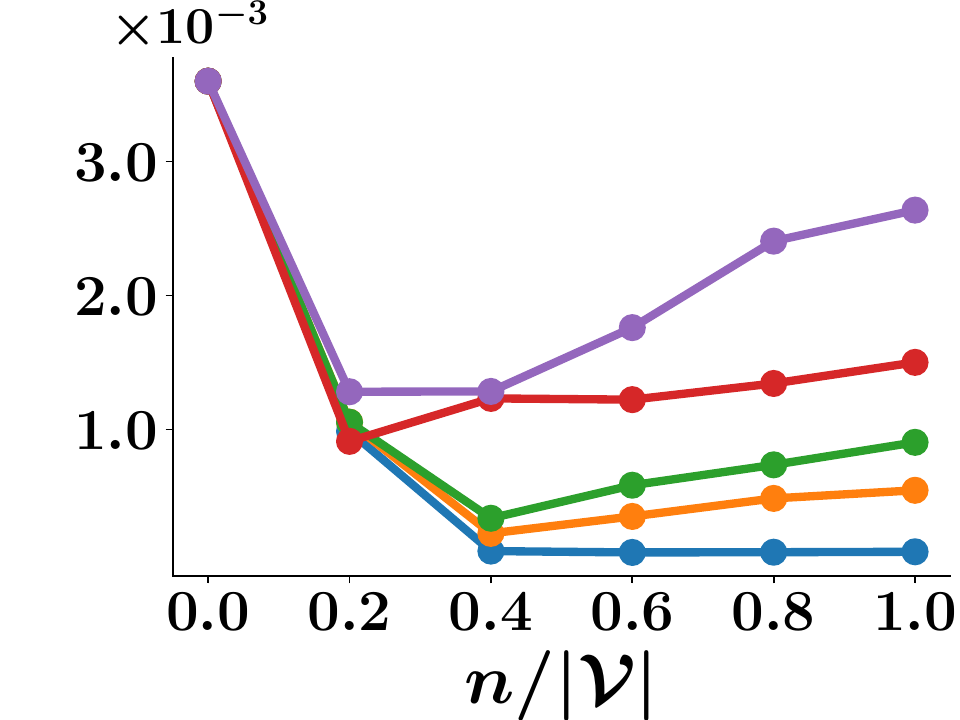}\label{fig:lte}} 
\caption{Mean squared error (MSE) achieved by the proposed greedy algorithm against the number of outsourced samples $n$ for 
different levels of human error ($\sigma_2$) on synthetic data. In all cases, we used the MLP model $h_{\theta}(\xb)$, $d=5$ and $\lambda = 5 \cdot 10^{-3}$.
%
For low levels of human error, the overall mean squared error decreases monotonically with respect to the number of outsourced samples.
In contrast, for high levels of human error, it is not beneficial to outsource samples to humans.}
 \label{fig:SynHumanVariation}
\end{figure}

Moreover, given the solutions provided by the above methods, we independently train three different (additional) models $h_{\theta}(\xb)$---nearest 
neighbors (NN), logistic regression (LR) and multilayer perceptron (MLP)---to decide which samples on the held-out set to outsource to humans. 
To train the LR and MLP models, we used the python package sklearn. 
More specifically, for the MLP model, we used the Adam optimizer with learning rate $10^{-3}$, set the hidden layer size as $1000$ and the number of 
epochs as 500 and, for the LR model, we used the liblinear solver and set the number of maximum epochs as $100$.
 \begin{figure}[!t]
\centering
\hspace*{0.5cm}{\includegraphics[width=0.55\textwidth]{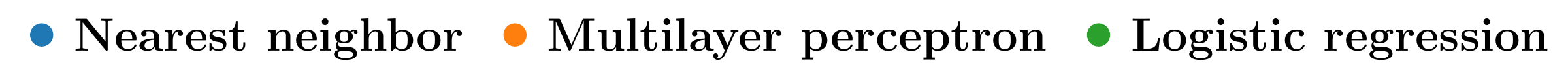}\label{fig:ltr}\vspace{-0.2cm}}\\
\subfloat[Gaussian]{\includegraphics[width=0.28\textwidth]{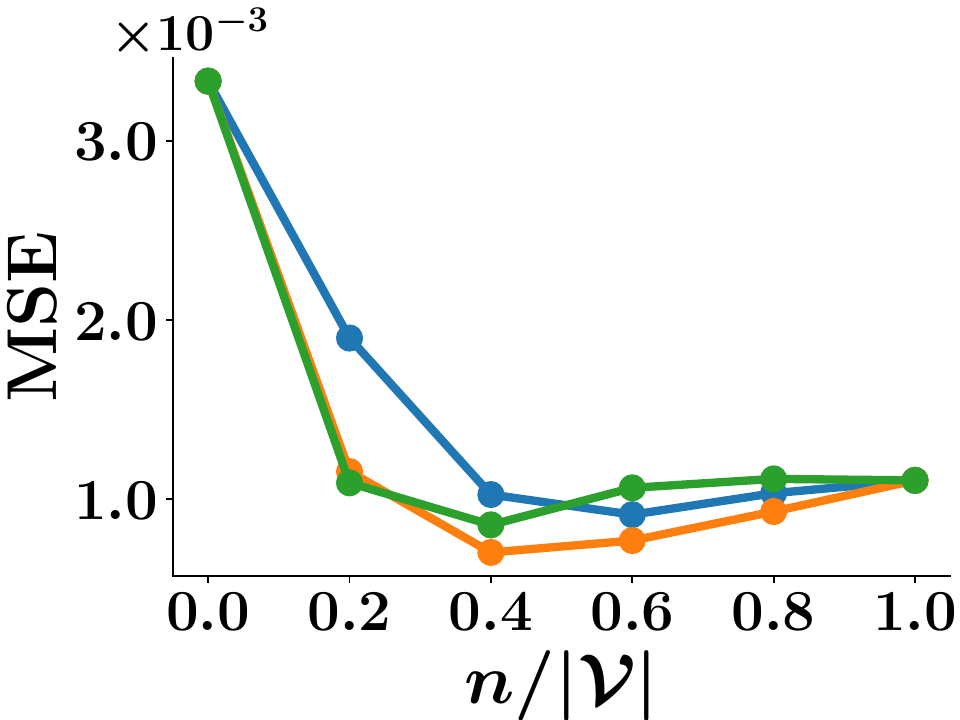}\label{fig:ghmodels}}\hspace*{0.5cm}
\subfloat[Logistic]{\includegraphics[width=0.28\textwidth]{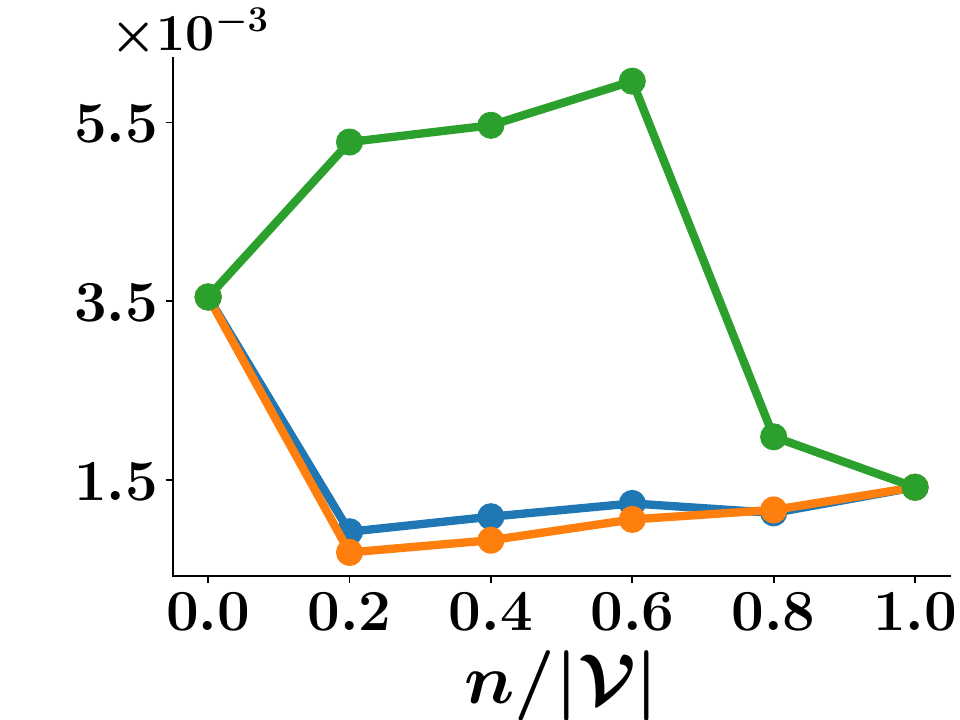}\label{fig:lhmodels}} 
\caption{Mean squared error (MSE) against number of outsourced samples $n$ for the proposed greedy algorithm and three different
models $h_{\theta}(\xb)$ on synthetic data. 
In all cases, we used $d=5$, $\sigma_2=10^{-3}$ and $\lambda = 5 \cdot 10^{-3}$.
NN and MLP achieve a comparable performance and they beat LR across a majority of automation levels for the Logistic dataset.
}
\label{fig:hmodels}
\end{figure}

\subsection{Results}
First, we look into the solution $(\wb^{*}(\Scal^{*}), \Scal^{*})$ provided by the greedy algorithm both for the Gaussian and logistic 
distributions and a different number of outsourced samples $n$. 
Figure~\ref{fig:qualSyn} summarizes the results, which reveal several interesting insights. For the logistic distribution, as $n$ increases, the 
greedy algorithm let the machine to focus on the samples where the relationship between features and the response variables is more linear 
and outsource the remaining points to humans.
For the Gaussian distribution, as $n$ increases, the greedy algorithm outsources samples on the tails of the distribution to humans.

Second, we compare the performance of the greedy algorithm against four competitive baselines in terms of mean squared error (MSE) on a held-out set. 
Figure~\ref{fig:quantSyn} summarizes the results, where we use the MLP model $h_{\theta}(\xb)$ for all methods. The results show that the greedy algorithm 
consistently outperforms the baselines for all automation levels and, for most automation levels, the competitive advantage provided by the greedy algorithm is statistically significant 
(Welch'{}s t-test, $p$-value $= 10^{-3}$). 
Also, using the greedy algorithm, we find automation levels under which humans and machines working together achieve better 
performance than machines on their own (Full automation) as well as humans on their own (No automation).
We obtained qualitatively similar results using the NN and LR models $h_{\theta}(\xb)$ (refer to Appendix~\ref{app:syn-expts}).

Next, we investigate how the performance of our greedy algorithm varies with respect to the amount of human error.
Figure~\ref{fig:SynHumanVariation} summarizes the results, where we again use the MLP model $h_{\theta}(\xb)$. The results show that, for low levels of human error, 
the overall mean squared error decreases monotonically with respect to the number of outsourced samples to humans.
In contrast, for high levels of human error, it is not beneficial to outsource samples.

Finally, we compare the performance of the greedy algorithm in terms of mean squared error (MSE) on a held-out set for the above mentioned 
models $h_{\theta}(\xb)$, \ie, nearest neighbors (NN), logistic regression (LR) and multilayer perceptron (MLP). Figure~\ref{fig:hmodels} 
summarizes the results, which show that NN and MLP achieve a comparable performance and they beat LR across a majority of automation levels 
for the Logistic dataset.
%

\section{Experiments on Real Data}
\label{sec:real}
In this section, we experiment with four real-world datasets from two important applications, medical diagnosis and content moderation. 
%
First, we look closely at the samples that our algorithm outsources to humans to better understand to which extent our algorithm has the ability to learn 
the underlying relationship between a given sample and its corresponding human and machine error.
Then, we compare the performance of the greedy algorithm with the same competitive baselines as in the experiments on synthetic data.
%
%
Finally, we perform a detailed sensitivity analysis that evaluate the robustness of our algorithm with respect to the choice of additional model $h_{\theta}(\xb)$ as well as with respect to several hyperparameters used in pre-processing of the four real-world datasets.
 \begin{figure}[!t]
\centering
\subfloat[Easy sample]{\includegraphics[width=0.20\textwidth]{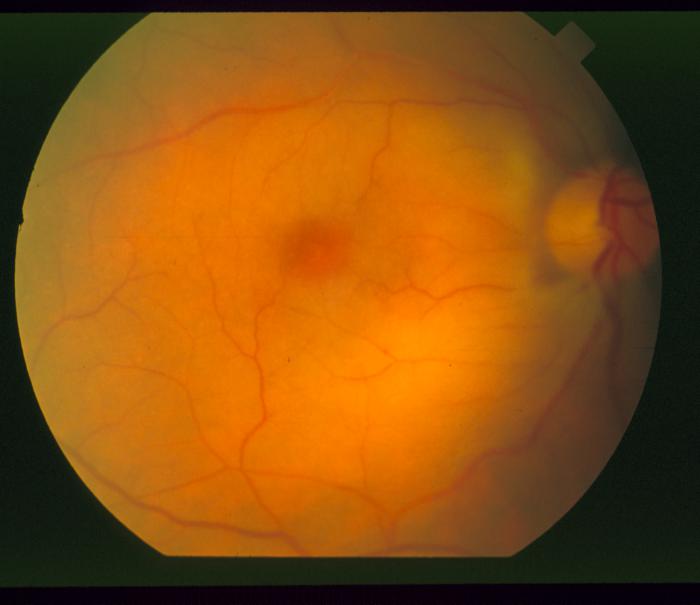}\label{fig:im4}}\hspace*{2cm}
\subfloat[Difficult sample]{\includegraphics[width=0.20\textwidth]{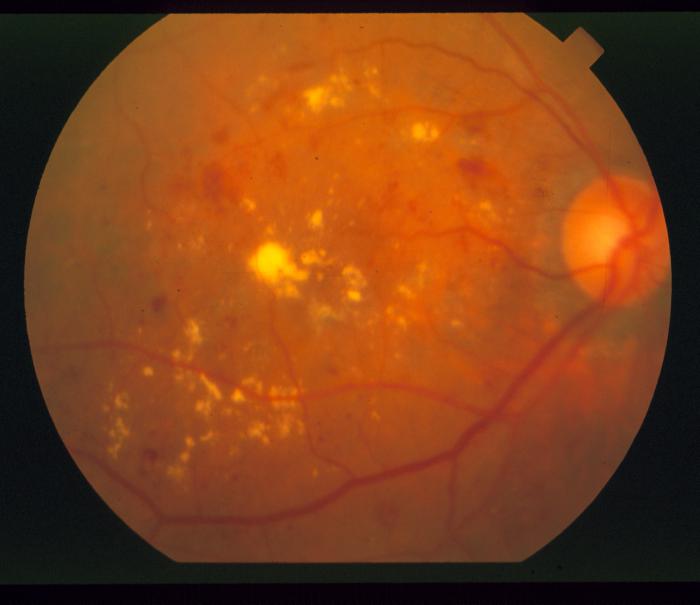}\label{fig:im58}}  
\caption{An easy and a difficult sample image from the Stare-D dataset. Both images are given a score of severity zero for the Drusen disease, which is 
characterized by pathological yellow spots. The easy sample does not contain yellow spots and thus it is easy to predict its score. In contrast, the difficult 
sample contains yellow spots, which are manifested not from Drusen, but diabetic retinopathy, and thus it is challenging to accurately predict its score.
As a result, the greedy algorithm decides to outsource the difficult sample to humans, whereas it lets the machine decide about the easy one. 
} \vspace{-3mm}
\label{fig:examples}
\end{figure}

\begin{figure}[!t]
\centering
\hspace*{0.5cm}{\includegraphics[width=0.5\textwidth]{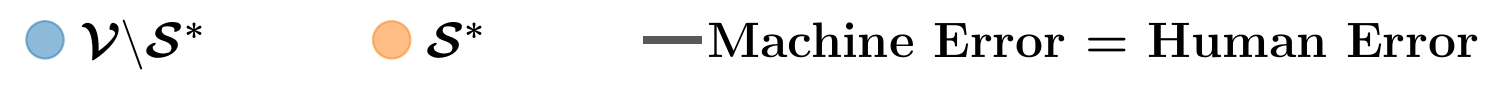}\label{fig:ltr}\vspace{-0.2cm}}\\
\subfloat{\includegraphics[width=0.22\textwidth]{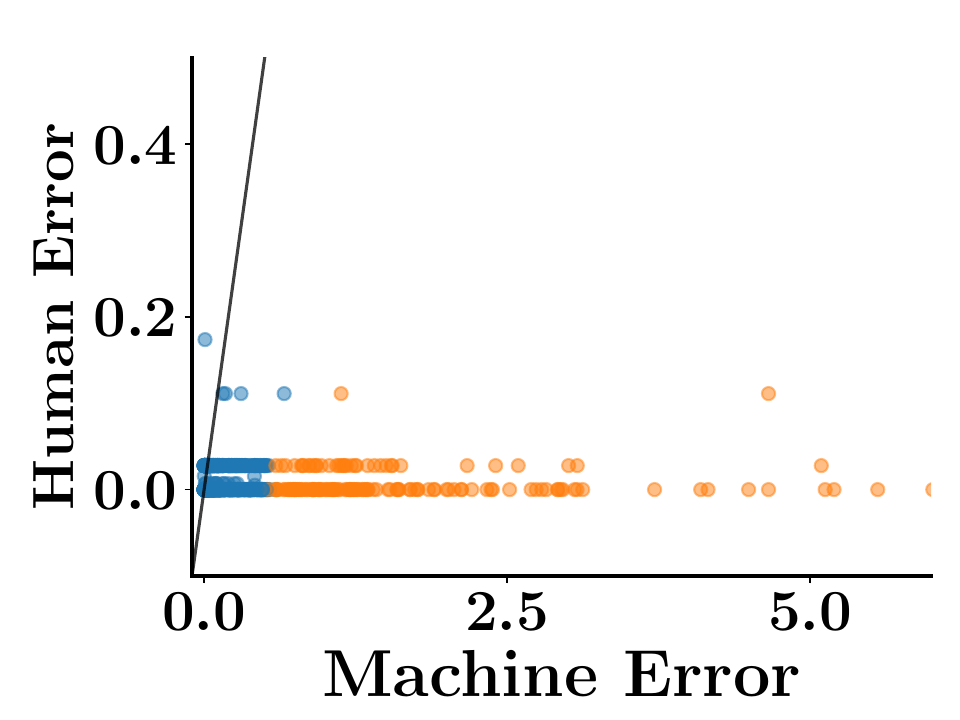}\label{fig:gtests}}\hspace*{0.2cm}
\subfloat{\includegraphics[width=0.22\textwidth]{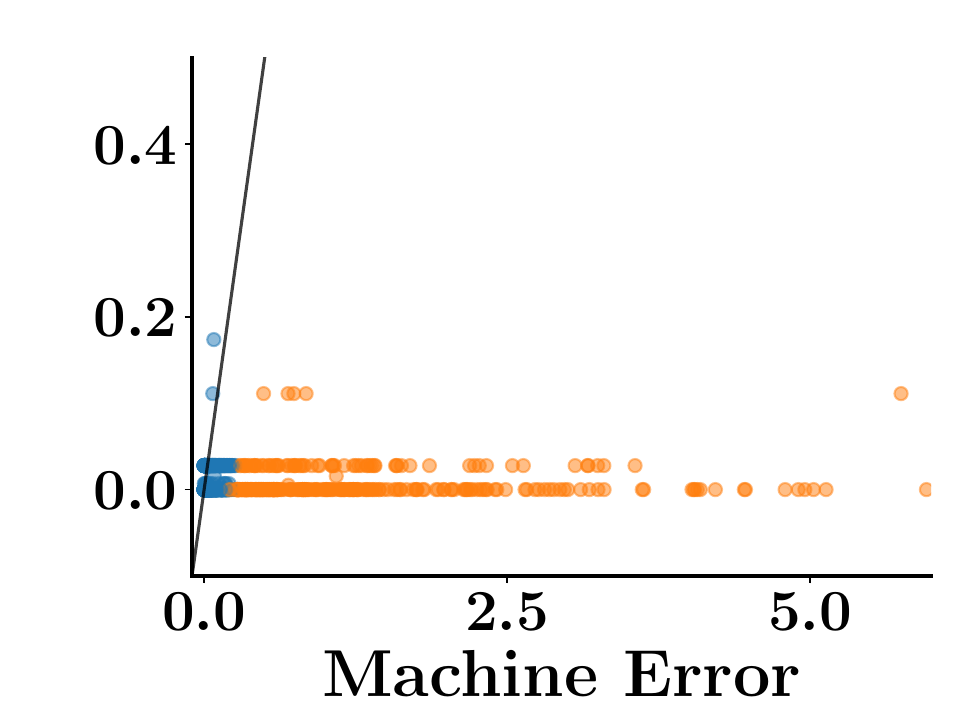}\label{fig:ltes1}} \hspace*{0.2cm}
\subfloat{\includegraphics[width=0.22\textwidth]{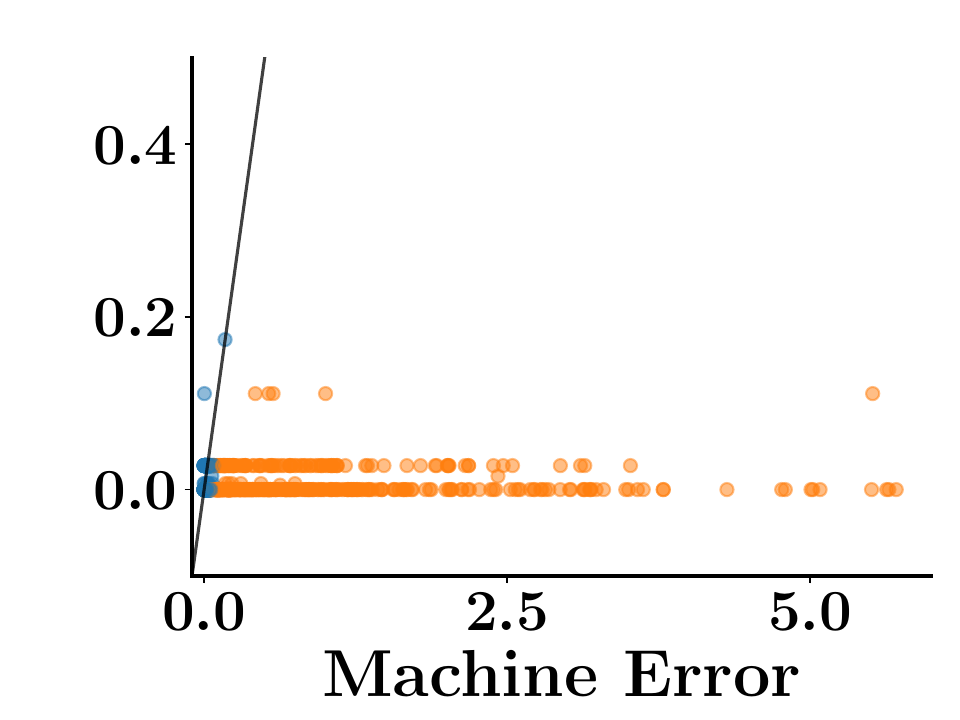}\label{fig:lte22}} \hspace*{0.2cm}
\subfloat{\includegraphics[width=0.22\textwidth]{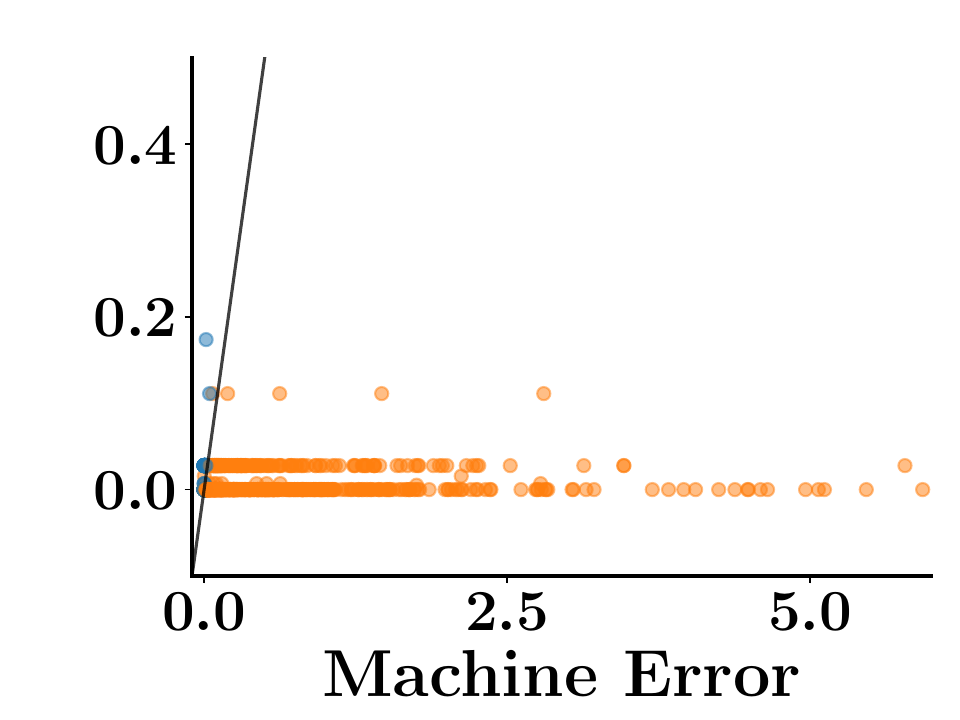}\label{fig:lte23}} \\
\subfloat{\includegraphics[width=0.22\textwidth]{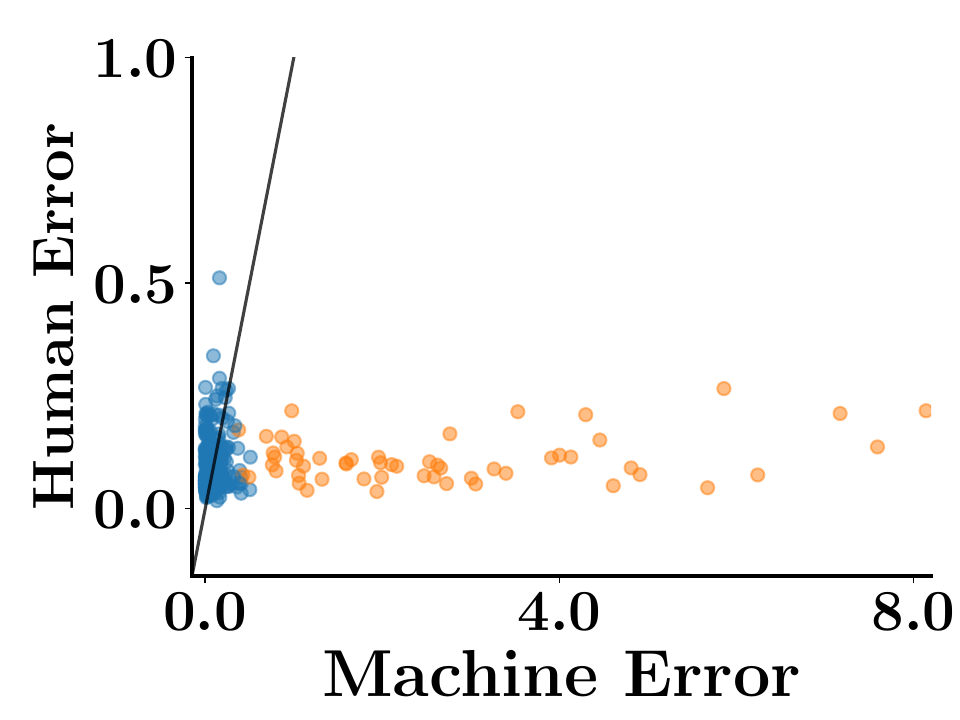}\label{fig:gtests}}\hspace*{0.2cm}
\subfloat{\includegraphics[width=0.22\textwidth]{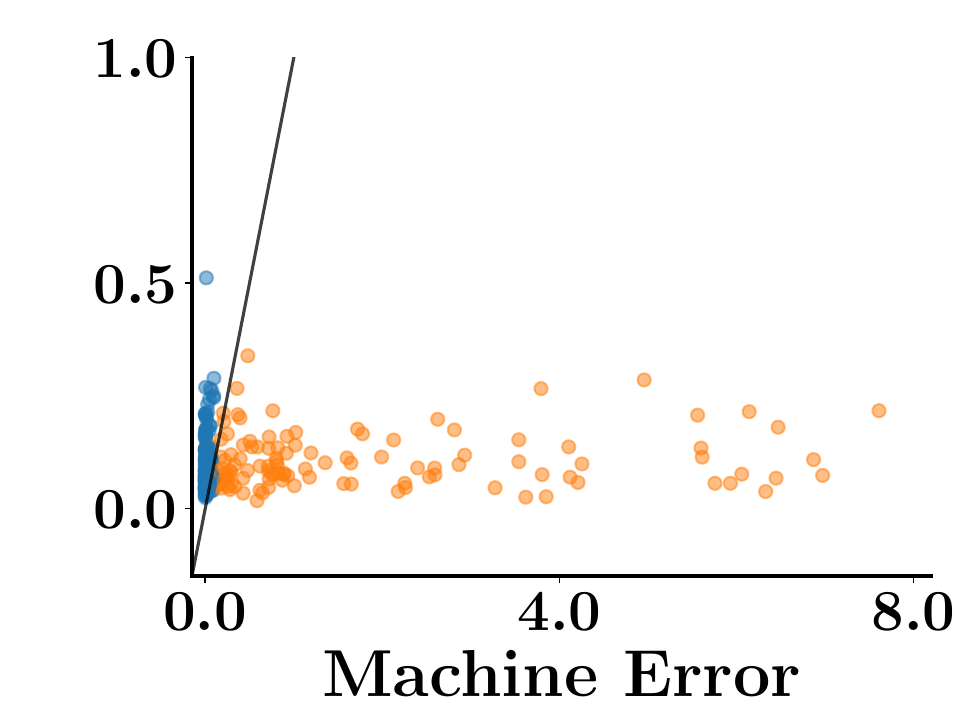}\label{fig:ltes1}} \hspace*{0.2cm}
\subfloat{\includegraphics[width=0.22\textwidth]{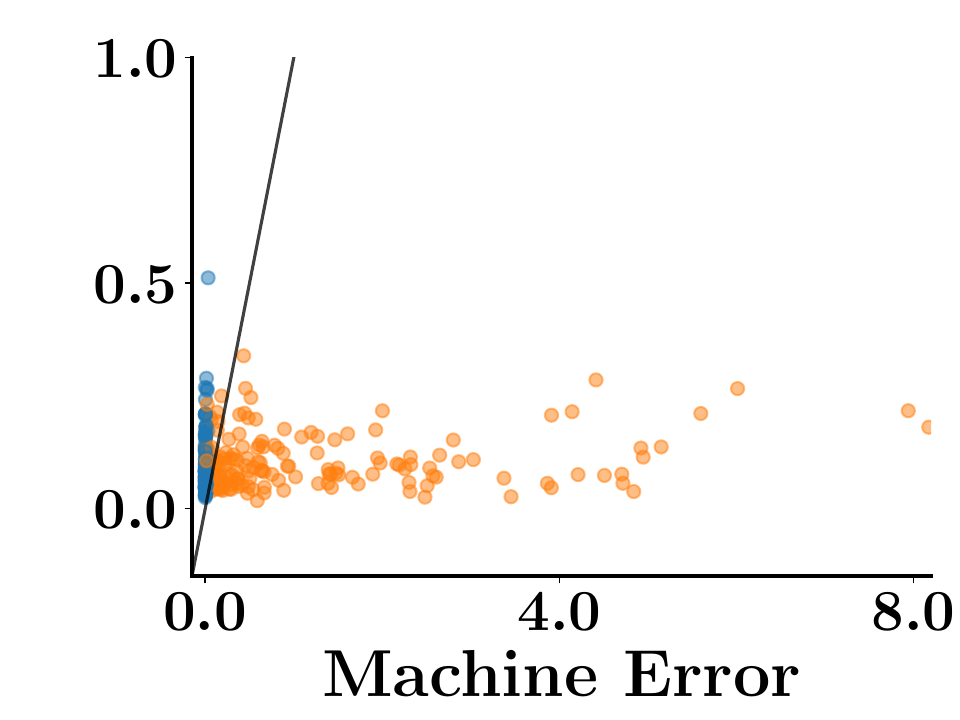}\label{fig:lte22}} \hspace*{0.2cm}
\subfloat{\includegraphics[width=0.22\textwidth]{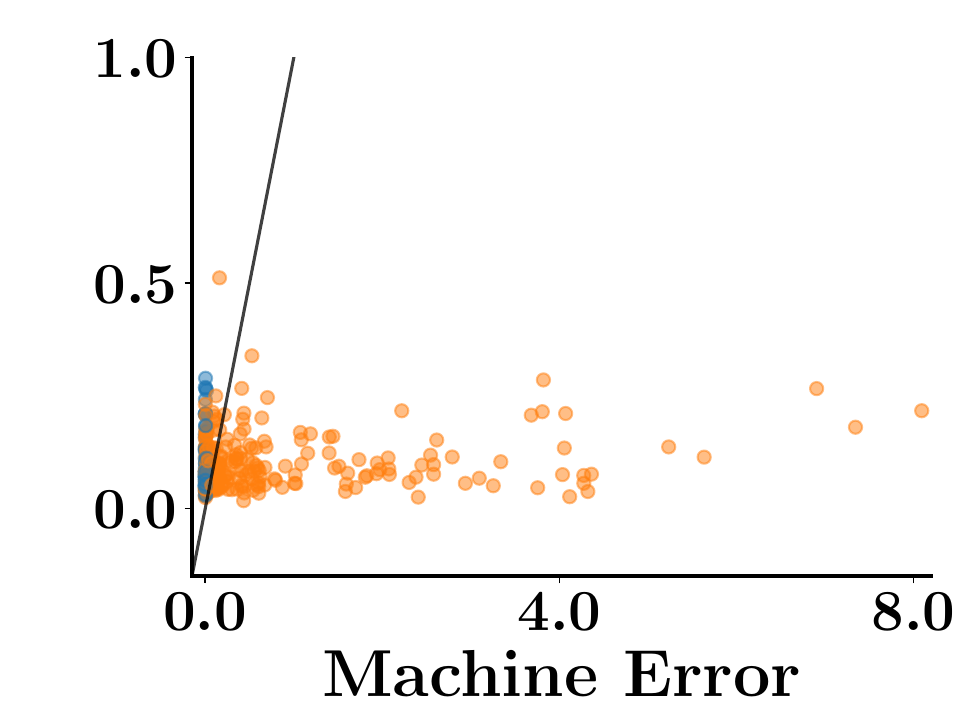}\label{fig:lte23}} \\
\subfloat{\includegraphics[width=0.22\textwidth]{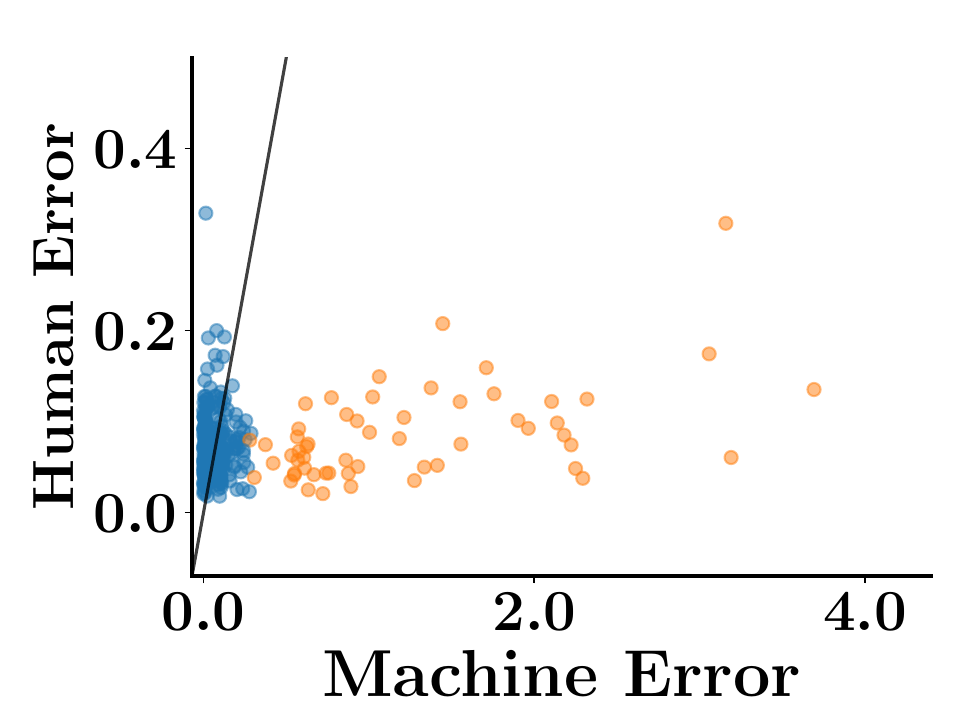}\label{fig:gtests}}\hspace*{0.2cm}
\subfloat{\includegraphics[width=0.22\textwidth]{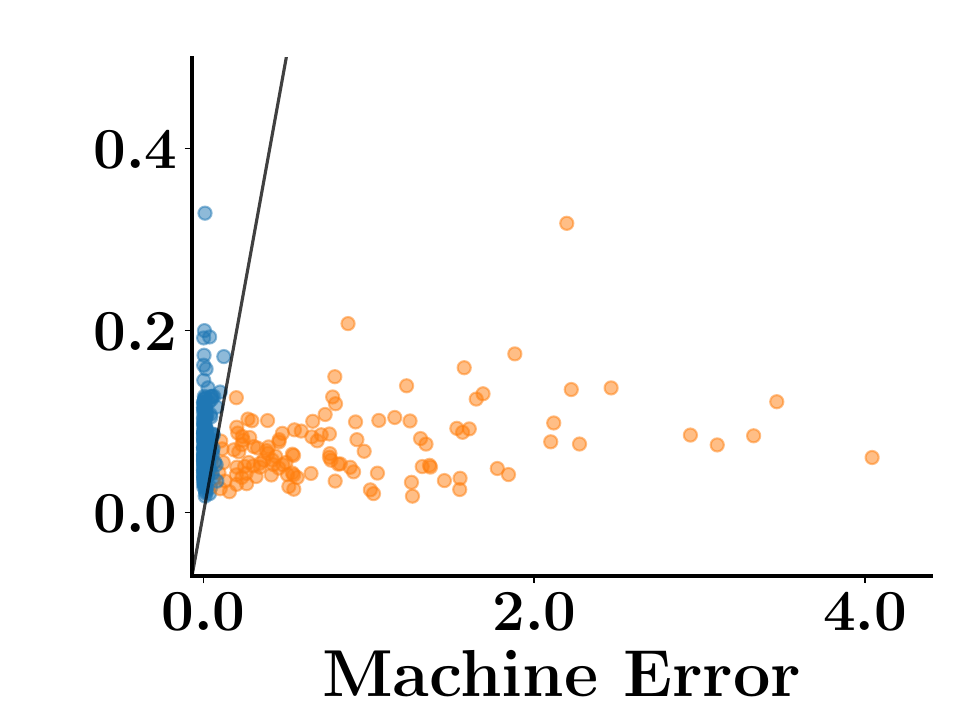}\label{fig:ltes1}} \hspace*{0.2cm}
\subfloat{\includegraphics[width=0.22\textwidth]{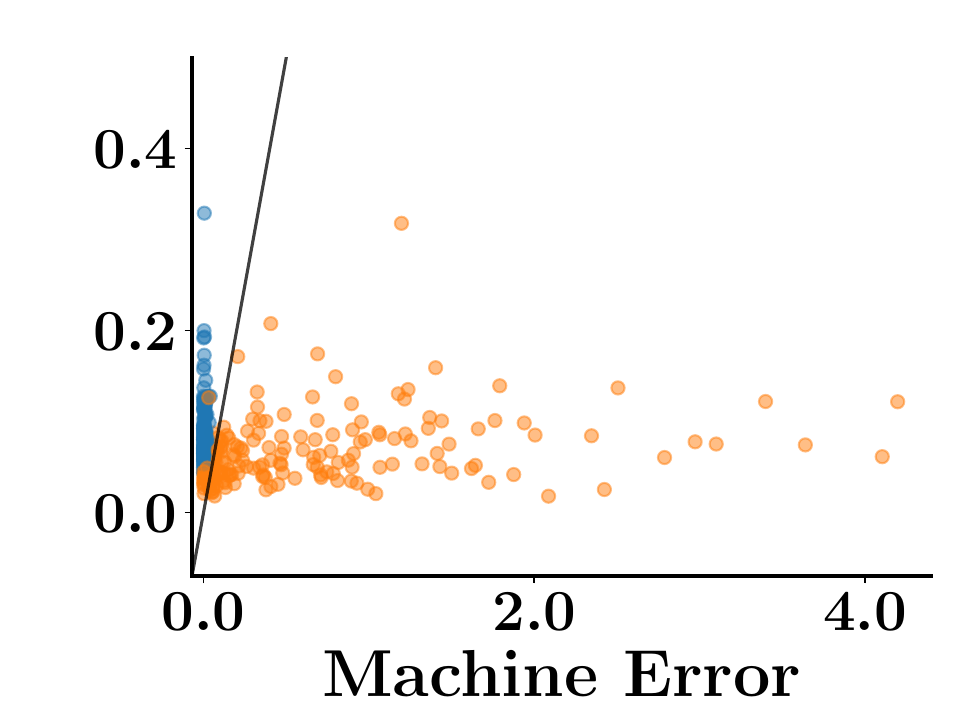}\label{fig:lte22}} \hspace*{0.2cm}
\subfloat{\includegraphics[width=0.22\textwidth]{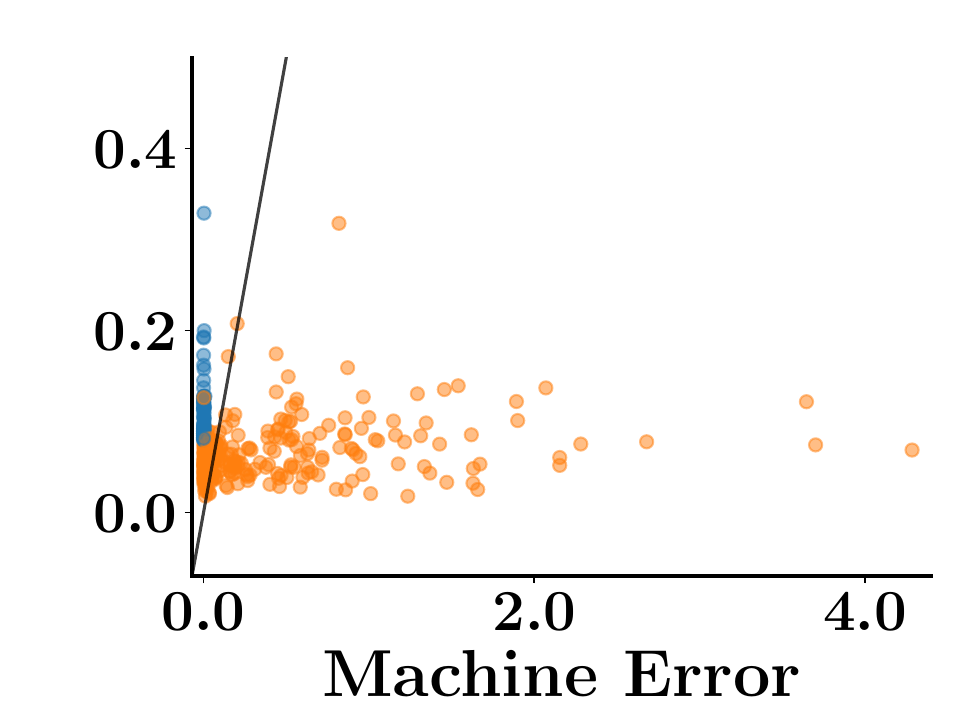}\label{fig:lte23}}\\ 
\subfloat[$|\Scal|=0.2|\Vcal|$]{\setcounter{subfigure}{1} \includegraphics[width=0.22\textwidth]{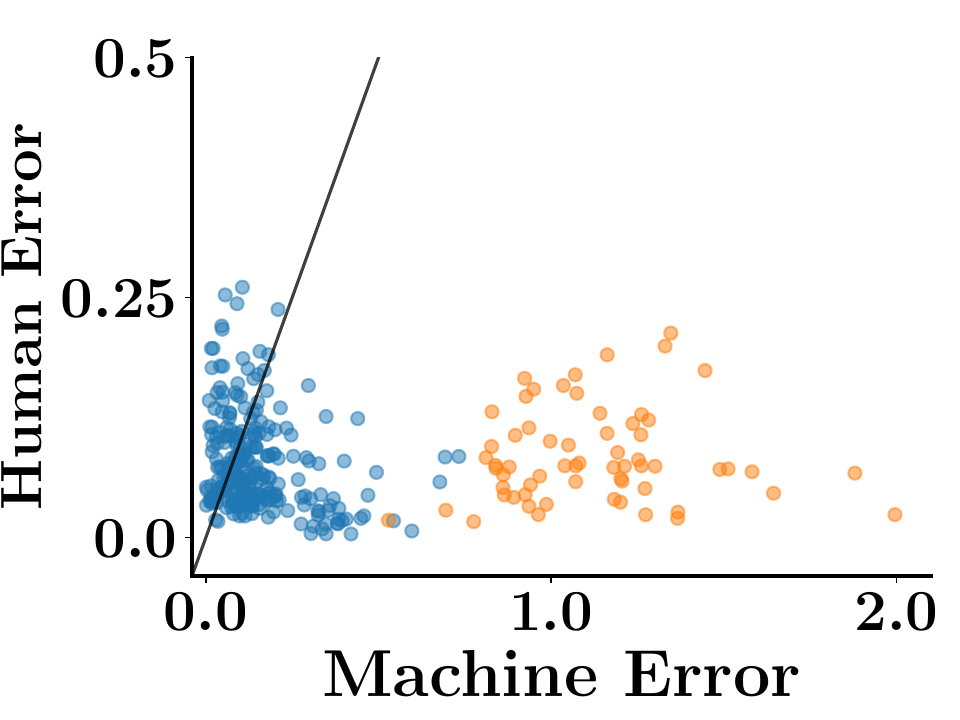}\label{fig:gtests}}\hspace*{0.2cm}
\subfloat[$|\Scal|=0.4|\Vcal|$]{\includegraphics[width=0.22\textwidth]{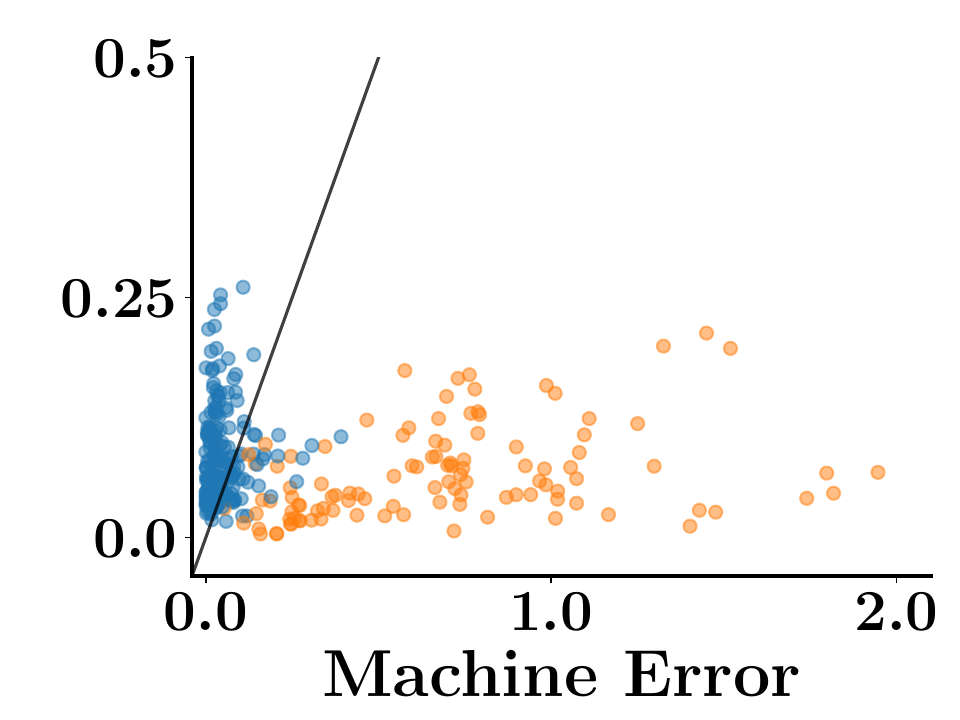}\label{fig:ltes1}} \hspace*{0.2cm}
\subfloat[$|\Scal|=0.6|\Vcal|$]{\includegraphics[width=0.22\textwidth]{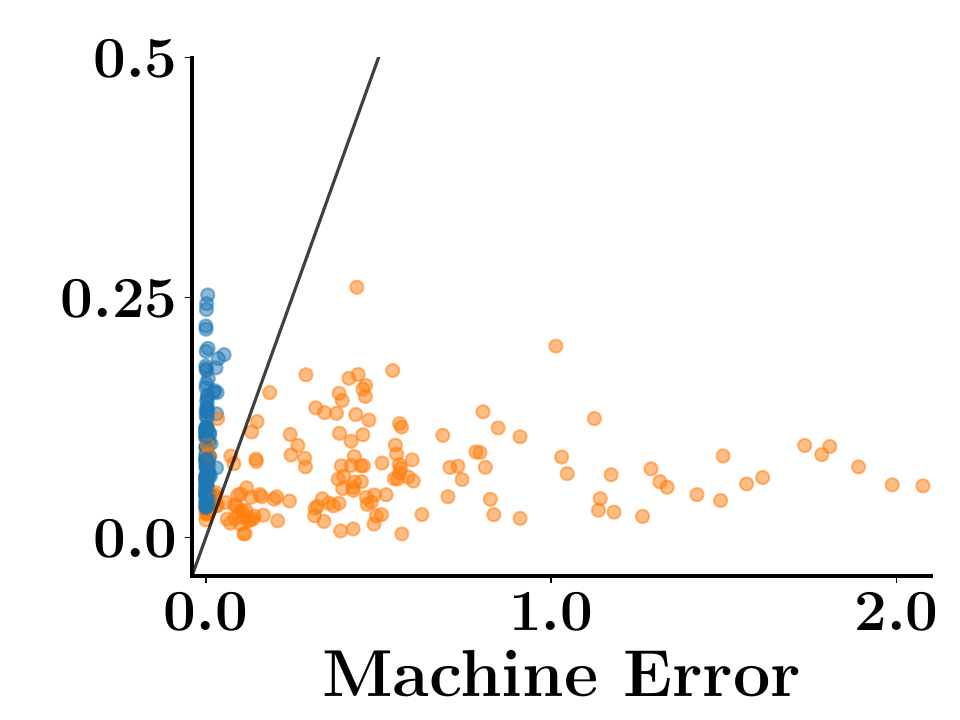}\label{fig:lte22}} \hspace*{0.2cm}
\subfloat[$|\Scal|=0.8|\Vcal|$]{\includegraphics[width=0.22\textwidth]{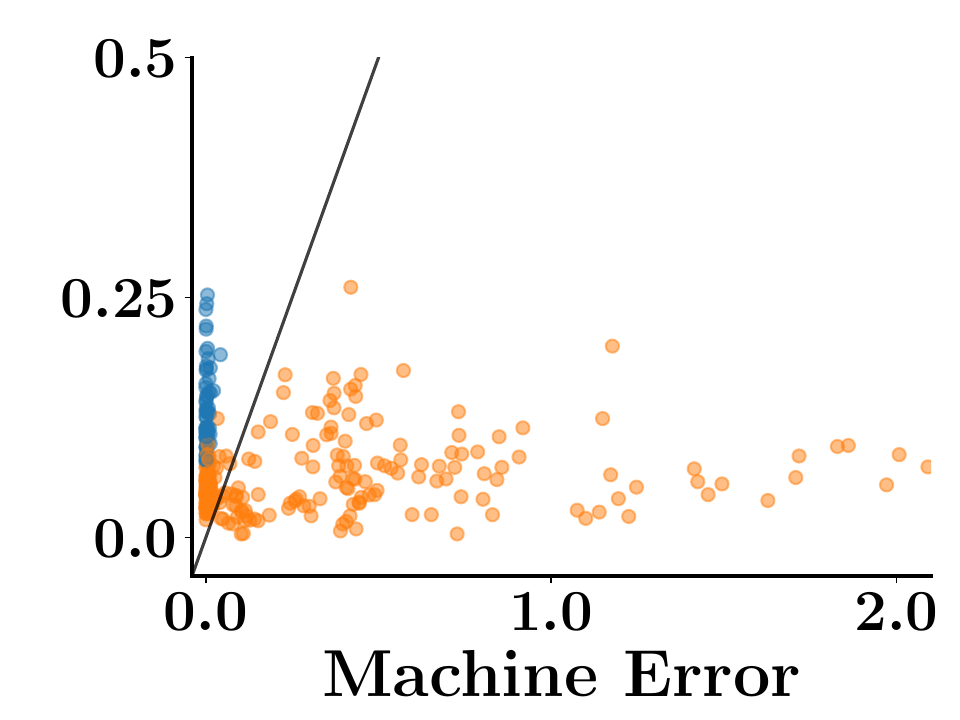}}
\caption{Human and machine error of all training samples on the Hatespeech (first row), Stare-H (second row), Stare-D (third row) and Messidor (third row) datasets 
for different automation levels.
Our algorithm outsources to humans those samples in which the machine error would have been the highest if it had to predict their response variables.
} \vspace{-3mm}
\label{fig:scatter-human-machine}
\end{figure}

\subsection{Experimental setup} We experiment with one dataset for content moderation and three datasets for medical diagnosis, which are
publicly available~\cite{hateoffensive,decenciere_feedback_2014,hoover2000locating}. More specifically:
\begin{itemize}
\item[(i)] \textbf{Hatespeech: } It consists of $\sim$$25000$ tweets
containing words, phrases and lexicons used in hate speech. Each tweet is given several scores 
by three to five annotators from Crowdflower, which measure the severity of hate speech. 

\item[(ii)] \textbf{Stare-H:} It consists of $\sim$$400$ retinal images. Each image is given a score by one single expert, on a five point scale, which measures the
severity of a retinal hemorrhage. 

\item[(iii)] \textbf{Stare-D:} It contains the same set of images from Stare-H. However, in this dataset, each image is given a score by a single expert, on a six point 
scale, which measures the severity of the Drusen disease.

\item[(iv)] \textbf{Messidor:} It contains $400$ eye images. Each image is given score by one single expert, on a four point scale, which measures the severity of 
an edema.
\end{itemize}
We first generate a $m'=100$ dimensional feature vector using fasttext~\cite{ft} for each sample in the Hatespeech dataset,  $m'=1000$ dimensional feature vector using Resnet~\cite{resnet} for each sample in the Stare-H, Stare-D, and a $m'=4096$-dimensional feature vector using VGG~\cite{simonyan2014very} for each 
sample in the Messidor dataset. Then, we use the top $m=50$ features, as identified by PCA, as $\xb$ in our experiments.
For the image datasets, the response variable $y$ is just the available score by a single expert and the human predictions are sampled from a categorical distribution 
$s \sim \text{Cat}(\pb_{\xb, y})$, where $\pb_{\xb, y} \sim \text{Dirichlet}(\alphab_{\xb, y})$ are the probabilities of each potential score value $s$ for a sample with features 
$\xb$ and $\alphab_{\xb, y}$ is a vector parameter that controls the human accuracy. 
Here, for each sample $(\xb, y)$, the element of $\alphab_{\xb, y}$ corresponding to the score $s = y$ has the highest value.
%
%
%
For the Hatespeech dataset, the response variable $y$ is the mean of the scores provided by the annotators and the human predictions are picked uniformly at random
from the available individual scores given by each annotator.
%
%
In each dataset, we compute the human error as $c(\xb,y)=\EE(y-s)^2$ for each sample $(\xb,y)$ and set the same value of $\lambda$ across all 
competitive methods.
%
%
Finally, in each experiment, we use 80\% samples for training and 20\% samples for testing.
%
%
%

\subsection{Results}  
\label{sec:expts-real-results}

We first look closely at the samples that our algorithm outsources to humans to better understand to which extent our algorithm has the ability to learn 
the underlying relationship between a given sample and its corresponding human and machine error.
Intuitively, human assistance should be required for those samples which are difficult (easy) for a machine (a human) to decide about. 
Figure~\ref{fig:examples} provides an illustrative example of an easy and a difficult sample image.
While both sample images are given a score of severity zero for the Drusen disease, one of them contains yellow spots, which are often a 
sign of Drusen disease\footnote{In this particular case, the patient suffered diabetic retinopathy, which is also characterized by yellow spots.}, and 
is therefore difficult to predict. 
In this particular case, the greedy algorithm outsourced the difficult sample to humans and let the machine decide about the easy one.
%
%
Does this intuitive assignment happen consistently? 
 \begin{figure}[!t]
\centering
\hspace*{0.5cm}{\includegraphics[width=0.6\textwidth]{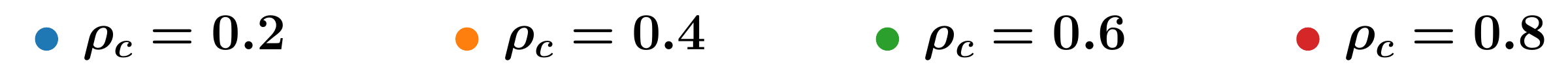}\label{fig:ltr}\vspace{-0.2cm}}\\
\subfloat[Stare-H]{\includegraphics[width=0.28\textwidth]{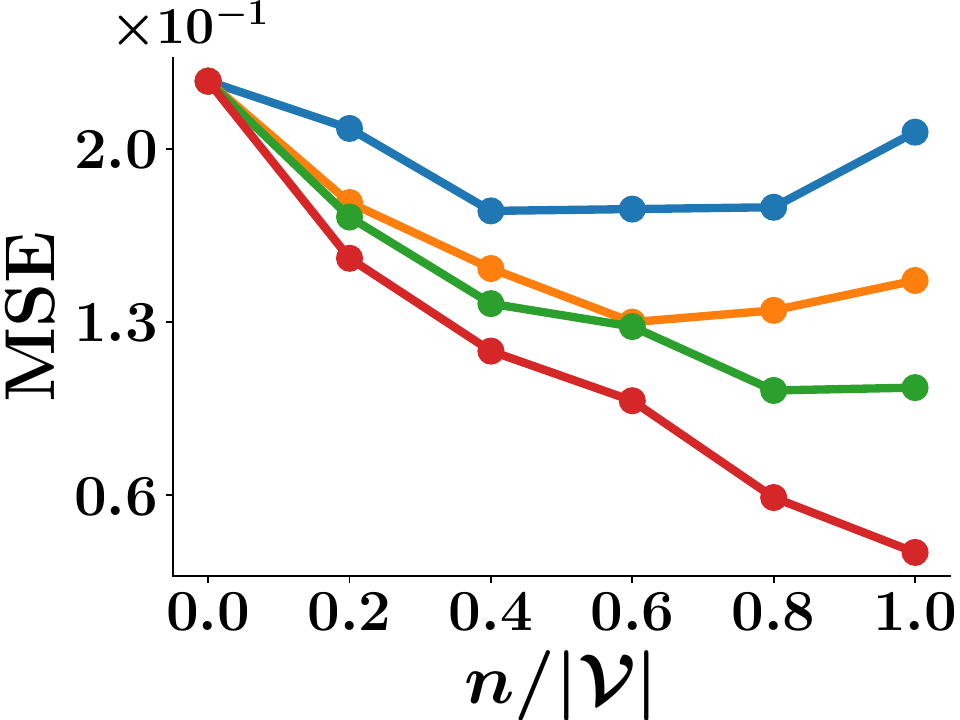}\label{fig:ltes1}}\hspace*{1cm}  
\subfloat[Stare-D]{\includegraphics[width=0.28\textwidth]{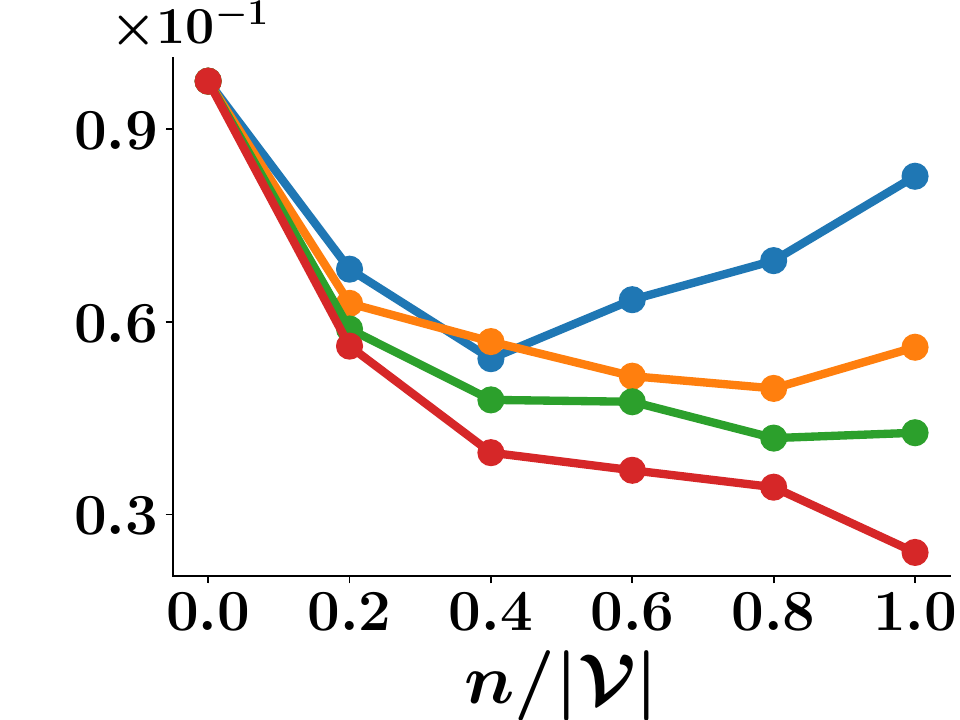}\label{fig:lte22}}\hspace*{1cm}  
\subfloat[Messidor]{\includegraphics[width=0.28\textwidth]{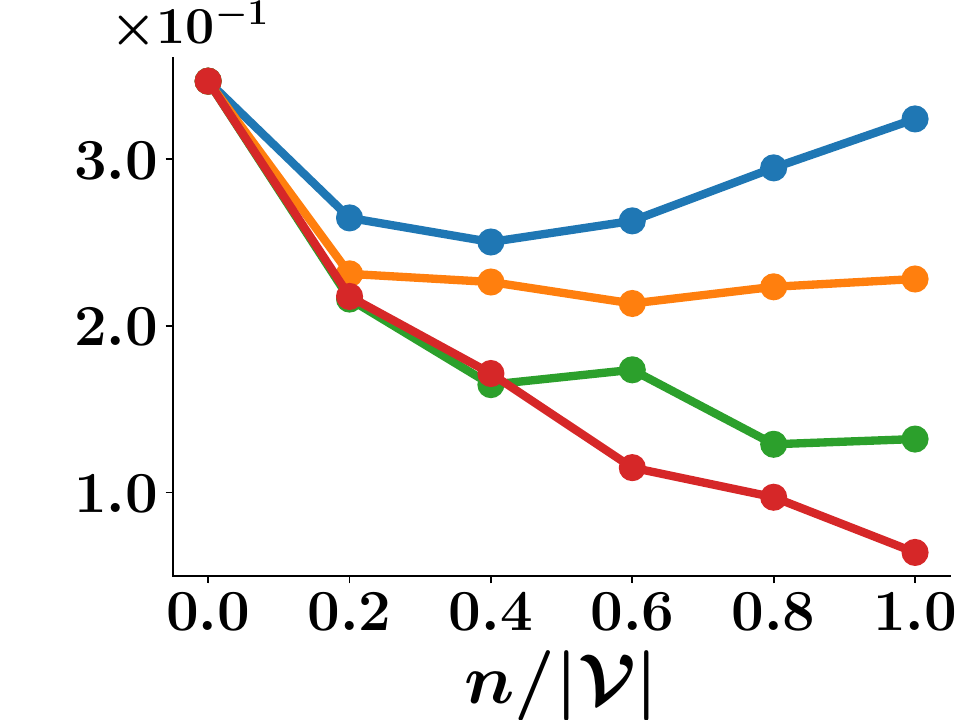}\label{fig:lte23}}
\caption{Mean squared error (MSE) achieved by the proposed greedy algorithm against the number of outsourced samples $n$ under different distributions
of human errors on three real-world datasets.
Under each distribution of human error, human error is low for a fractions $\rho_c$ of the samples and high for the remaining fraction $1-\rho_c$.
In all cases, we use the LR model $h_{\theta}(\xb)$ and set $c(\xb, y) = 10^{-4}$ for samples in which the human error is low and we set $c(\xb, y) = 0.25, 0.1, 0.4$ respectively for Stare-H, Stare-D and Messidor for samples in which the human error is high.
As long as there are samples that humans can predict with low error, the greedy algorithm does outsource them to humans and thus the 
overall performance improves. 
However, whenever the fraction of outsourced samples is higher than the fraction of samples with low human error, the performance degrades. 
This results in a characteristic U-shaped curve.
}\vspace{-3mm}
 \label{fig:realErrorWithC}
\end{figure}

To answer the above question, we run two complementary experiments. 
In a first experiment, we compare the human and machine error for each (training and test) sample, independently on whether it was outsourced to humans. 
Figure~\ref{fig:scatter-human-machine} summarizes the results for samples in the training set, which show that our algorithm outsources to humans those 
samples in which the machine error would have been the highest if it had to predict their response variables. Figure~\ref{fig:scatter-human-machine-test} 
in Appendix~\ref{app:expl-real} summarizes the results for samples in the test set, where the same qualitative observation holds.
In a second experiment, we run our greedy algorithm on the Messidor, Stare-H and Stare-D datasets under different distributions of human error and assess
to which extent the greedy algorithm outsources to humans those samples they can predict more accurately.
More specifically, we sample the human predictions from a non-uniform categorical distribution with parameters $\alpha_{\xb, y}$ under which human error 
is low for a fraction $\rho_c$ of the samples and high for the remaining fraction $1-\rho_c$.
Figure~\ref{fig:realErrorWithC} shows the performance of the greedy algorithm for different $\rho_c$ values on a held-out set, where we use the LR model 
$h_{\theta}(\xb)$ for all methods. 
We observe that, as long as there are samples that humans can predict with low error, the greedy algorithm outsources them to humans and thus the 
overall performance improves. 
However, whenever the fraction of outsourced samples is higher than the fraction of samples with low human error, the performance degrades, and this
leads to a characteristic U-shaped curve. 
The results from both experiments suggest that our algorithm has the ability to learn the underlying relationship between a given sample and its corresponding 
human and machine error.
 \begin{figure*}[!t]
\centering
\hspace*{0.5cm}{\includegraphics[width=0.6\textwidth]{legend_notbold}\label{fig:ltr}\vspace{-0.2cm}}\\
\subfloat[Hatespeech]{\includegraphics[width=0.22\textwidth]{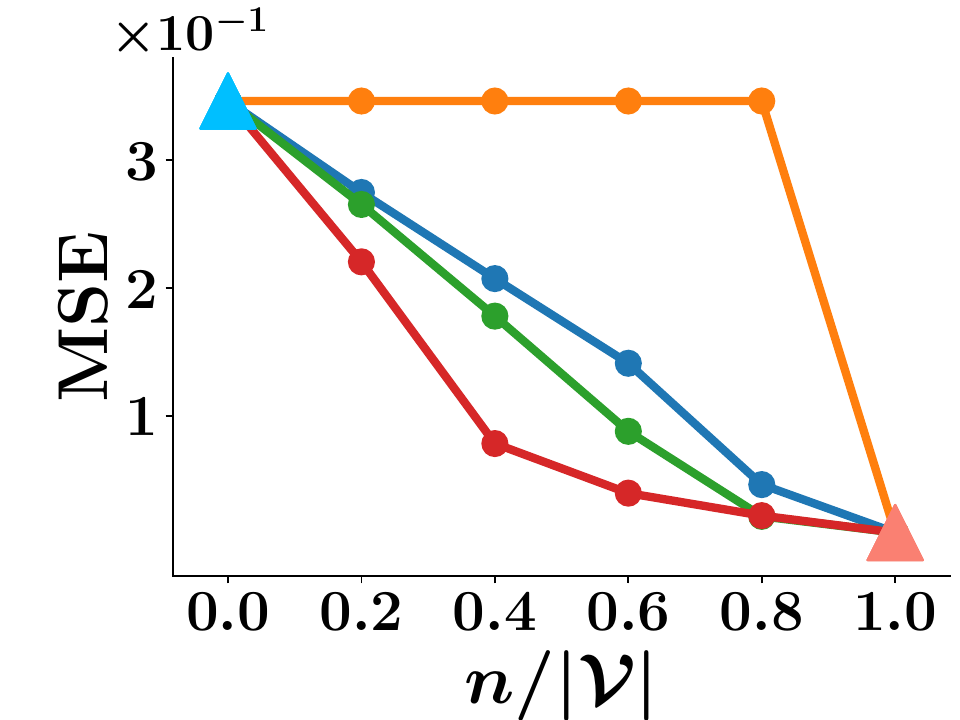}\label{fig:gtests}}\hspace*{0.2cm}
\subfloat[Stare-H]{\includegraphics[width=0.22\textwidth]{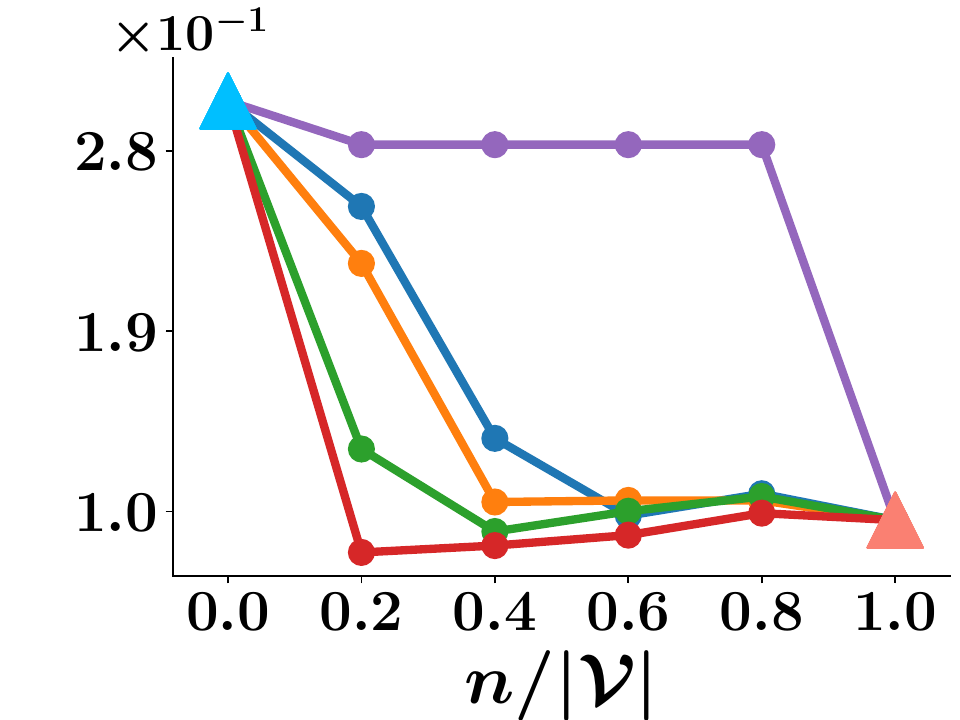}\label{fig:ltes1}} \hspace*{0.2cm}
\subfloat[Stare-D]{\includegraphics[width=0.22\textwidth]{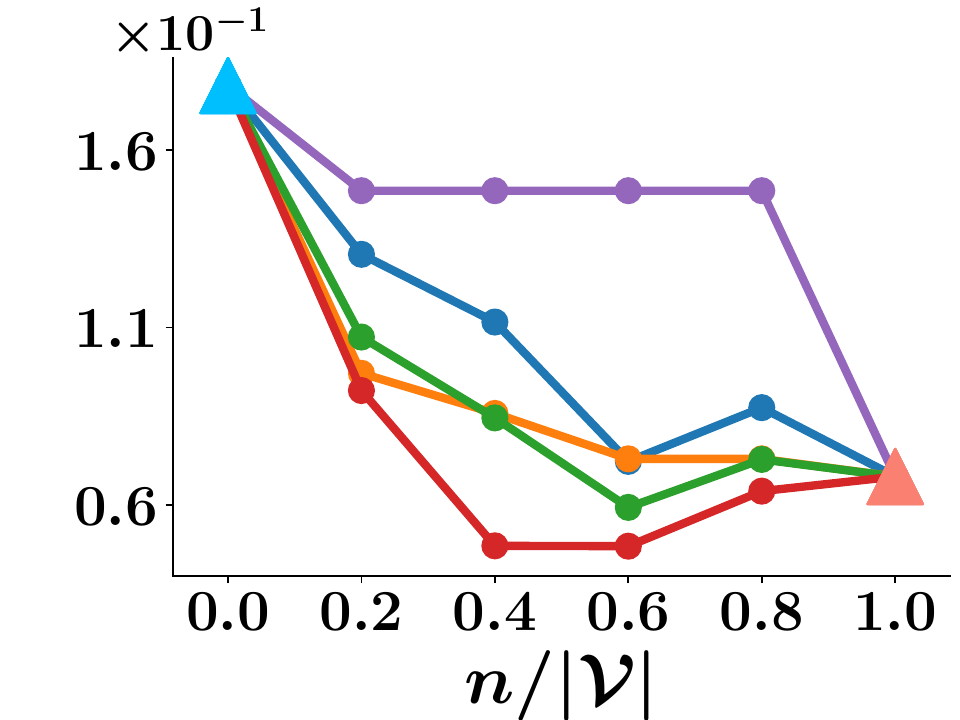}\label{fig:lte22}} \hspace*{0.2cm}
\subfloat[Messidor]{\includegraphics[width=0.22\textwidth]{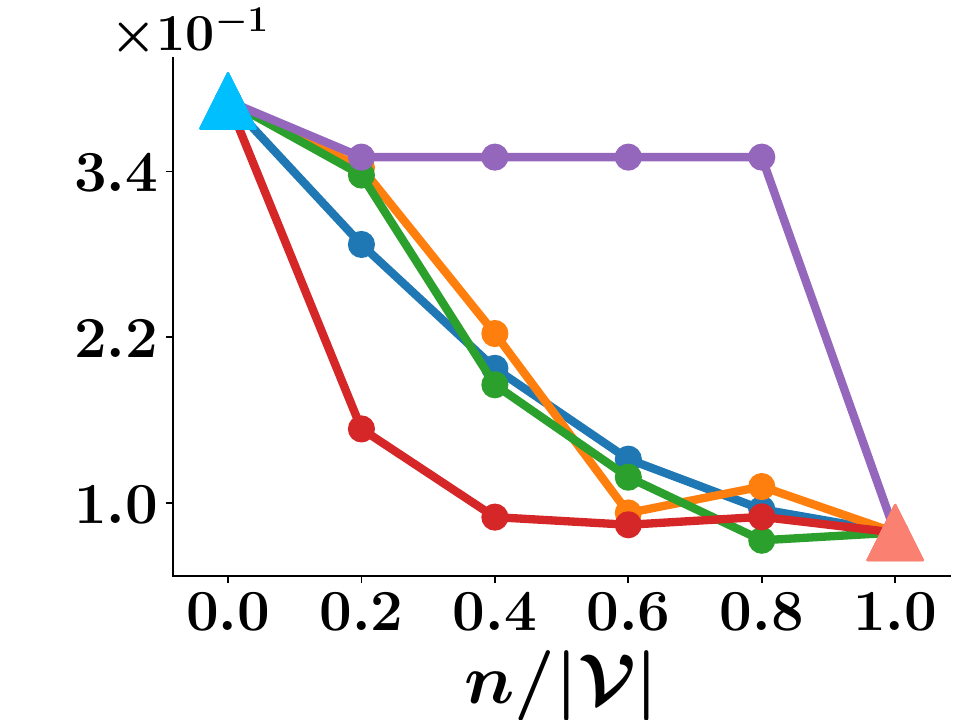}\label{fig:lte23}} 
\caption{Mean squared error (MSE) against number of outsourced samples $n$ for the proposed greedy algorithm, DS~\cite{iyer2012algorithms}, distorted greedy~\cite{harshaw2019submodular}, triage~\cite{raghu2019algorithmic} and CRR~\cite{bhatia2017consistent} on four real-world datasets.
In all cases, we use the LR model $h_{\theta}(\xb)$ and, for clarity, we explicitly highlight the performance under no automation and full automation. 
In Stare-H, Stare-D and Messidor, for each sample $(\xb, y)$, the element of $\alphab_{\xb, y}$ corresponding to the score $s = y$ has the highest value. For example, 
if $y = 0$, we set $\alpha_{\xb,y} = [6, 3, 2, 2, 1]$ for Panel (b), $\alpha_{\xb,y} = [6, 3, 2, 2, 1, 1]$ for Panel (c), and $\alphab_{\xb,y} = [6, 3, 1]$ for Panel (d).
The greedy algorithm outperforms the baselines 
%
%
across a majority of automation levels. For most automation levels, the competitive advantage provided by the greedy algorithm is statistically significant 
(Welch'{}s t-test, $p$-value $= 10^{-3}$). We omitted the DS algorithm for the Hatespeech dataset because it did not scale. 
%
}  
 \label{fig:quantrealone}
\end{figure*}
\vspace{5mm}
 \begin{figure}[!t]
\centering
\hspace*{0.5cm}{\includegraphics[width=0.6\textwidth]{legend_vary_test}\label{fig:ltr}\vspace{-0.2cm}}\\
\subfloat[Hatespeech]{\includegraphics[width=0.22\textwidth]{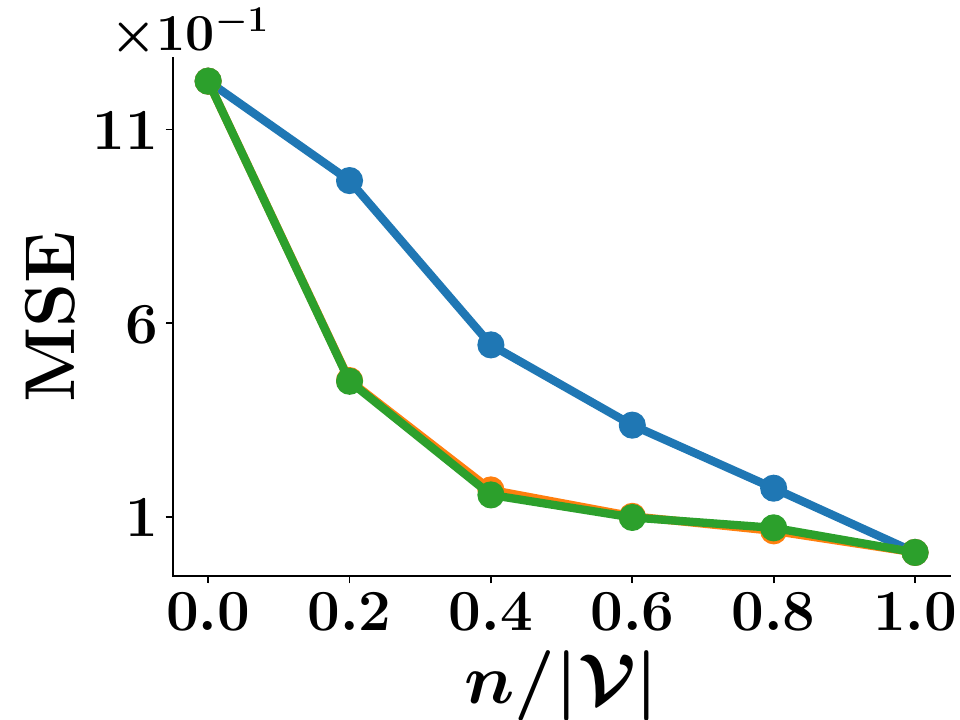}\label{fig:gtests}}\hspace*{0.2cm}
\subfloat[Stare-H]{\includegraphics[width=0.22\textwidth]{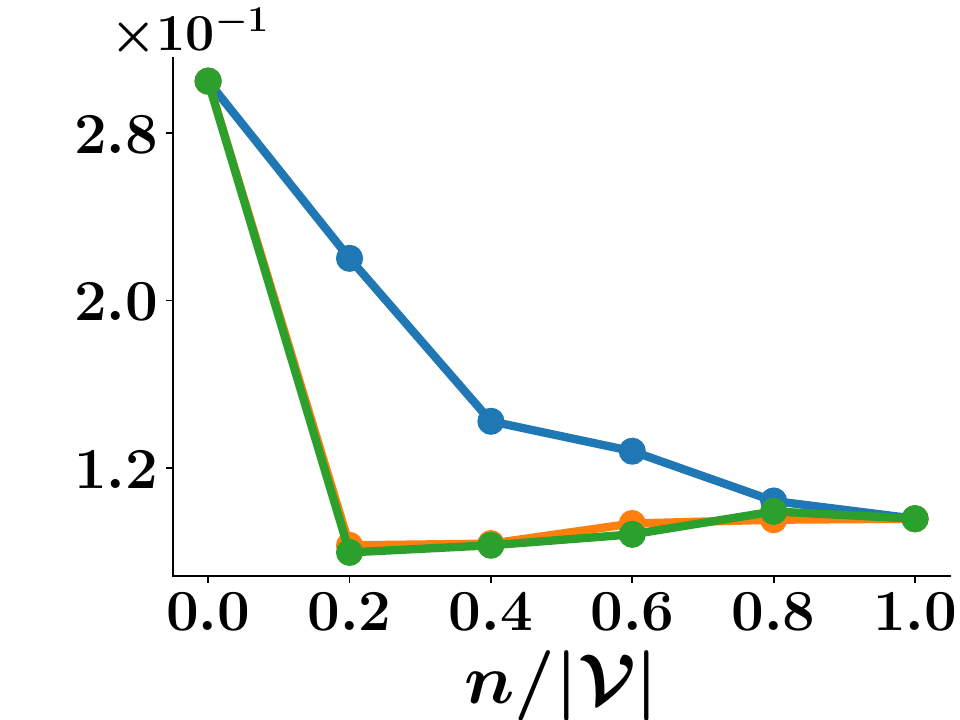}\label{fig:ltes1}} \hspace*{0.2cm}
\subfloat[Stare-D]{\includegraphics[width=0.22\textwidth]{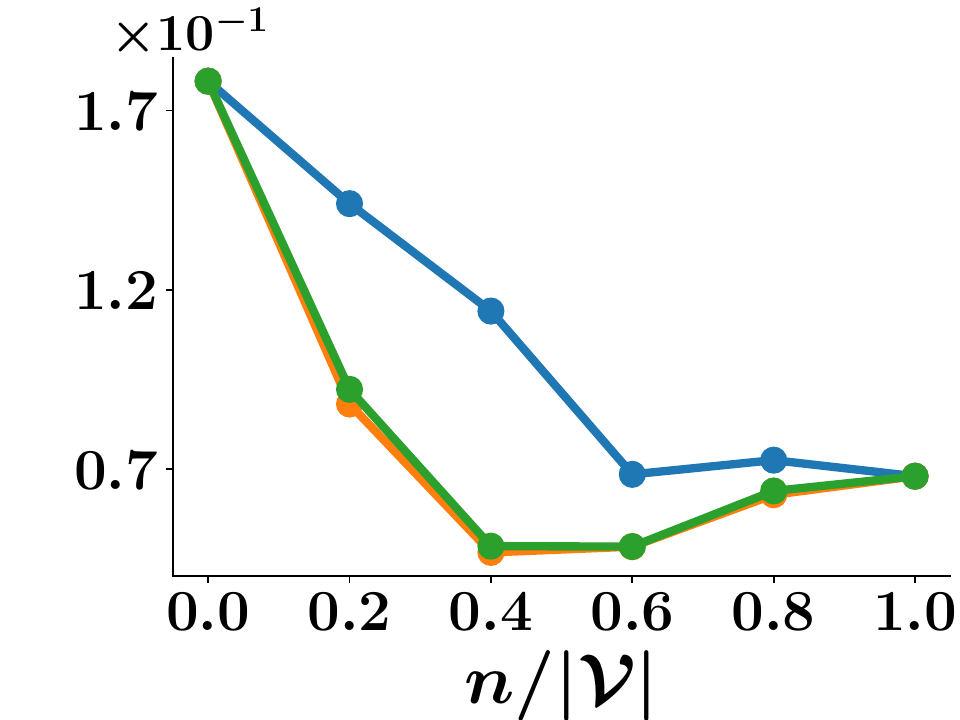}\label{fig:lte22}} \hspace*{0.2cm}
\subfloat[Messidor]{\includegraphics[width=0.22\textwidth]{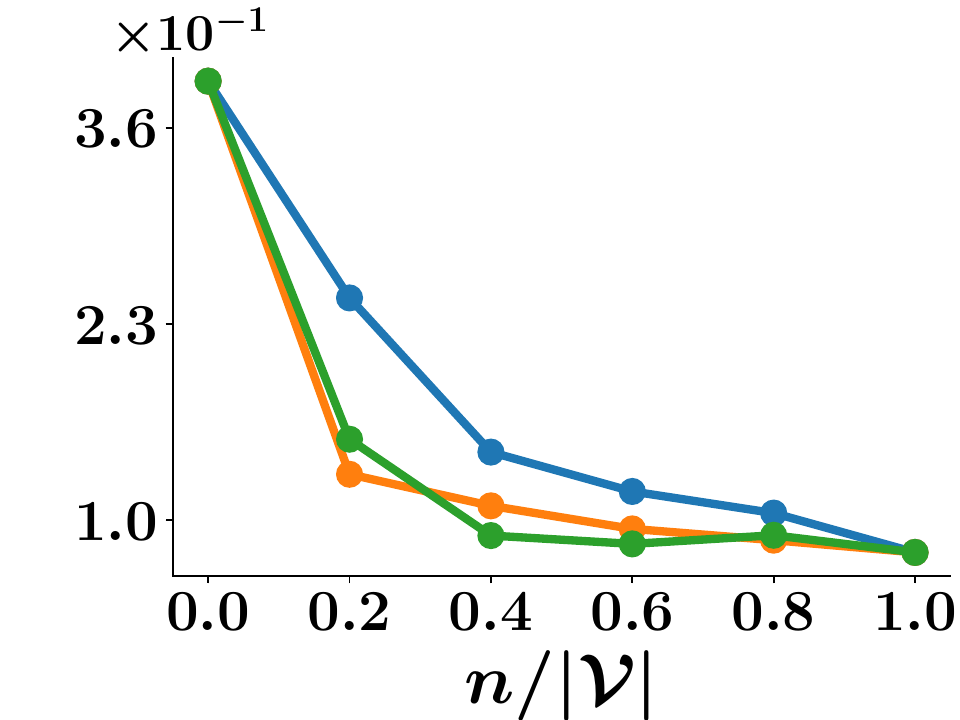}\label{fig:lte23}} 
\caption{Sensitivity analysis with respect to the choice of additional model $h_{\theta}(\xb)$. 
%
In all panels, we measure performance in terms of mean squared error (MSE) against number of outsourced samples $n$.
In Stare-H, Stare-D and Messidor, for each sample $(\xb, y)$, the element of $\alphab_{\xb, y}$ corresponding to the score $s = y$ has the highest value, similarly as 
in Figure~\ref{fig:quantrealone}. 
The results show that LR performs best, followed closely by MLP, and NN performs worst.
} 
\label{fig:hmodels-real}
\end{figure}

Second, we compare the performance of the greedy algorithm in terms of mean squared error (MSE) on a held-out set against the same 
competitive baselines used in the experiments on synthetic data.
%
%
Figure~\ref{fig:quantrealone} summarizes the results, where we use the LR model $h_{\theta}(\xb)$ for all methods. The results show that the greedy 
algorithm outperforms the baselines across a majority of automation levels.
Moreover, for most automation levels, the competitive advantage provided by the greedy algorithm is statistically significant (Welch'{}s t-test, $p$-value 
$= 10^{-3}$). Here, we note that error reaches the steady state around $n/|\Vcal|=0.3$. Hence, reducing the clinical workload as a factor of 2 still provides 
a similar performance to the case when all samples are outsourced to medical experts.
%
We obtained qualitatively similar results using the NN and MLP models $h_{\theta}(\xb)$ (refer to Appendix~\ref{app:h-real}). 
\begin{figure}[!t]
\centering
\hspace*{0.5cm}{\includegraphics[width=0.6\textwidth]{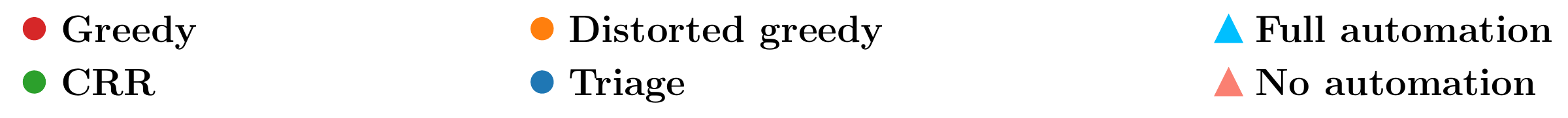}\label{fig:ltr}\vspace{-0.2cm}}\\

 \subfloat[Sensitivity with respect to the training set size $|\Vcal|$ on the Hatespeech dataset]{
        \scriptsize
	\stackunder[5pt]{\includegraphics[width=0.20\textwidth]{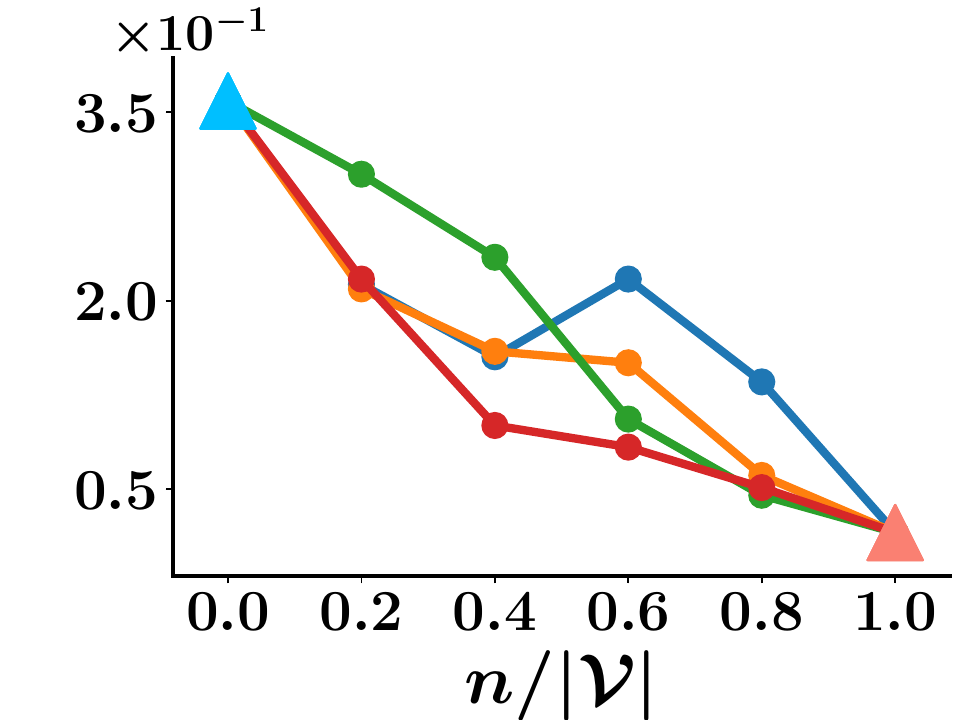}}{$|\Vcal| = 100$}
	\stackunder[5pt]{\includegraphics[width=0.20\textwidth]{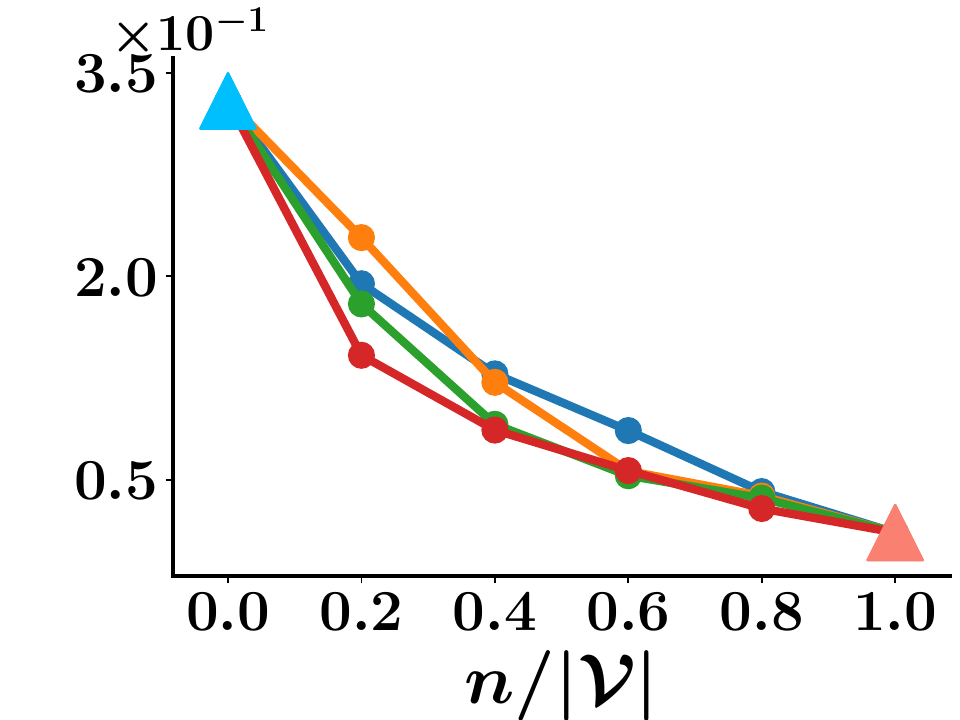}}{$|\Vcal| = 500$}
	\stackunder[5pt]{\includegraphics[width=0.20\textwidth]{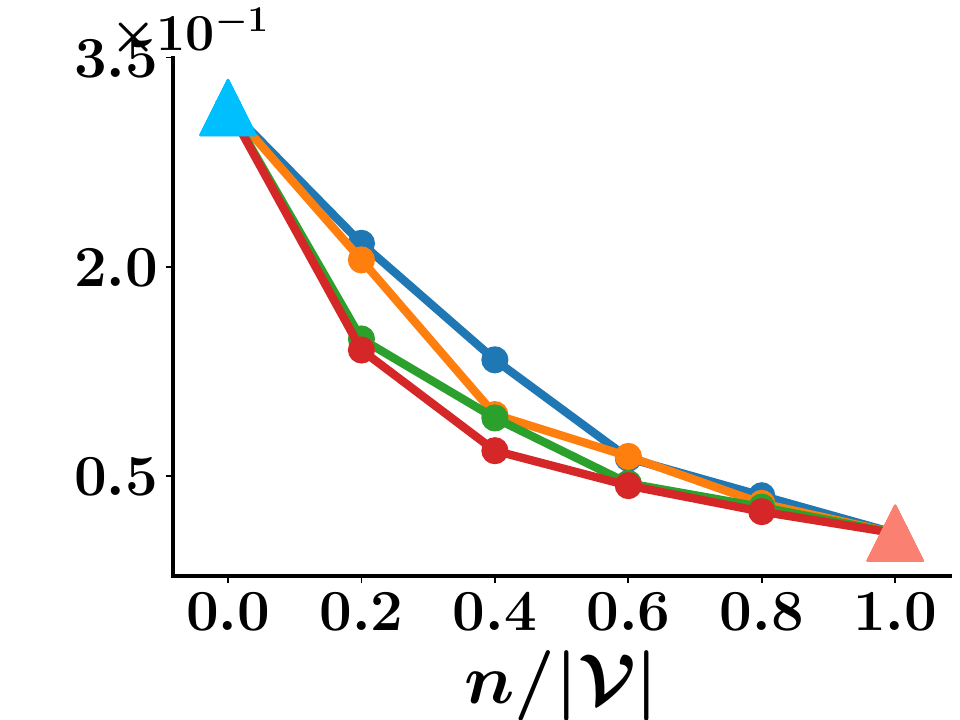}}{$|\Vcal| = 1000$}
	\stackunder[5pt]{\includegraphics[width=0.20\textwidth]{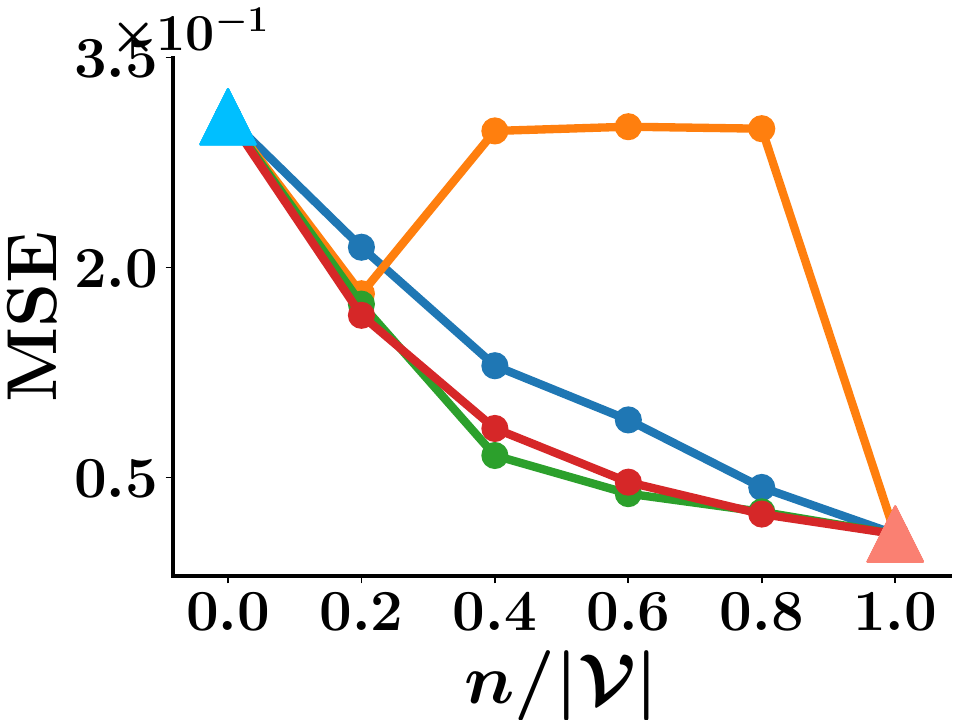}}{$|\Vcal| = 2000$}
	\stackunder[5pt]{\includegraphics[width=0.20\textwidth]{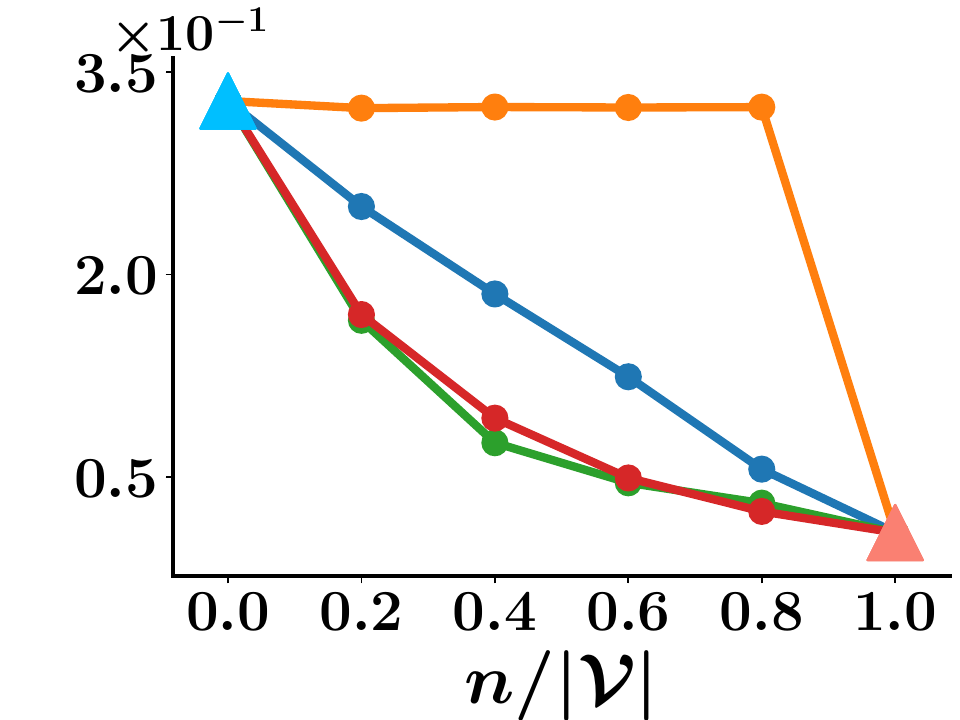}}{$|\Vcal| = 5000$}
   } \\
\vspace{3mm}

\subfloat[Sensitivity with respect to the output dimensions of the VGG net on the Messidor dataset]{
        \scriptsize
	\stackunder[5pt]{\includegraphics[width=0.20\textwidth]{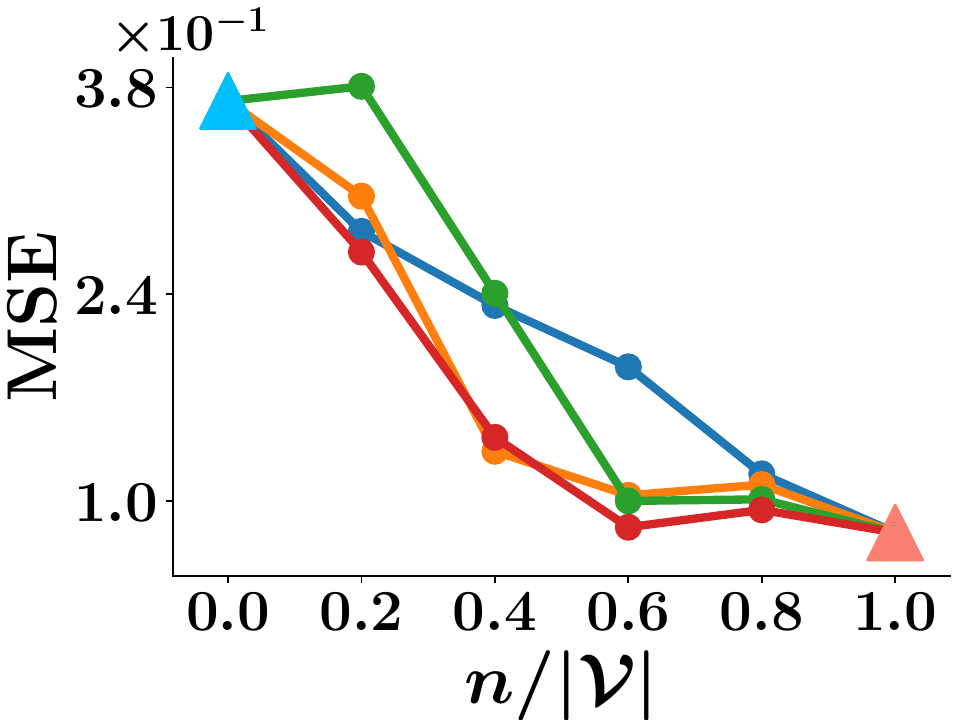}}{$m^\prime = 1024$}
	\stackunder[5pt]{\includegraphics[width=0.20\textwidth]{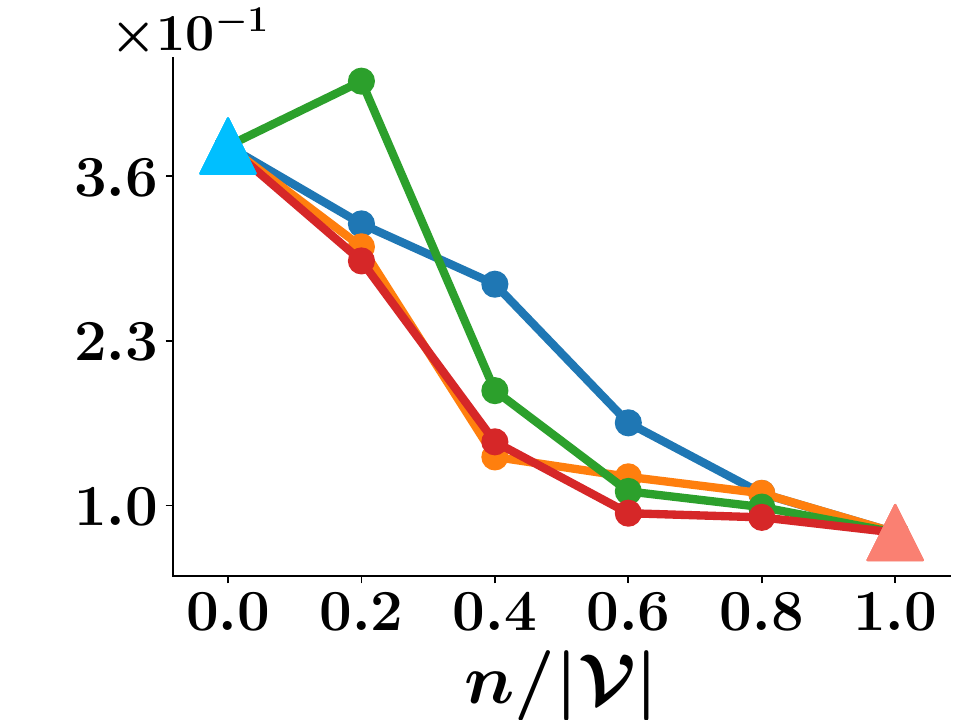}}{$m^\prime = 2048$}
	\stackunder[5pt]{\includegraphics[width=0.20\textwidth]{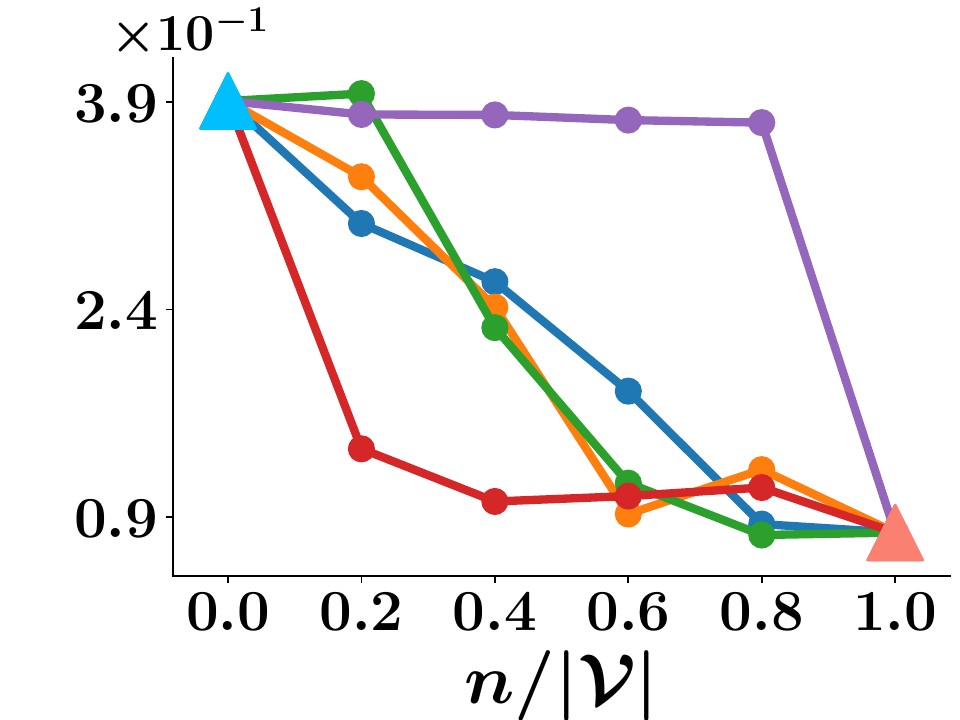}}{$m^\prime = 4096$}
	\stackunder[5pt]{\includegraphics[width=0.20\textwidth]{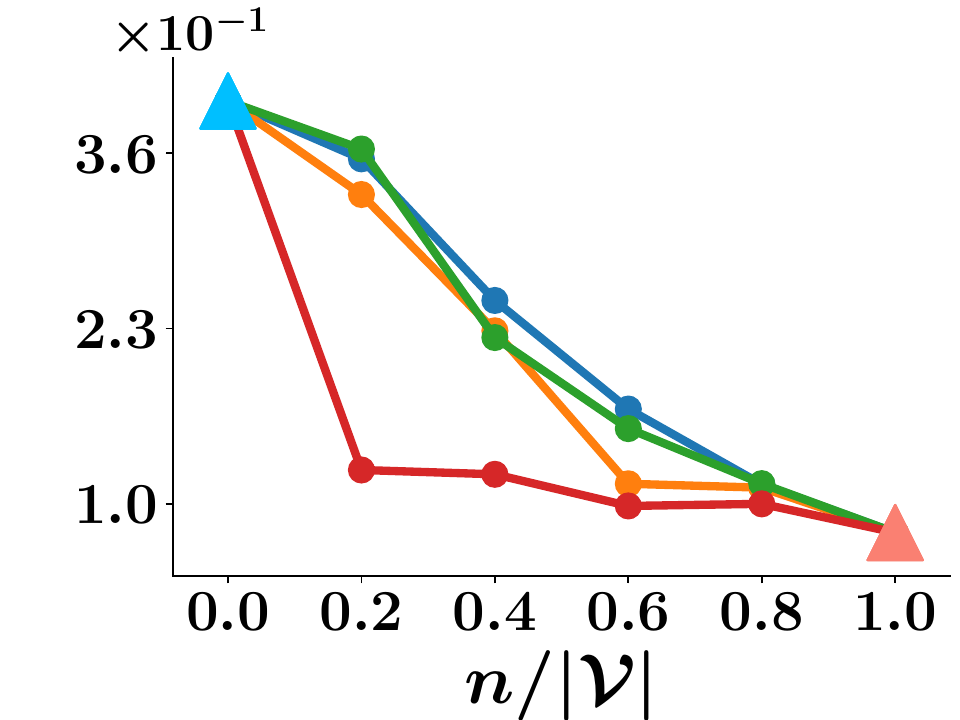}}{$m^\prime = 6144$}
	\stackunder[5pt]{\includegraphics[width=0.20\textwidth]{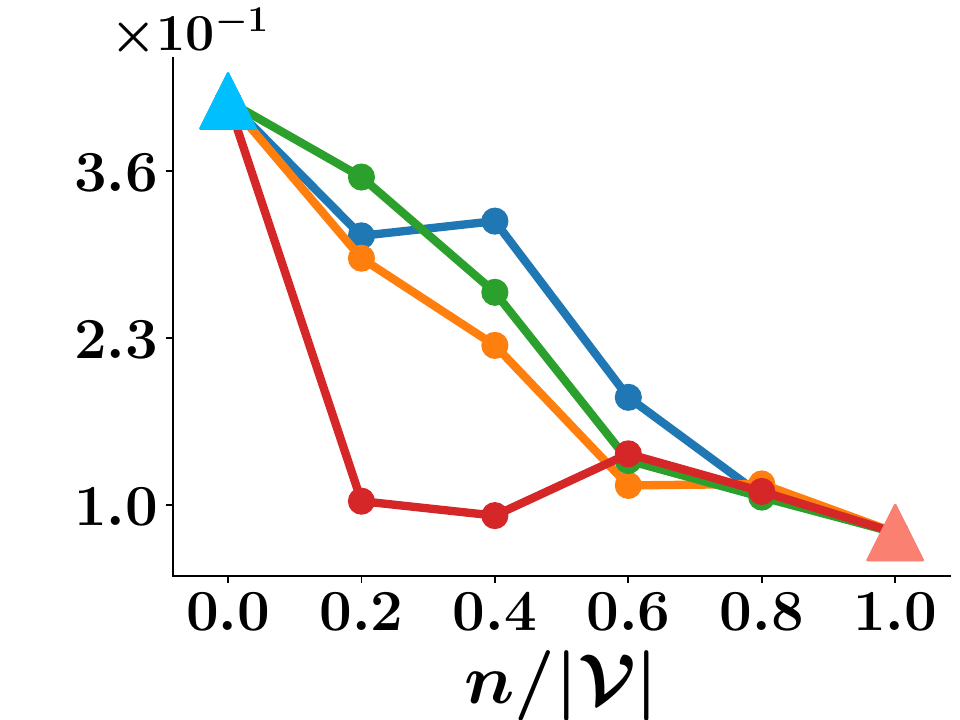}}{$m^\prime = 8192$}
   } \\
\vspace{3mm}
\subfloat[Sensitivity with respect to the output dimensions of PCA on the Messidor dataset]{
        \scriptsize
	\stackunder[5pt]{\includegraphics[width=0.20\textwidth]{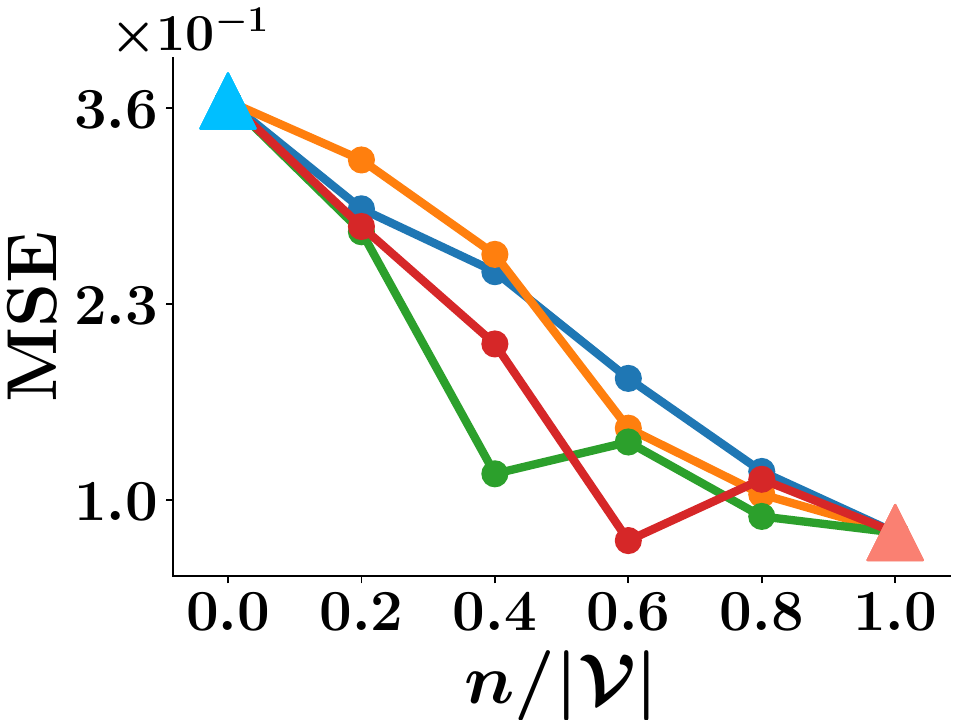}}{$m = 25$}
	\stackunder[5pt]{\includegraphics[width=0.20\textwidth]{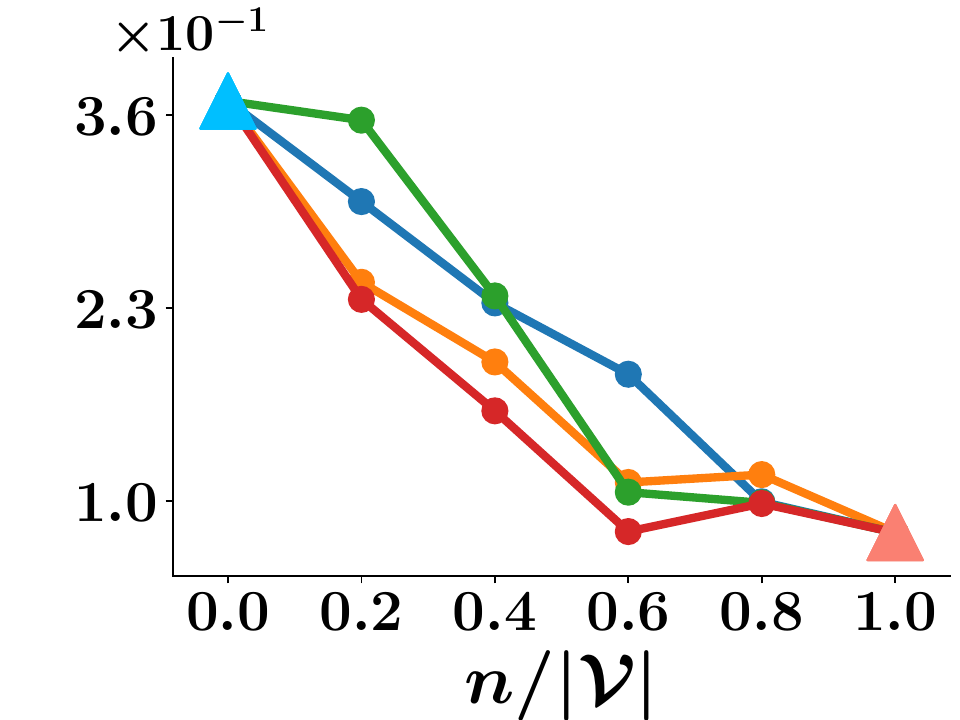}}{$m = 50$}
	\stackunder[5pt]{\includegraphics[width=0.20\textwidth]{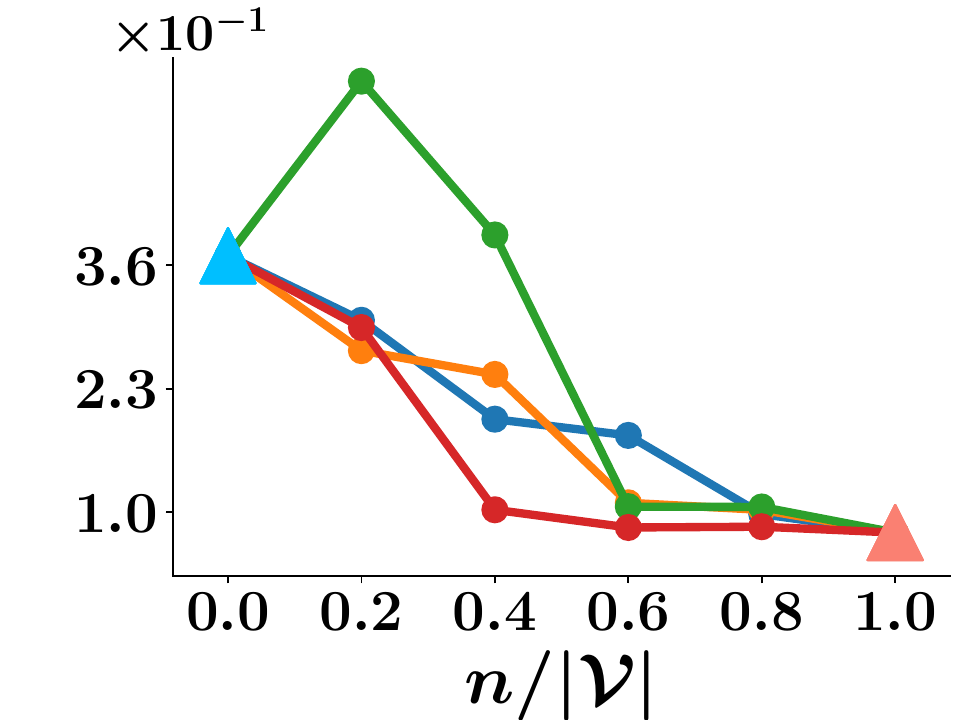}}{$m = 75$}
	\stackunder[5pt]{\includegraphics[width=0.20\textwidth]{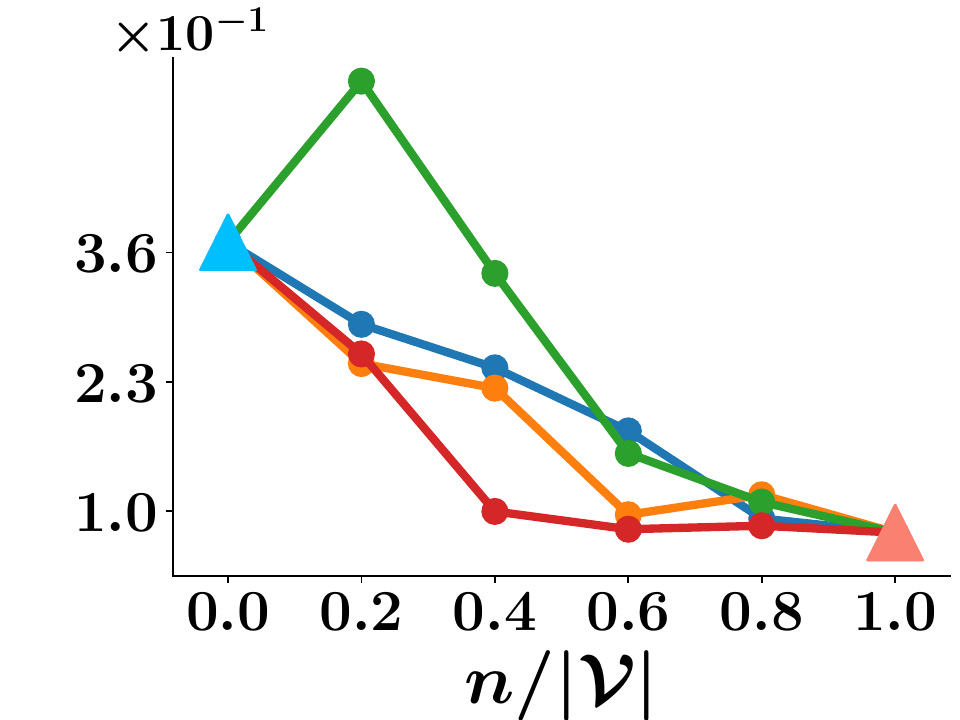}}{$m = 100$}
	\stackunder[5pt]{\includegraphics[width=0.20\textwidth]{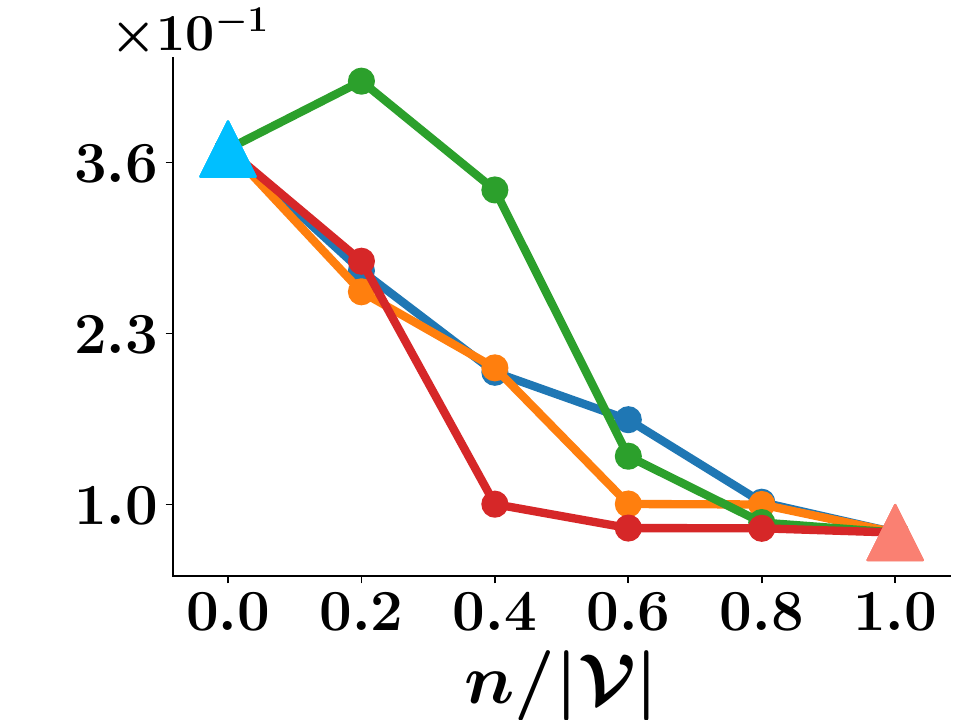}}{$m = 125$}
   }
\caption{Sensitivity analysis with respect to the hyperparameters used in our experiments. 
In all panels, we measure performance in terms of mean squared error (MSE) against number of outsourced samples $n$ and use a MLP model $h_{\theta}(\xb)$.
We found qualitatively similar results for the LR and NN models $h_{\theta}(\xb)$.
}\vspace{-2mm}
\label{fig:sensitivity-hyperparameters}
\end{figure}

Third, we perform a detailed sensitivity analysis to evaluate the robustness of our algorithm with respect to the choice of additional model $h_{\theta}(\xb)$ as 
well as with respect to several hyperparameters used in our experiments.
In terms of the choice of additional model, Figure~\ref{fig:hmodels-real} summarizes the results, which show that, in contrast with the experiments on synthetic 
data (refer to Figure~\ref{fig:hmodels}), LR performs best, followed closely by MLP, and NN performs worst.
In terms of hyperparameters used in our experiments, we analyze the sensitivity with respect to: (i) $|\Vcal|$, \ie, the size of the ground set in Hatespeech dataset; 
(ii) the dimension ($m^\prime$) of the intermediate embedding provided by VGG; and, (iii) the dimension ($m$) of the output features provided by the PCA for the
Messidor dataset. 
Figure~\ref{fig:sensitivity-hyperparameters} summarizes the results, which show that our algorithm outperforms the baselines for most of the hyperparameter values.
We did perform further experiments on sensitivity, \eg, sensitivity with respect to $m$ for the Hatespeech dataset, however, we we do not report them since we obtained 
qualitatively similar results.

\section{Conclusions}
\label{sec:conclusions}
%
%
In this paper, we have initiated the development of machine learning models that are optimized to operate under different automation 
levels.
We have focused on ridge regression under human assistance and have shown that a simple greedy algorithm is able to find a solution with 
nontrivial approximation guarantees.
Using both synthetic and real-world data, we have shown that this greedy algorithm has the ability to learn the underlying relationship 
between a given sample and its corresponding human and machine error, it outperforms several competitive algorithms, and it is robust with respect
to several design choices and hyperparameters used in the experiments. 
Moreover, we have also shown that humans and machines working together 
can achieve a considerably better performance than each of them would achieve on their own.

Our work also opens many interesting venues for future work.
For example, it would be very interesting to advance the development of other more sophisticated machine learning models, both for regression and 
classification, under different automation levels.
It would be helpful to find tighter lower bounds on the parameter $\alpha$, which better characterize the good empirical performance.
It would be very interesting to study sequential decision making scenarios under human assistance in a variety of scenarios, \eg, autonomous driving under 
different automation levels~\cite{meresht2020learning}.
In our particular problem setting, the computational cost of the greedy algorithm is quadratic on the size of the training set. To lower this cost down, one can think
of adapting highly efficient streaming algorithms for maximizing weakly submodular functions~\cite{elenberg2017streaming}.
In our work, we have assumed that we know the human error for every training sample. It would be interesting to tackle the problem in an online learning setting,
where one needs to balance exploration, \eg, the estimation of the human error, and exploitation, \eg, low error over time.
Finally, it would be interesting to assess the performance of ridge regression under human assistance using interventional experiments on a real-world application.

\bibliographystyle{plainnat}
\bibliography{refs} 

\begin{thebibliography}{56}
\providecommand{\natexlab}[1]{#1}
\providecommand{\url}[1]{\texttt{#1}}
\expandafter\ifx\csname urlstyle\endcsname\relax
  \providecommand{\doi}[1]{doi: #1}\else
  \providecommand{\doi}{doi: \begingroup \urlstyle{rm}\Url}\fi

\bibitem[Bartlett and Wegkamp(2008)]{bartlett2008classification}
Peter Bartlett and Marten Wegkamp.
\newblock Classification with a reject option using a hinge loss.
\newblock \emph{JMLR}, 9\penalty0 (Aug):\penalty0 1823--1840, 2008.

\bibitem[Bhatia et~al.(2017)Bhatia, Jain, Kamalaruban, and
  Kar]{bhatia2017consistent}
Kush Bhatia, Prateek Jain, Parameswaran Kamalaruban, and Purushottam Kar.
\newblock Consistent robust regression.
\newblock In \emph{NeurIPS}, pages 2110--2119, 2017.

\bibitem[Bogunovic et~al.(2018)Bogunovic, Zhao, and
  Cevher]{bogunovic2018robust}
Ilija Bogunovic, Junyao Zhao, and Volkan Cevher.
\newblock Robust maximization of non-submodular objectives.
\newblock \emph{arXiv preprint arXiv:1802.07073}, 2018.

\bibitem[Boyd et~al.(1994)Boyd, El~Ghaoui, Feron, and Balakrishnan]{lmibook}
Stephen Boyd, Laurent El~Ghaoui, Eric Feron, and Venkataramanan Balakrishnan.
\newblock \emph{Linear matrix inequalities in system and control theory},
  volume~15.
\newblock Siam, 1994.

\bibitem[Chen and Price(2017)]{chen2017active}
Xue Chen and Eric Price.
\newblock Active regression via linear-sample sparsification.
\newblock \emph{arXiv preprint arXiv:1711.10051}, 2017.

\bibitem[Cheng et~al.(2015)Cheng, Danescu-Niculescu-Mizil, and
  Leskovec]{cheng2015antisocial}
Justin Cheng, Cristian Danescu-Niculescu-Mizil, and Jure Leskovec.
\newblock Antisocial behavior in online discussion communities.
\newblock In \emph{ICWSM}, 2015.

\bibitem[Cohn et~al.(1995)Cohn, Ghahramani, and Jordan]{cohn1995active}
David Cohn, Zoubin Ghahramani, and Michael Jordan.
\newblock Active learning with statistical models.
\newblock In \emph{NeurIPS}, pages 705--712, 1995.

\bibitem[Cortes et~al.(2016)Cortes, DeSalvo, and Mohri]{cortes2016learning}
Corinna Cortes, Giulia DeSalvo, and Mehryar Mohri.
\newblock Learning with rejection.
\newblock In \emph{ALT}, pages 67--82. Springer, 2016.

\bibitem[Davidson et~al.()Davidson, Warmsley, Macy, and Weber]{hateoffensive}
Thomas Davidson, Dana Warmsley, Michael Macy, and Ingmar Weber.
\newblock Automated hate speech detection and the problem of offensive
  language.
\newblock ICWSM, pages 512--515.

\bibitem[De et~al.(2020)De, Koley, Ganguly, and Gomez-Rodriguez]{de2020aaai}
Abir De, Paramita Koley, Niloy Ganguly, and Manuel Gomez-Rodriguez.
\newblock Regression under human assistance.
\newblock In \emph{AAAI}, 2020.

\bibitem[Decencière et~al.(2014)Decencière, Zhang, Cazuguel, Lay, Cochener,
  Trone, Gain, Ordonez, Massin, Erginay, Charton, and
  Klein]{decenciere_feedback_2014}
Etienne Decencière, Xiwei Zhang, Guy Cazuguel, Bruno Lay, Béatrice Cochener,
  Caroline Trone, Philippe Gain, Richard Ordonez, Pascale Massin, Ali Erginay,
  Béatrice Charton, and Jean-Claude Klein.
\newblock Feedback on a publicly distributed database: the messidor database.
\newblock \emph{Image Analysis \& Stereology}, 33\penalty0 (3):\penalty0
  231--234, August 2014.

\bibitem[Elenberg et~al.(2017)Elenberg, Dimakis, Feldman, and
  Karbasi]{elenberg2017streaming}
Ethan Elenberg, Alexandros~G Dimakis, Moran Feldman, and Amin Karbasi.
\newblock Streaming weak submodularity: Interpreting neural networks on the
  fly.
\newblock In \emph{Advances in Neural Information Processing Systems}, pages
  4044--4054, 2017.

\bibitem[Everett and Roberts(2018)]{everett2018learning}
Richard Everett and Stephen Roberts.
\newblock Learning against non-stationary agents with opponent modelling and
  deep reinforcement learning.
\newblock In \emph{2018 AAAI Spring Symposium Series}, 2018.

\bibitem[Feng et~al.(2014)Feng, Xu, Mannor, and Yan]{feng2014robust}
Jiashi Feng, Huan Xu, Shie Mannor, and Shuicheng Yan.
\newblock Robust logistic regression and classification.
\newblock In \emph{NeurIPS}, pages 253--261, 2014.

\bibitem[Gatmiry and Gomez-Rodriguez(2019)]{gatmiry2019}
Khashayar Gatmiry and Manuel Gomez-Rodriguez.
\newblock Non-submodular function maximization subject to a matroid constraint,
  with applications.
\newblock \emph{arXiv preprint arXiv:1811.07863}, 2019.

\bibitem[Geifman and El-Yaniv(2019)]{geifman2019selectivenet}
Yonatan Geifman and Ran El-Yaniv.
\newblock Selectivenet: A deep neural network with an integrated reject option.
\newblock \emph{arXiv preprint arXiv:1901.09192}, 2019.

\bibitem[Geifman et~al.(2018)Geifman, Uziel, and El-Yaniv]{geifman2018bias}
Yonatan Geifman, Guy Uziel, and Ran El-Yaniv.
\newblock Bias-reduced uncertainty estimation for deep neural classifiers.
\newblock 2018.

\bibitem[Ghosh et~al.(2019)Ghosh, Tschiatschek, Mahdavi, and
  Singla]{ghosh2020towards}
Ahana Ghosh, Sebastian Tschiatschek, Hamed Mahdavi, and Adish Singla.
\newblock Towards deployment of robust ai agents for human-machine
  partnerships.
\newblock In \emph{AAMAS}, 2019.

\bibitem[Grover et~al.(2018)Grover, Al-Shedivat, Gupta, Burda, and
  Edwards]{grover2018learning}
Aditya Grover, Maruan Al-Shedivat, Jayesh~K Gupta, Yura Burda, and Harrison
  Edwards.
\newblock Learning policy representations in multiagent systems.
\newblock In \emph{ICML}, 2018.

\bibitem[Guo and Schuurmans(2008)]{guo2008discriminative}
Yuhong Guo and Dale Schuurmans.
\newblock Discriminative batch mode active learning.
\newblock In \emph{NeurIPS}, pages 593--600, 2008.

\bibitem[Hadfield-Menell et~al.(2016)Hadfield-Menell, Russell, Abbeel, and
  Dragan]{hadfield2016cooperative}
Dylan Hadfield-Menell, Stuart~J Russell, Pieter Abbeel, and Anca Dragan.
\newblock Cooperative inverse reinforcement learning.
\newblock In \emph{Advances in neural information processing systems}, pages
  3909--3917, 2016.

\bibitem[Harshaw et~al.(2019)Harshaw, Feldman, Ward, and
  Karbasi]{harshaw2019submodular}
Christopher Harshaw, Moran Feldman, Justin Ward, and Amin Karbasi.
\newblock Submodular maximization beyond non-negativity: Guarantees, fast
  algorithms, and applications.
\newblock \emph{arXiv preprint arXiv:1904.09354}, 2019.

\bibitem[Hashemi et~al.(2019)Hashemi, Ghasemi, Vikalo, and
  Topcu]{hashemi2019submodular}
Abolfazl Hashemi, Mahsa Ghasemi, Haris Vikalo, and Ufuk Topcu.
\newblock Submodular observation selection and information gathering for
  quadratic models.
\newblock \emph{arXiv preprint arXiv:1905.09919}, 2019.

\bibitem[Haug et~al.(2018)Haug, Tschiatschek, and Singla]{haug2018teaching}
Luis Haug, Sebastian Tschiatschek, and Adish Singla.
\newblock Teaching inverse reinforcement learners via features and
  demonstrations.
\newblock In \emph{Advances in Neural Information Processing Systems}, pages
  8464--8473, 2018.

\bibitem[He et~al.(2016)He, Zhang, Ren, and Sun]{resnet}
Kaiming He, Xiangyu Zhang, Shaoqing Ren, and Jian Sun.
\newblock Deep residual learning for image recognition.
\newblock In \emph{CVPR}, 2016.

\bibitem[Hoi et~al.(2006)Hoi, Jin, Zhu, and Lyu]{hoi2006batch}
Steven Hoi, Rong Jin, Jianke Zhu, and Michael~R Lyu.
\newblock Batch mode active learning and its application to medical image
  classification.
\newblock In \emph{ICML}, 2006.

\bibitem[Hoover et~al.(2000)Hoover, Kouznetsova, and
  Goldbaum]{hoover2000locating}
Adam Hoover, Valentina Kouznetsova, and Michael Goldbaum.
\newblock Locating blood vessels in retinal images by piecewise threshold
  probing of a matched filter response.
\newblock \emph{IEEE Transactions on Medical imaging}, 19\penalty0
  (3):\penalty0 203--210, 2000.

\bibitem[Iyer and Bilmes(2012)]{iyer2012algorithms}
Rishabh Iyer and Jeff Bilmes.
\newblock Algorithms for approximate minimization of the difference between
  submodular functions, with applications.
\newblock \emph{arXiv preprint arXiv:1207.0560}, 2012.

\bibitem[Joulin et~al.(2016)Joulin, Grave, Bojanowski, Douze, J{\'e}gou, and
  Mikolov]{ft}
Armand Joulin, Edouard Grave, Piotr Bojanowski, Matthijs Douze, H{\'e}rve
  J{\'e}gou, and Tomas Mikolov.
\newblock Fasttext. zip: Compressing text classification models.
\newblock \emph{arXiv preprint arXiv:1612.03651}, 2016.

\bibitem[Kamalaruban et~al.(2019)Kamalaruban, Devidze, Cevher, and
  Singla]{kamalaruban2019interactive}
Parameswaran Kamalaruban, Rati Devidze, Volkan Cevher, and Adish Singla.
\newblock Interactive teaching algorithms for inverse reinforcement learning.
\newblock In \emph{IJCAI}, 2019.

\bibitem[Lehmann et~al.(2006)Lehmann, Lehmann, and
  Nisan]{lehmann2006combinatorial}
Benny Lehmann, Daniel Lehmann, and Noam Nisan.
\newblock Combinatorial auctions with decreasing marginal utilities.
\newblock \emph{Games and Economic Behavior}, 55\penalty0 (2):\penalty0
  270--296, 2006.

\bibitem[Liu et~al.(2019)Liu, Wang, Liang, Salakhutdinov, Morency, and
  Ueda]{liu2019deep}
Ziyin Liu, Zhikang Wang, Paul~Pu Liang, Russ~R Salakhutdinov, Louis-Philippe
  Morency, and Masahito Ueda.
\newblock Deep gamblers: Learning to abstain with portfolio theory.
\newblock In \emph{NeurIPS}, 2019.

\bibitem[Macindoe et~al.(2012)Macindoe, Kaelbling, and
  Lozano-P{\'e}rez]{macindoe2012pomcop}
Owen Macindoe, Leslie~Pack Kaelbling, and Tom{\'a}s Lozano-P{\'e}rez.
\newblock Pomcop: Belief space planning for sidekicks in cooperative games.
\newblock In \emph{Eighth Artificial Intelligence and Interactive Digital
  Entertainment Conference}, 2012.

\bibitem[Meresht et~al.(2020)Meresht, De, Singla, and
  Gomez-Rodriguez]{meresht2020learning}
Vahid~Balazadeh Meresht, Abir De, Adish Singla, and Manuel Gomez-Rodriguez.
\newblock Learning to switch between machines and humans.
\newblock \emph{arXiv preprint arXiv:2002.04258}, 2020.

\bibitem[Nikolaidis et~al.(2015)Nikolaidis, Ramakrishnan, Gu, and
  Shah]{nikolaidis2015efficient}
Stefanos Nikolaidis, Ramya Ramakrishnan, Keren Gu, and Julie Shah.
\newblock Efficient model learning from joint-action demonstrations for
  human-robot collaborative tasks.
\newblock In \emph{2015 10th ACM/IEEE International Conference on Human-Robot
  Interaction (HRI)}, pages 189--196. IEEE, 2015.

\bibitem[Nikolaidis et~al.(2017)Nikolaidis, Forlizzi, Hsu, Shah, and
  Srinivasa]{nikolaidis2017mathematical}
Stefanos Nikolaidis, Jodi Forlizzi, David Hsu, Julie Shah, and Siddhartha
  Srinivasa.
\newblock Mathematical models of adaptation in human-robot collaboration.
\newblock \emph{arXiv preprint arXiv:1707.02586}, 2017.

\bibitem[Pradel and Sen(2018)]{pradel2018deepbugs}
Michael Pradel and Koushik Sen.
\newblock Deepbugs: A learning approach to name-based bug detection.
\newblock \emph{Proceedings of the ACM on Programming Languages}, 2\penalty0
  (OOPSLA):\penalty0 147, 2018.

\bibitem[Radanovic et~al.(2019)Radanovic, Devidze, Parkes, and
  Singla]{radanovic2019learning}
Goran Radanovic, Rati Devidze, David Parkes, and Adish Singla.
\newblock Learning to collaborate in markov decision processes.
\newblock In \emph{ICML}, 2019.

\bibitem[Raghu et~al.(2019{\natexlab{a}})Raghu, Blumer, Corrado, Kleinberg,
  Obermeyer, and Mullainathan]{raghu2019algorithmic}
Maithra Raghu, Katy Blumer, Greg Corrado, Jon Kleinberg, Ziad Obermeyer, and
  Sendhil Mullainathan.
\newblock The algorithmic automation problem: Prediction, triage, and human
  effort.
\newblock \emph{arXiv preprint arXiv:1903.12220}, 2019{\natexlab{a}}.

\bibitem[Raghu et~al.(2019{\natexlab{b}})Raghu, Blumer, Sayres, Obermeyer,
  Kleinberg, Mullainathan, and Kleinberg]{raghu2019direct}
Maithra Raghu, Katy Blumer, Rory Sayres, Ziad Obermeyer, Bobby Kleinberg,
  Sendhil Mullainathan, and Jon Kleinberg.
\newblock Direct uncertainty prediction for medical second opinions.
\newblock In \emph{ICML}, 2019{\natexlab{b}}.

\bibitem[Ramaswamy et~al.(2018)Ramaswamy, Tewari, and
  Agarwal]{ramaswamy2018consistent}
Harish Ramaswamy, Ambuj Tewari, and Shivani Agarwal.
\newblock Consistent algorithms for multiclass classification with an abstain
  option.
\newblock \emph{Electronic J. of Statistics}, 12\penalty0 (1):\penalty0
  530--554, 2018.

\bibitem[Sabato and Munos(2014)]{sabato2014active}
Sivan Sabato and Remi Munos.
\newblock Active regression by stratification.
\newblock In \emph{NeurIPS}, pages 469--477, 2014.

\bibitem[Sherman and Morrison(1950)]{sherman1950adjustment}
Jack Sherman and Winifred~J Morrison.
\newblock Adjustment of an inverse matrix corresponding to a change in one
  element of a given matrix.
\newblock \emph{The Annals of Mathematical Statistics}, 21\penalty0
  (1):\penalty0 124--127, 1950.

\bibitem[Simonyan and Zisserman(2014)]{simonyan2014very}
Karen Simonyan and Andrew Zisserman.
\newblock Very deep convolutional networks for large-scale image recognition.
\newblock \emph{arXiv preprint arXiv:1409.1556}, 2014.

\bibitem[Studer et~al.(2011)Studer, Kuppinger, Pope, and
  Bolcskei]{studer2011recovery}
Christoph Studer, Patrick Kuppinger, Graeme Pope, and Helmut Bolcskei.
\newblock Recovery of sparsely corrupted signals.
\newblock \emph{IEEE Transactions on Information Theory}, 58\penalty0
  (5):\penalty0 3115--3130, 2011.

\bibitem[Suggala et~al.(2019)Suggala, Bhatia, Ravikumar, and
  Jain]{suggala2019adaptive}
Arun~Sai Suggala, Kush Bhatia, Pradeep Ravikumar, and Prateek Jain.
\newblock Adaptive hard thresholding for near-optimal consistent robust
  regression.
\newblock \emph{arXiv preprint arXiv:1903.08192}, 2019.

\bibitem[Sugiyama(2006)]{sugiyama2006active}
Masashi Sugiyama.
\newblock Active learning in approximately linear regression based on
  conditional expectation of generalization error.
\newblock \emph{JMLR}, 7\penalty0 (Jan):\penalty0 141--166, 2006.

\bibitem[Thulasidasan et~al.(2019)Thulasidasan, Bhattacharya, Bilmes,
  Chennupati, and Mohd-Yusof]{thulasidasan2019combating}
Sunil Thulasidasan, Tanmoy Bhattacharya, Jeff Bilmes, Gopinath Chennupati, and
  Jamal Mohd-Yusof.
\newblock Combating label noise in deep learning using abstention.
\newblock \emph{arXiv preprint arXiv:1905.10964}, 2019.

\bibitem[Topol(2019)]{topol2019high}
Eric Topol.
\newblock High-performance medicine: the convergence of human and artificial
  intelligence.
\newblock \emph{Nature medicine}, 25\penalty0 (1):\penalty0 44, 2019.

\bibitem[Tsakonas et~al.(2014)Tsakonas, Jald{\'e}n, Sidiropoulos, and
  Ottersten]{tsakonas2014convergence}
Efthymios Tsakonas, Joakim Jald{\'e}n, Nicholas~D Sidiropoulos, and Bj{\"o}rn
  Ottersten.
\newblock Convergence of the huber regression m-estimate in the presence of
  dense outliers.
\newblock \emph{IEEE Signal Processing Letters}, 21\penalty0 (10):\penalty0
  1211--1214, 2014.

\bibitem[Tschiatschek et~al.(2019)Tschiatschek, Ghosh, Haug, Devidze, and
  Singla]{tschiatschek2019learner}
Sebastian Tschiatschek, Ahana Ghosh, Luis Haug, Rati Devidze, and Adish Singla.
\newblock Learner-aware teaching: Inverse reinforcement learning with
  preferences and constraints.
\newblock In \emph{Advances in Neural Information Processing Systems}, pages
  4147--4157, 2019.

\bibitem[Willett et~al.(2006)Willett, Nowak, and Castro]{willett2006faster}
Rebecca Willett, Robert Nowak, and Rui~M Castro.
\newblock Faster rates in regression via active learning.
\newblock In \emph{NeurIPS}, 2006.

\bibitem[Wilson and Daugherty(2018)]{wilson2018collaborative}
H~James Wilson and Paul~R Daugherty.
\newblock Collaborative intelligence: humans and ai are joining forces.
\newblock \emph{Harvard Business Review}, 96\penalty0 (4):\penalty0 114--123,
  2018.

\bibitem[Wright and Ma(2010)]{wright2010dense}
John Wright and Yi~Ma.
\newblock Dense error correction via $\ell^{1}$-minimization.
\newblock \emph{IEEE Transactions on Information Theory}, 56\penalty0
  (7):\penalty0 3540--3560, 2010.

\bibitem[Zheng et~al.(2018)Zheng, Meng, Hao, Zhang, Yang, and
  Fan]{zheng2018deep}
Yan Zheng, Zhaopeng Meng, Jianye Hao, Zongzhang Zhang, Tianpei Yang, and
  Changjie Fan.
\newblock A deep bayesian policy reuse approach against non-stationary agents.
\newblock In \emph{Advances in Neural Information Processing Systems}, pages
  954--964, 2018.

\bibitem[Ziyin et~al.(2020)Ziyin, Chen, Wang, Liang, Salakhutdinov, Morency,
  and Ueda]{ziyin2020learning}
Liu Ziyin, Blair Chen, Ru~Wang, Paul~Pu Liang, Ruslan Salakhutdinov,
  Louis-Philippe Morency, and Masahito Ueda.
\newblock Learning not to learn in the presence of noisy labels.
\newblock \emph{arXiv preprint arXiv:2002.06541}, 2020.

\end{thebibliography}
\clearpage
\newpage

\appendix
\section{Auxiliary Lemmas and Propositions} \label{app:auxiliary}

\begin{proposition}\label{prop1}
Assume $c(\xb_k,y_k)\le \gamma y^2 _k$ and $\lambda\ge \frac{\gamma}{1-\gamma}  || \xb_k || _2 ^2$. Then,
it holds that
\begin{equation*}
\left[\begin{array}{cc} y^2  _{k}   -c(\xb_k,y_k)  &      y_k  \xb^\top _k  \\  \xb_k y_k &  \lambda \II  +\xb_k  \xb^\top _k  
 \end{array}\right] \mge  0
 \end{equation*} 
\end{proposition}
\begin{proof}
We use the Schur complement property for positive-definiteness~\cite[Page 8]{lmibook}\\
on the matrix $ \left[\begin{array}{cc} y^2  _{k}   -c(\xb_k,y_k)  &      y_k  \xb^\top _k  \\  \xb_k y_k &  \lambda \II  +\xb_k  \xb^\top _k  
 \end{array}\right]$, \ie,
 \begin{equation*}
\lambda \II +\xb _k \xb ^\top _k -\xb_k \xb^\top  _k (y^2 _k /(y^2 _k -c(\xb_k,y_k))) \mge  \lambda \II-\frac{\gamma}{1-\gamma} \xb _k \xb^\top _k.
 \end{equation*} 
Given that $\xb_k \xb^\top _k$ is a rank one matrix, it has only one non-zero eigenvalue.  Hence it is same as $\tr (\xb_k \xb^\top _k)=||\xb _k||^2 _2$, which along with the assumed bound on $\lambda$ proves that 
\begin{equation*}
\lambda \II +\xb _k \xb ^\top _k -\xb_k \xb^\top  _k (y^2 _k /(y^2 _k -c(\xb_k,y_k))) \mge 0. 
\end{equation*}
Then, from the Schur complement method, 
we have that
\begin{equation*}
\left[\begin{array}{cc} y^2  _{k}   -c(\xb_k,y_k)  &      y_k  \xb^\top _k  \\  \xb_k y_k &  \lambda \II  +\xb_k  \xb^\top _k  
 \end{array}\right] \mge  0.
\end{equation*}
\vspace{-7mm}
\end{proof}
\vspace{-4mm}
\begin{proposition}\label{prop2}
Let $\Ab$ and $\Bb$ be two positive definite matrices such that $\Ab\mge \Bb$. Then, it holds that $\det(\Ab)\ge \det(\Bb)$.
\end{proposition}
\begin{proof}
Let $\Ab=\Lb\Lb^{\top}$ be the Cholesky factorization and note that, since $\Ab$ is \emph{strictly} positive definite, $\Lb$ has an inverse. Then,
it follows that
\begin{align*}
\Ab \mge \Bb\implies& \II\mge \Lb^{-1} \Bb \Lb^{-\top}\mge 0 \implies 1 > \text{eig}_i (  \Lb^{-1} \Bb \Lb^{-\top} )>0 \ \forall \  \text{eigenvalues }\text{eig}_i\\
\implies & 1> \prod_i  \text{eig}_i (  \Lb^{-1} \Bb \Lb^{-\top} ) \implies 1> \det (  \Lb^{-1} \Bb \Lb^{-\top} ) \\
\implies & 1> (1/\det(\Ab)) (\det (\Bb))
\end{align*}
which immediately gives the required result.
\end{proof}
\vspace{-6mm}
\begin{lemma}[Sherman-Morrison formula~\cite{sherman1950adjustment}]\label{lem:morrison}
Assume $\Ab$ is an invertible matrix. Then, the following equality holds:
\begin{equation}
(\Ab+\ub\vb^\top)^{-1}=\Ab ^{-1}-\frac{\Ab ^{-1}\ub\vb^\top \Ab ^{-1}}  {1+ \vb^\top \Ab^{-1} \ub}
\end{equation}
\end{lemma}
\vspace{-4mm}

\begin{proposition}\label{prop:XX1}
Assume $\Ab$ and $\Bb$ are invertible matrices. Then, the following equality holds:
\begin{equation}
(\Ab-\Bb)^{1}=\Ab ^{-1}+(\Ab \Bb^{-1} \Ab -\Ab)^{-1}
\end{equation}
\end{proposition}
\begin{proof}
We observe that,
$(\Ab \Bb^{-1} \Ab -\Ab)= (\Ab \Bb^{-1} \Ab -\Ab)\Ab^{-1} (\Ab-\Bb)+(\Ab-\Bb)$. Pre-multiply by $(\Ab \Bb^{-1} \Ab -\Ab)^{-1}$ and post-multiply by $(\Ab-\Bb)^{-1}$ on both sides to get the result.
\end{proof}
\vspace{-4mm}

\begin{proposition}\label{prop:tx}
 The function $t(x)=\frac{\log (x-a)}{\log x}$ is increasing for $x>a+1$.
\end{proposition}
\begin{proof}
 \begin{align}
  dt(x)/dx= \frac{x \log x- (x-a) \log (x-a) }{ x (x-a) (\log x)^2}>0
 \end{align}
\vspace{-4mm}
\end{proof}

\clearpage
\newpage
\section{Performance with the LR and NN models $h_{\theta}(\xb)$ on synthetic data} 
\label{app:syn-expts}
%
 \begin{figure}[!h]
\centering
\hspace*{0.5cm}{\includegraphics[width=0.55\textwidth]{legend_notbold}\label{fig:ltr}\vspace{-0.2cm}}\\
\subfloat[Gaussian, LR]{\includegraphics[width=0.22\textwidth]{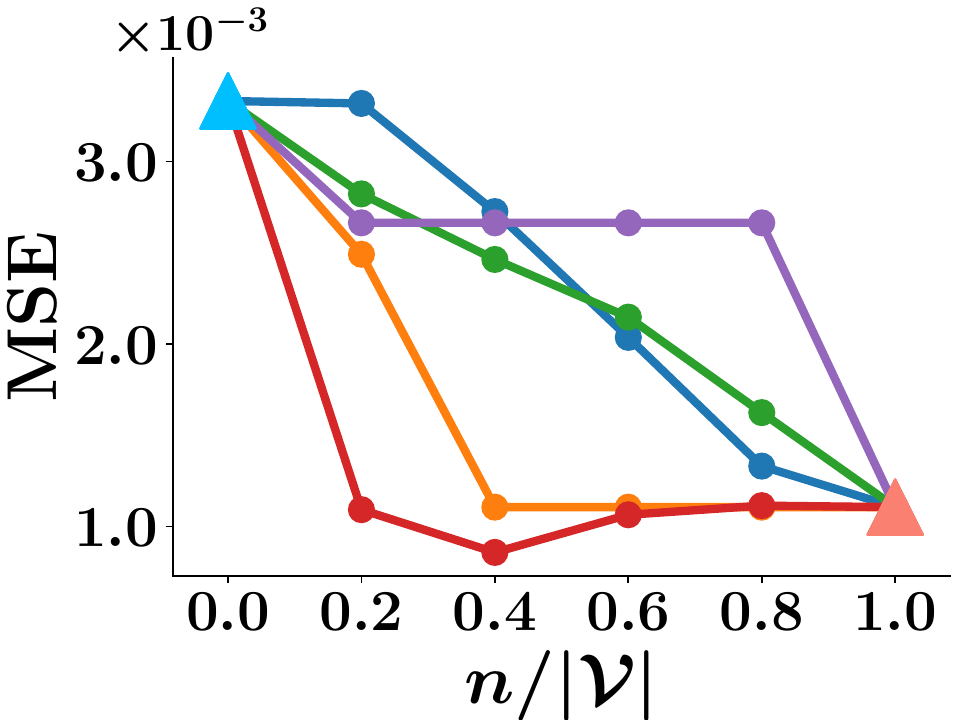}\label{fig:gtest}}\hspace*{0.5cm}
\subfloat[Logistic, LR]{\includegraphics[width=0.22\textwidth]{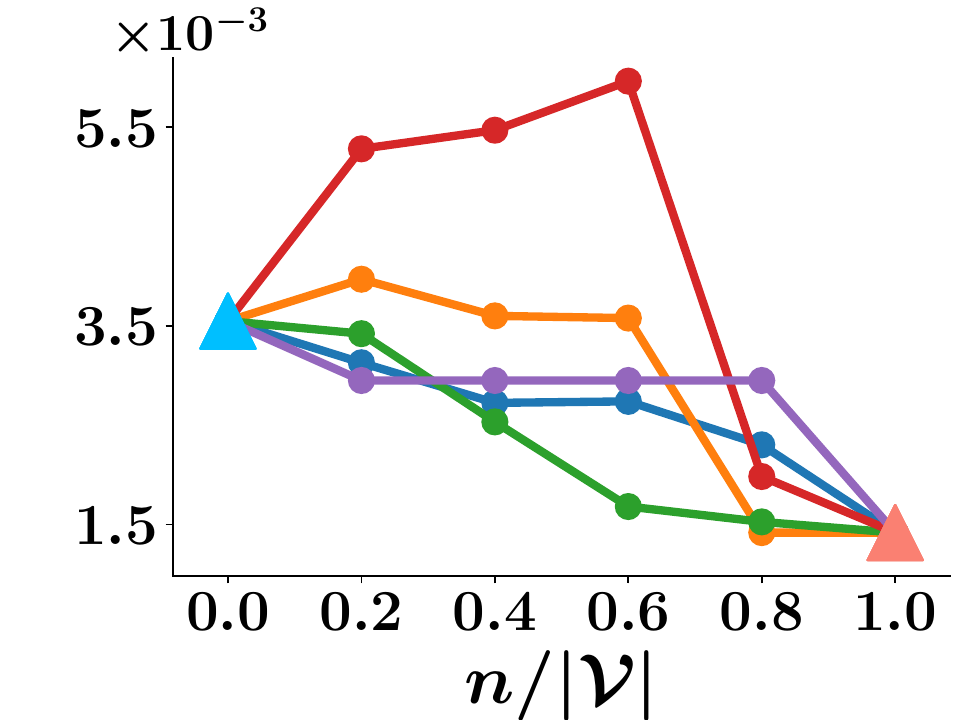}\label{fig:lte}} 
\subfloat[Gaussian, NN]{\includegraphics[width=0.22\textwidth]{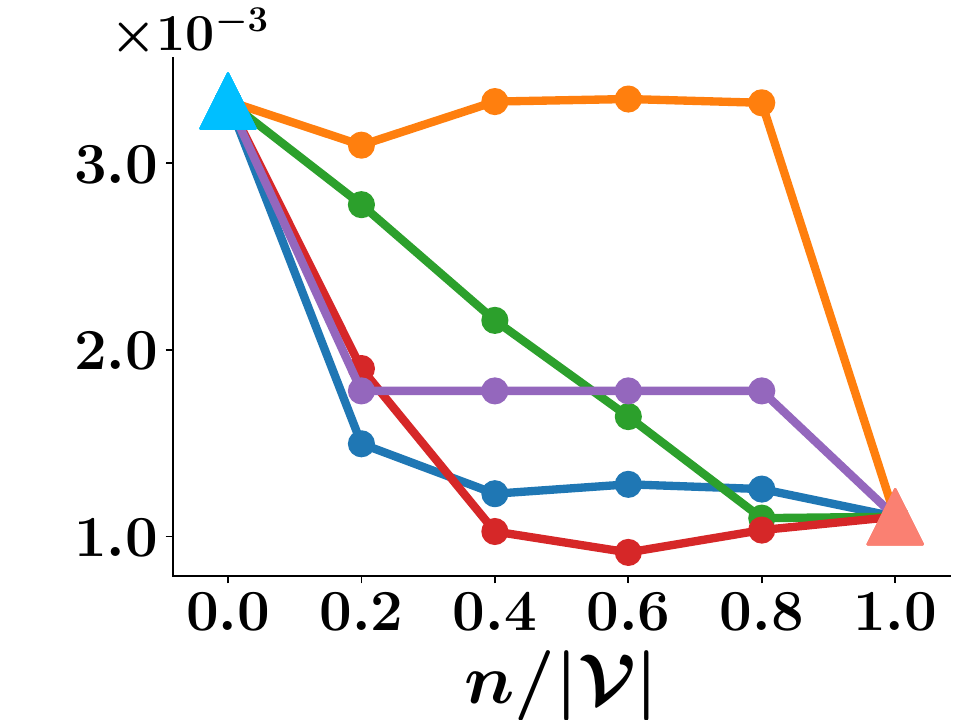}\label{fig:gtest}}\hspace*{0.5cm}
\subfloat[Logistic, NN]{\includegraphics[width=0.22\textwidth]{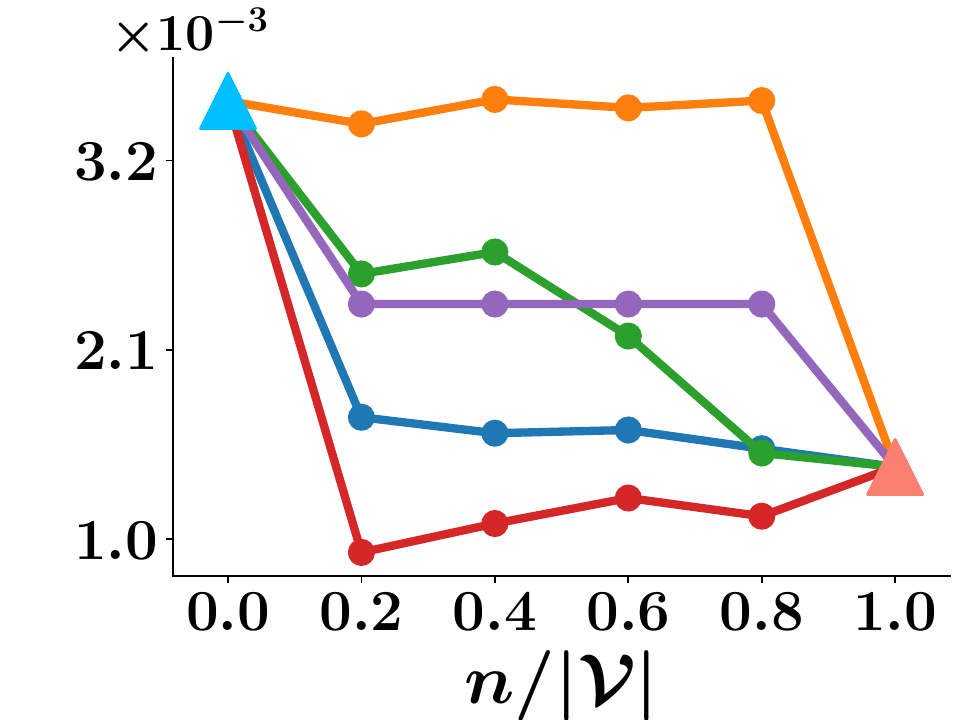}\label{fig:lte}} 
\caption{Mean squared error (MSE) against number of outsourced samples $n$ for the proposed greedy algorithm and four baselines with 
the LR and NN models $h_{\theta}(\xb)$ on synthetic data. 
In all cases, we used $d=5$, $\sigma_2=10^{-3}$ and $\lambda = 5 \cdot 10^{-3}$.
For most automation levels, the competitive advantage provided by the greedy algorithm is statistically significant (Welch'{}s t-test, $p$-value $= 10^{-3}$).
}
\label{fig:quantSyn-app-nn-add}
\end{figure} 
\section{Performance with the MLP and NN models $h_{\theta}(\xb)$ on real data}
\label{app:h-real}

\begin{figure}[!h]
\centering
\captionsetup[subfigure]{labelformat=empty}
\hspace*{0.5cm}{\includegraphics[width=0.55\textwidth]{legend_notbold}\label{fig:ltr}\vspace{-0.2cm}}\\
\subfloat{\includegraphics[width=0.20\textwidth]{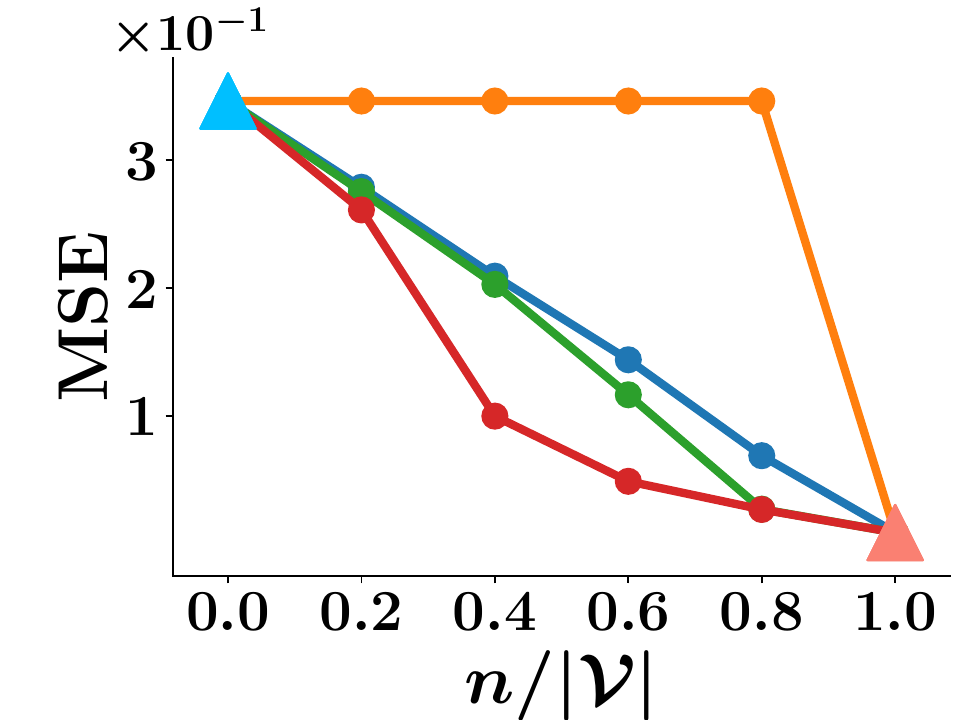}\label{fig:gtests}}\hspace*{0.2cm}
\subfloat{\includegraphics[width=0.20\textwidth]{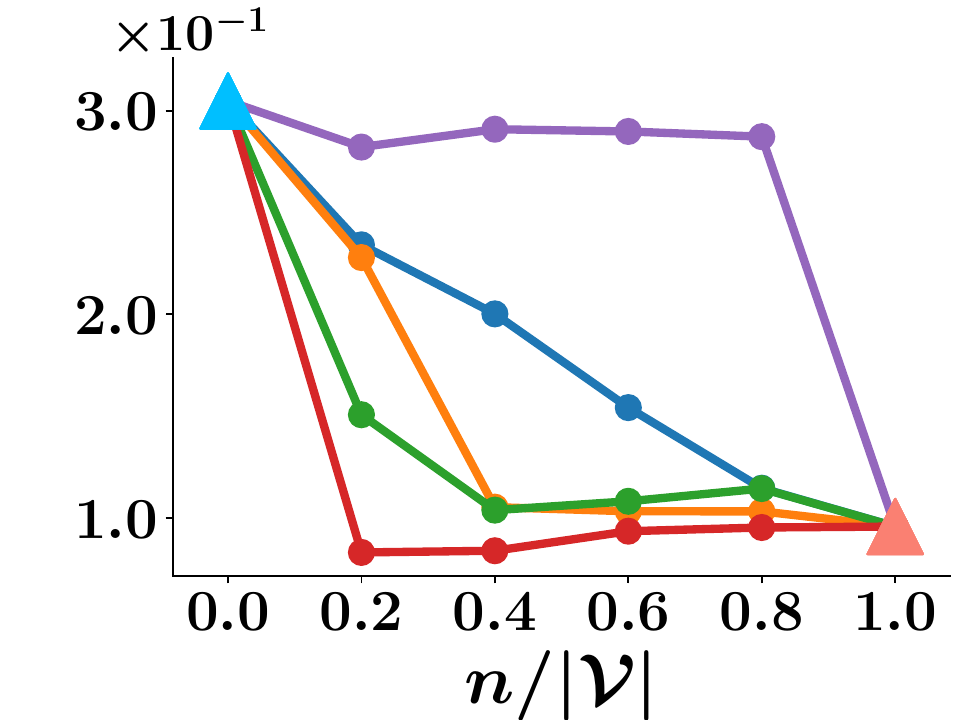}\label{fig:ltes1}} \hspace*{0.2cm}
\subfloat{\includegraphics[width=0.20\textwidth]{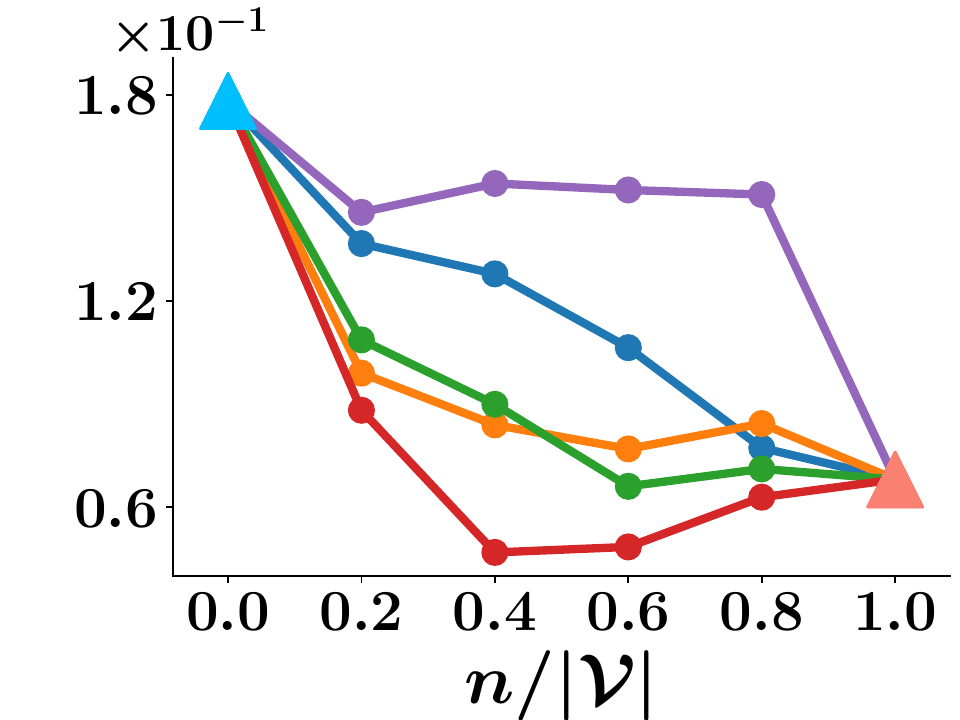}\label{fig:lte22}} \hspace*{0.2cm}
\subfloat{\includegraphics[width=0.20\textwidth]{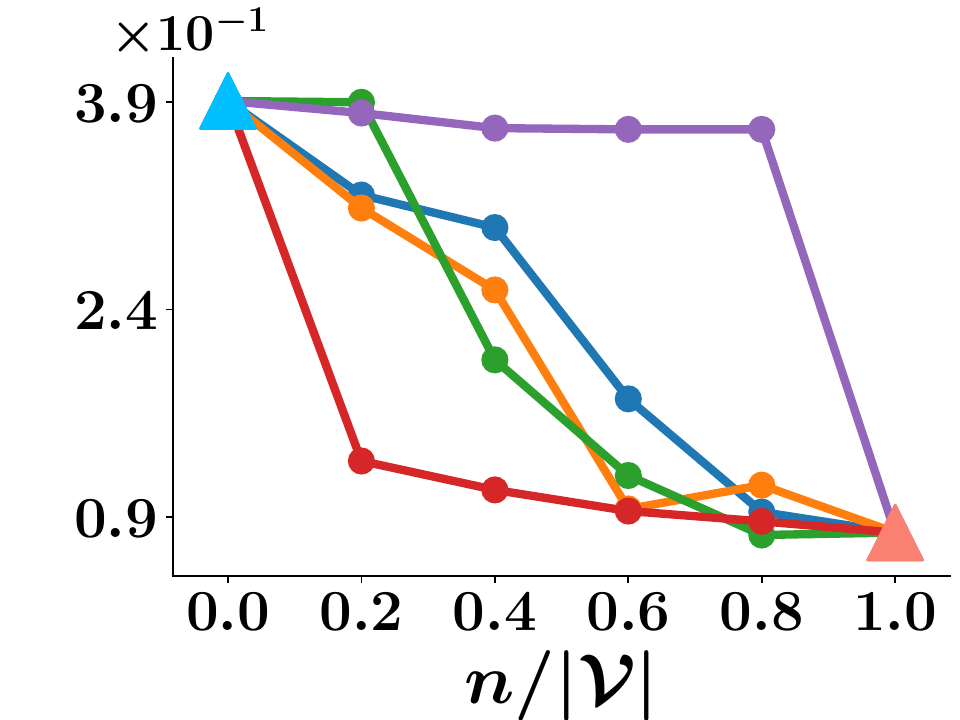}\label{fig:lte23}} \\
\vspace{-3mm}
\subfloat[Hatespeech]{\includegraphics[width=0.20\textwidth]{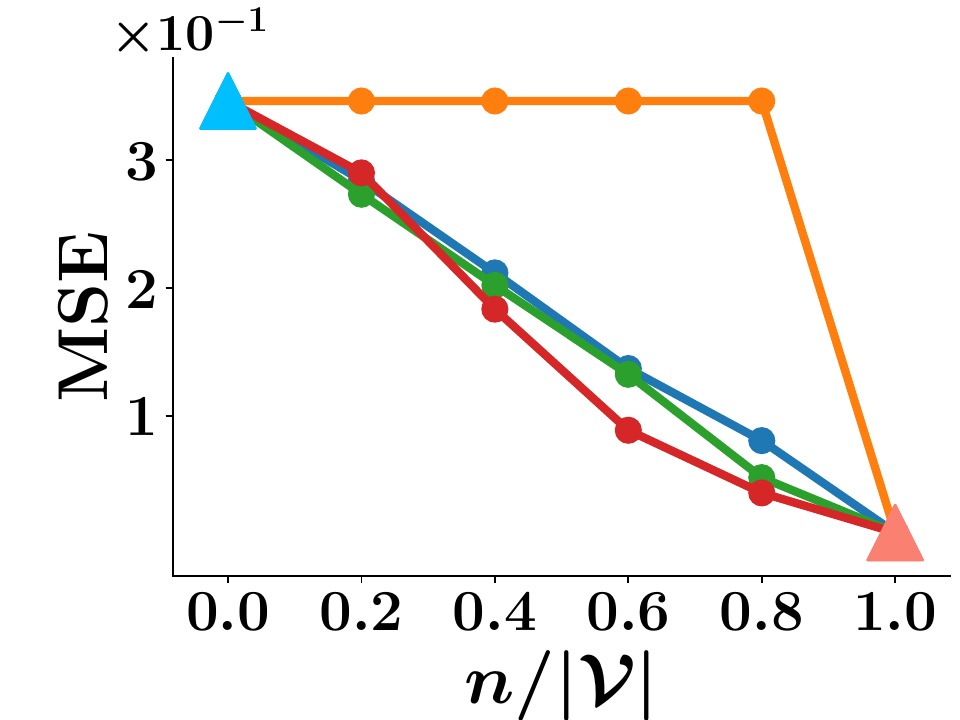}\label{fig:gtests}}\hspace*{0.2cm}
\subfloat[Stare-H]{\includegraphics[width=0.20\textwidth]{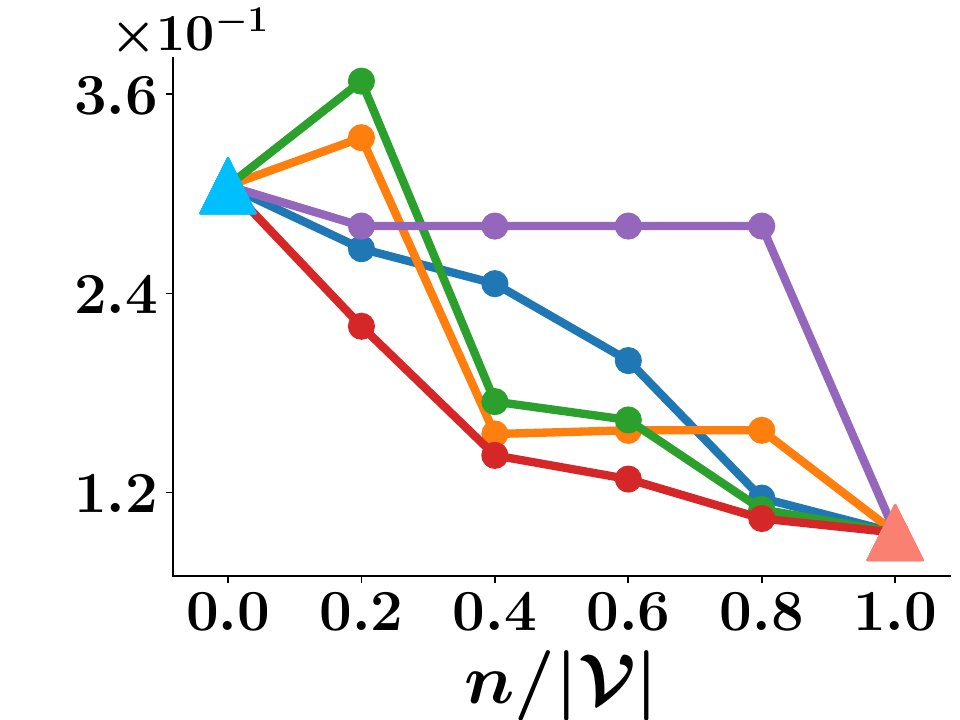}\label{fig:ltes1}} \hspace*{0.2cm}
\subfloat[Stare-D]{\includegraphics[width=0.20\textwidth]{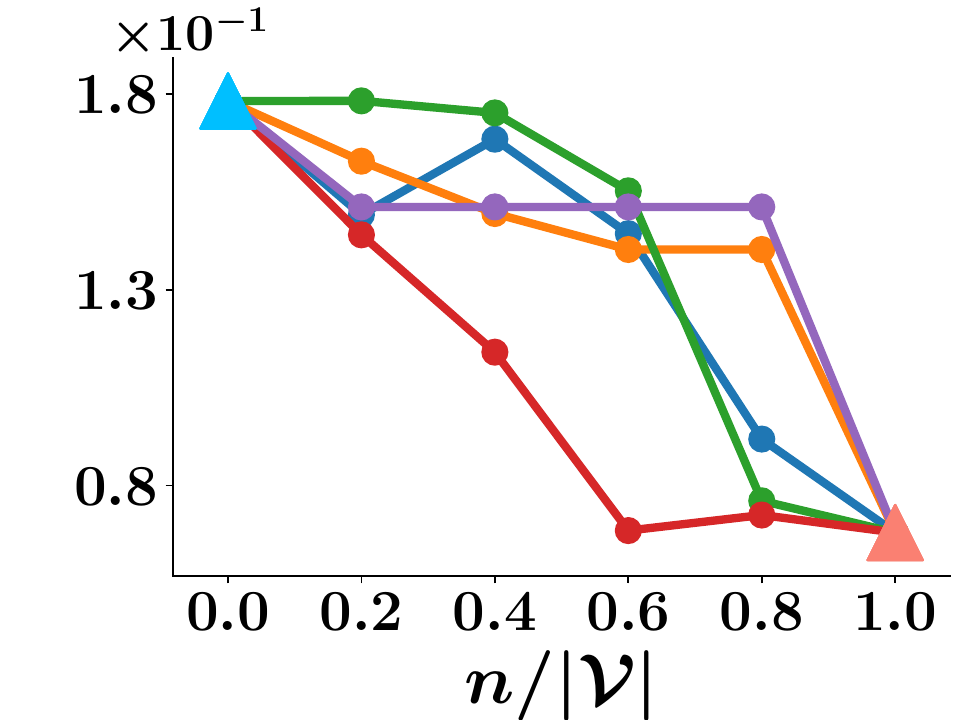}\label{fig:lte22}} \hspace*{0.2cm}
\subfloat[Messidor]{\includegraphics[width=0.20\textwidth]{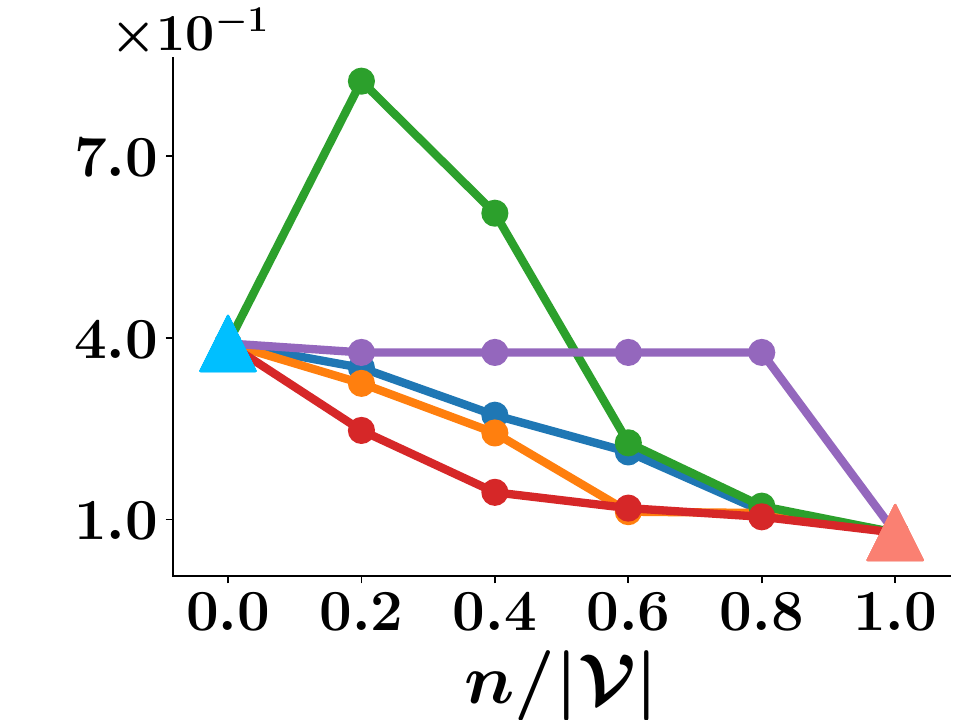}\label{fig:lte23}} 
\caption{Mean squared error (MSE) against number of outsourced samples $n$ for the proposed greedy algorithm and four baselines with
the MLP model (first row) and NN model (second row) $h_{\theta}(\xb)$ on four real-world datasets.
In Stare-H, Stare-D and Messidor, for each sample $(\xb, y)$, the element of $\alphab_{\xb, y}$ corresponding to the 
score $s = y$ has the highest value, similarly as in Figure~\ref{fig:hmodels-real}.
For most automation levels, the competitive advantage provided by the greedy
algorithm is statistically significant (Welch'{}s t-test, $p$-value $= 10^{-3}$).
%
}
 \label{fig:quantrealone-app-nn}
\end{figure}
 \newpage
\section{Human and machine error for test samples in real data} 
\label{app:expl-real}
\begin{figure}[!h]
\centering
\hspace*{0.5cm}{\includegraphics[width=0.5\textwidth]{subsets}\label{fig:ltr}\vspace{-0.2cm}}\\
\subfloat{\includegraphics[width=0.22\textwidth]{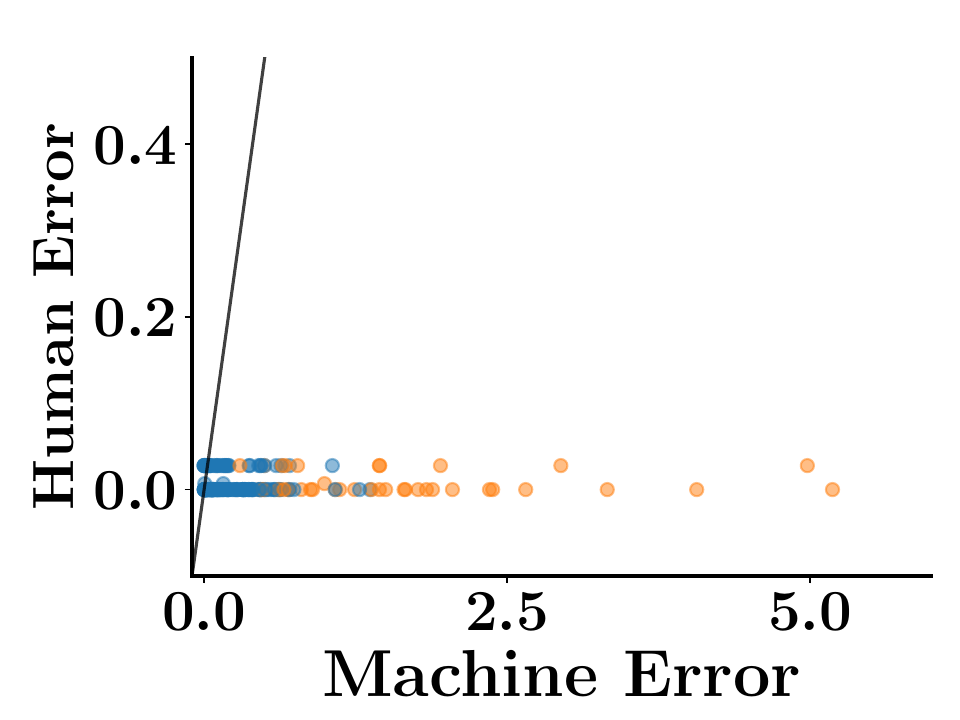}\label{fig:gtests}}\hspace*{0.2cm}
\subfloat{\includegraphics[width=0.22\textwidth]{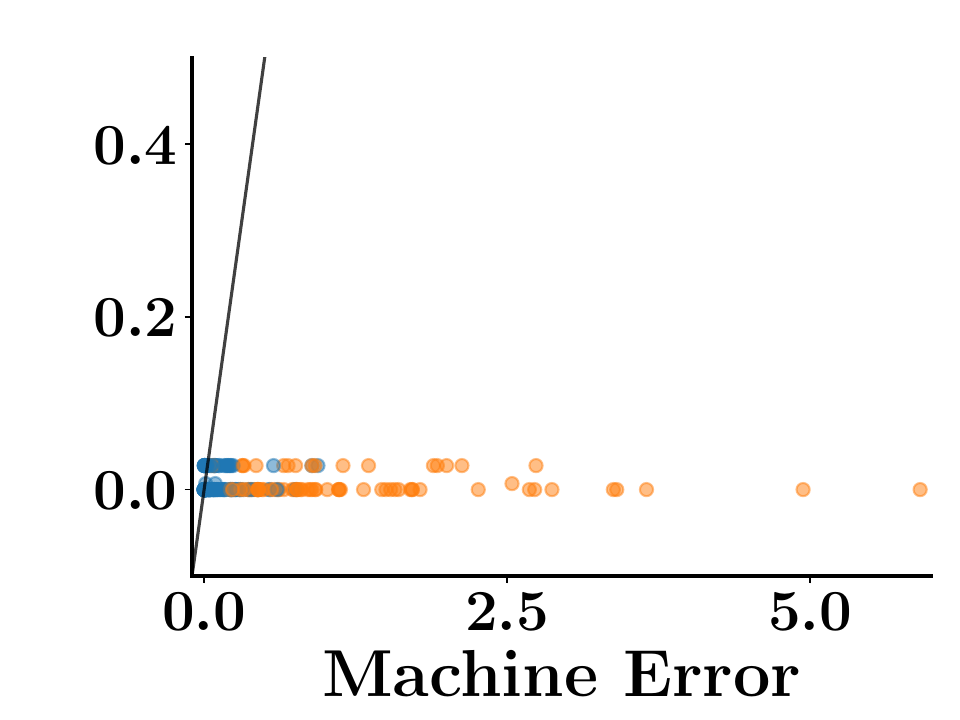}\label{fig:ltes1}} \hspace*{0.2cm}
\subfloat{\includegraphics[width=0.22\textwidth]{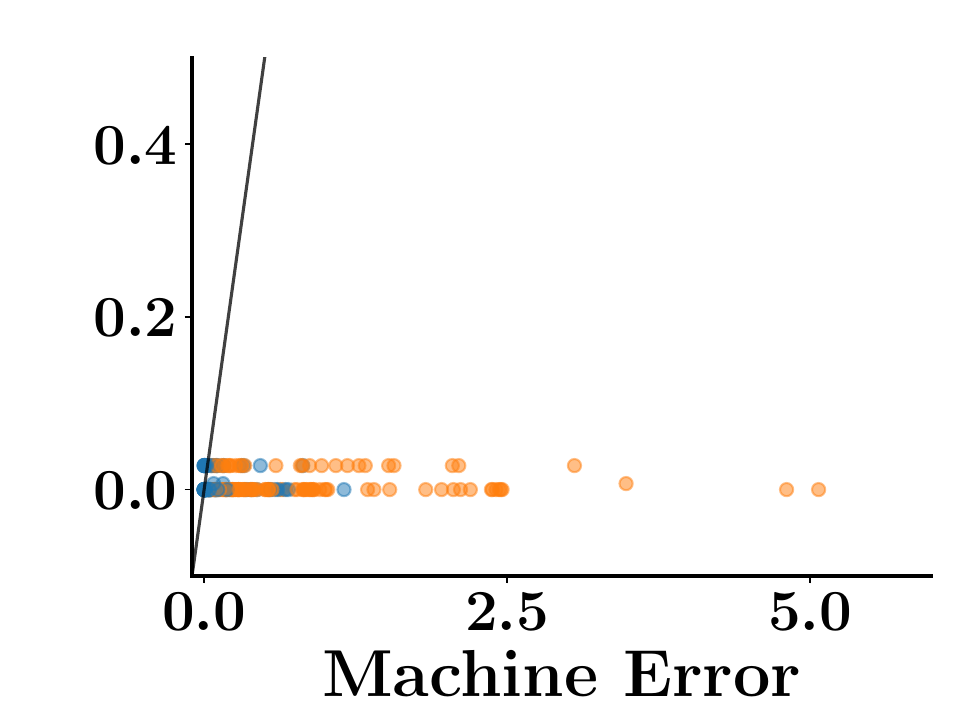}\label{fig:lte22}} \hspace*{0.2cm}
\subfloat{\includegraphics[width=0.22\textwidth]{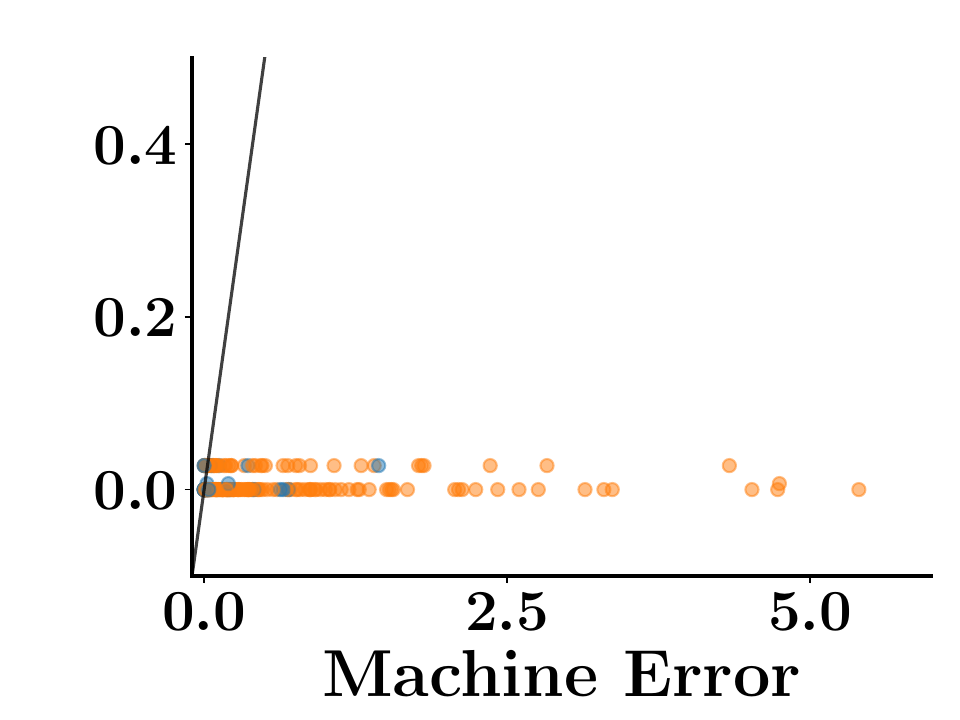}\label{fig:lte23}} \\
\subfloat{\includegraphics[width=0.22\textwidth]{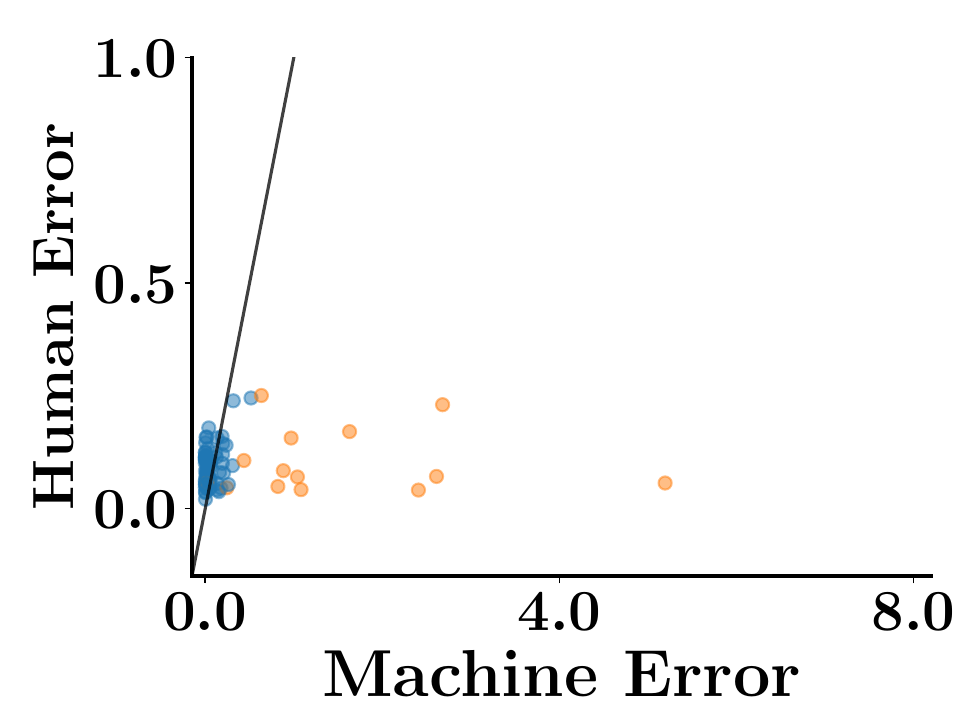}\label{fig:gtests}}\hspace*{0.2cm}
\subfloat{\includegraphics[width=0.22\textwidth]{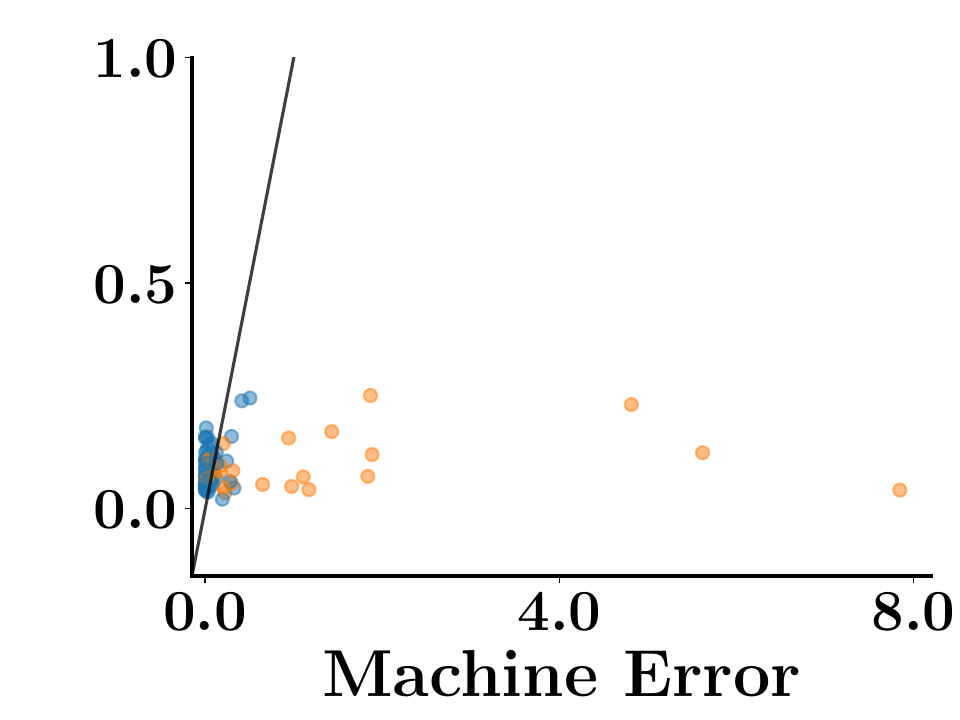}\label{fig:ltes1}} \hspace*{0.2cm}
\subfloat{\includegraphics[width=0.22\textwidth]{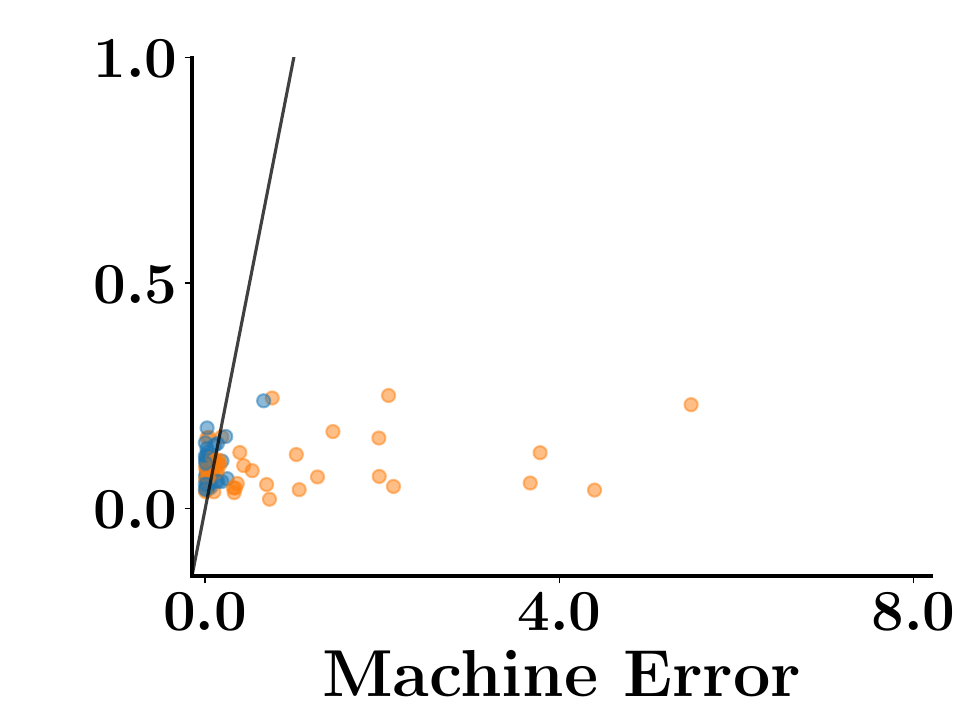}\label{fig:lte22}} \hspace*{0.2cm}
\subfloat{\includegraphics[width=0.22\textwidth]{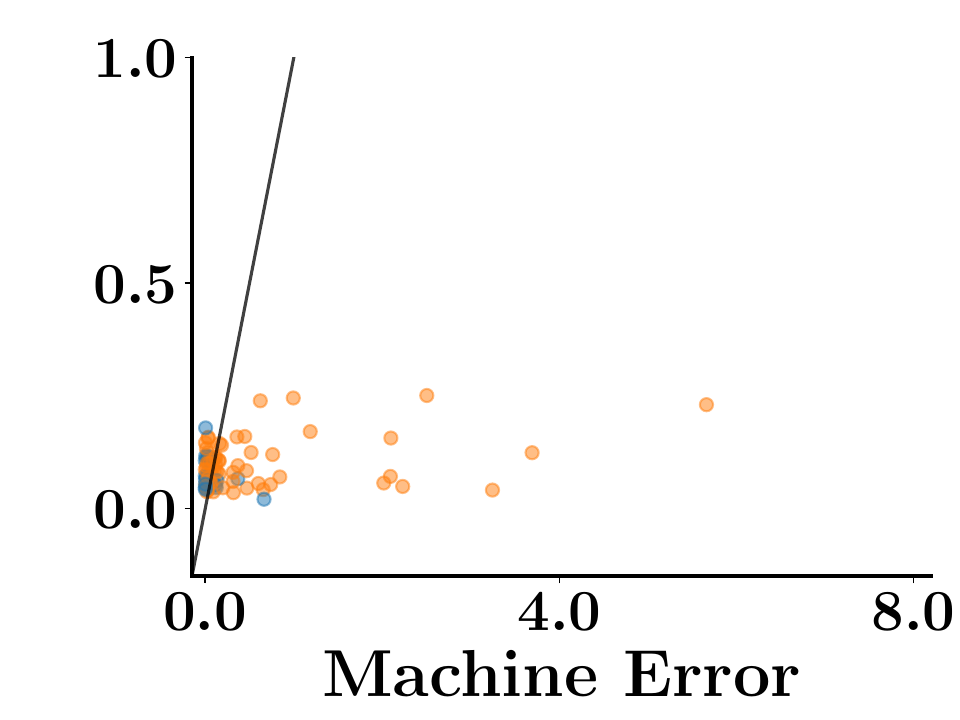}\label{fig:lte23}} \\
\subfloat{\includegraphics[width=0.22\textwidth]{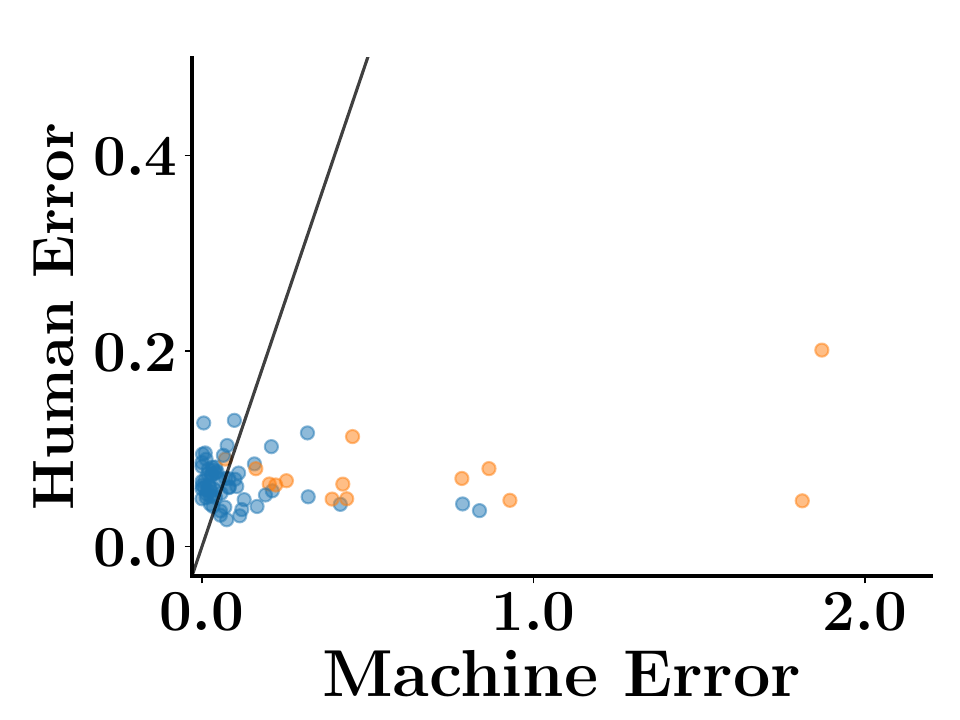}\label{fig:gtests}}\hspace*{0.2cm}
\subfloat{\includegraphics[width=0.22\textwidth]{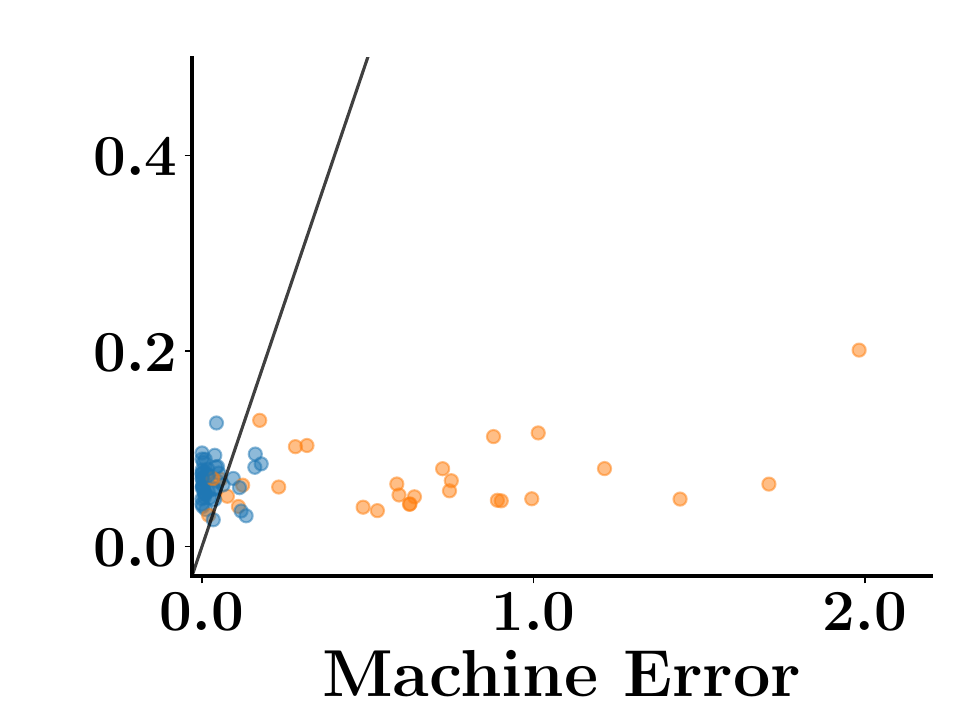}\label{fig:ltes1}} \hspace*{0.2cm}
\subfloat{\includegraphics[width=0.22\textwidth]{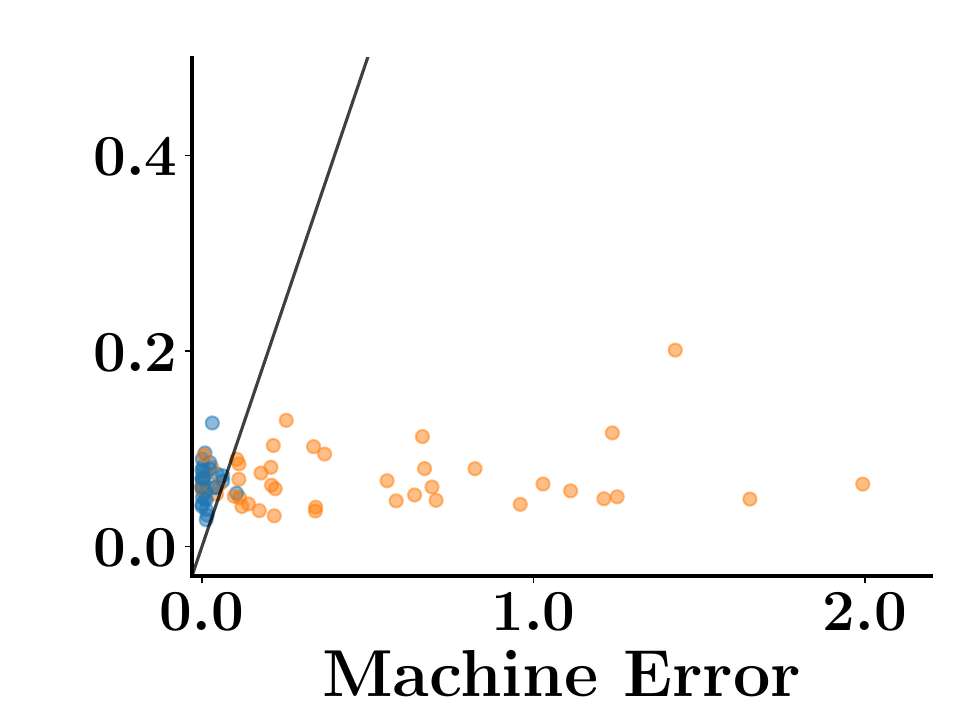}\label{fig:lte22}} \hspace*{0.2cm}
\subfloat{\includegraphics[width=0.22\textwidth]{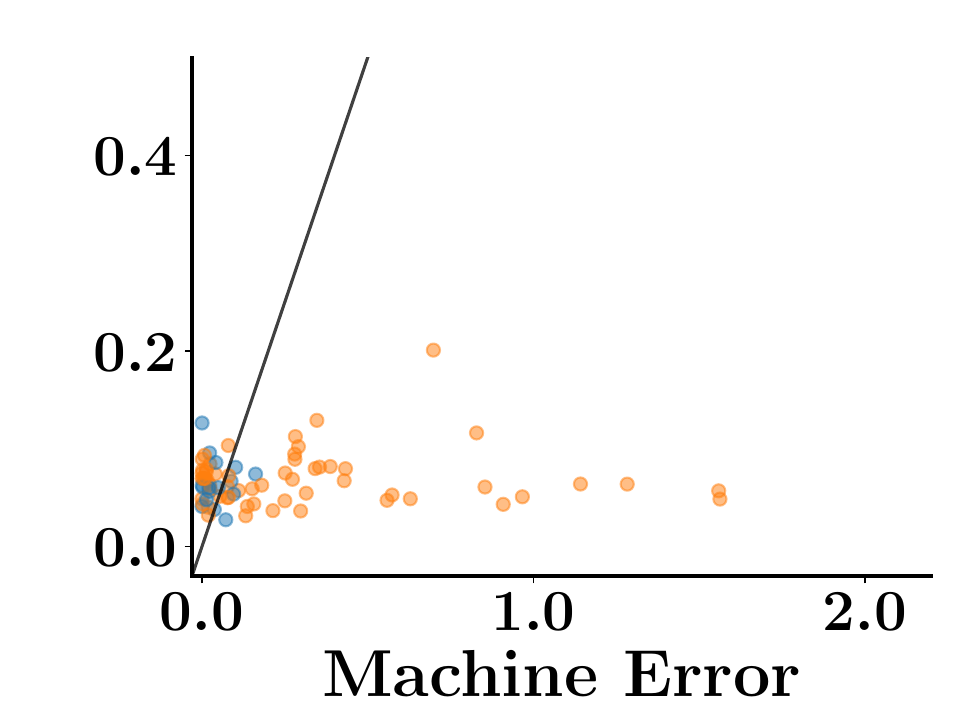}\label{fig:lte23}}\\ 
\subfloat[$|\Scal|=0.2|\Vcal|$]{\setcounter{subfigure}{1} \includegraphics[width=0.22\textwidth]{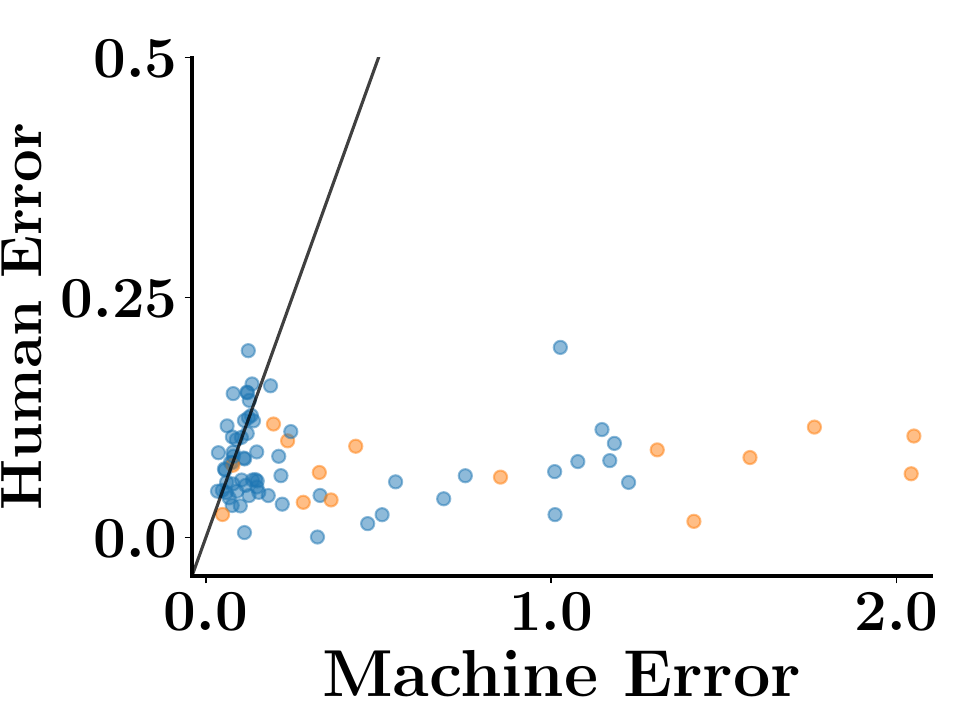}\label{fig:gtests}}\hspace*{0.2cm}
\subfloat[$|\Scal|=0.4|\Vcal|$]{\includegraphics[width=0.22\textwidth]{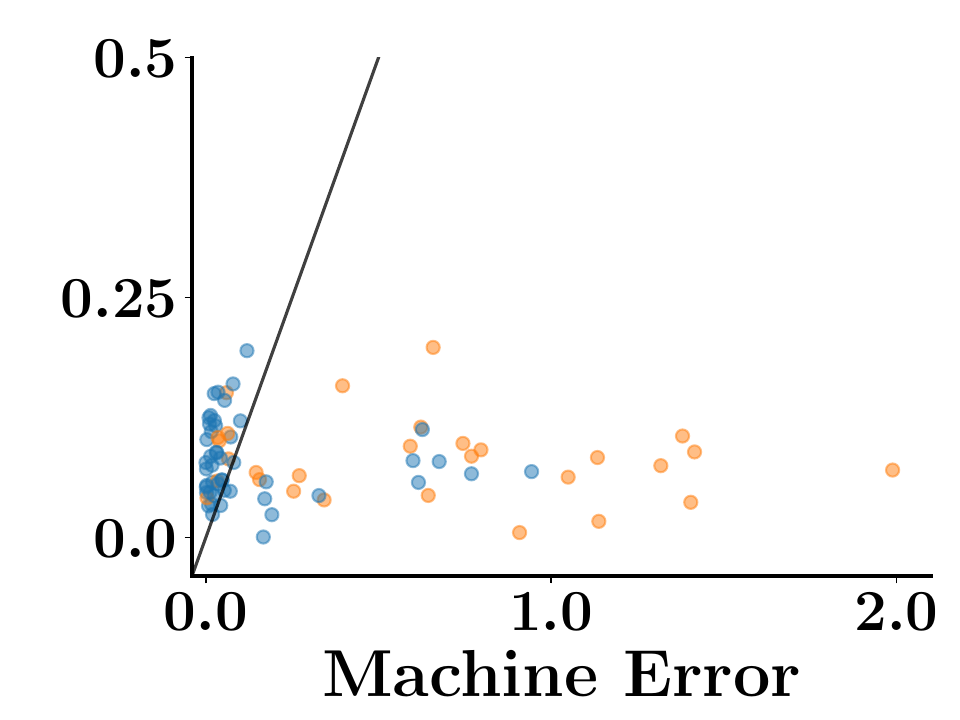}\label{fig:ltes1}} \hspace*{0.2cm}
\subfloat[$|\Scal|=0.6|\Vcal|$]{\includegraphics[width=0.22\textwidth]{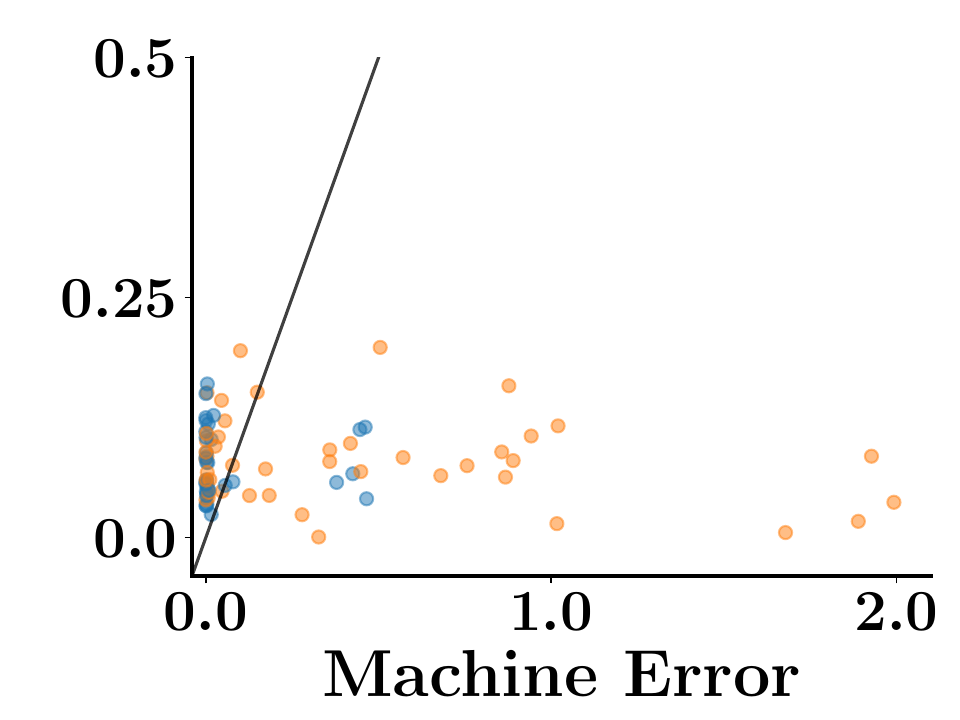}\label{fig:lte22}} \hspace*{0.2cm}
\subfloat[$|\Scal|=0.8|\Vcal|$]{\includegraphics[width=0.22\textwidth]{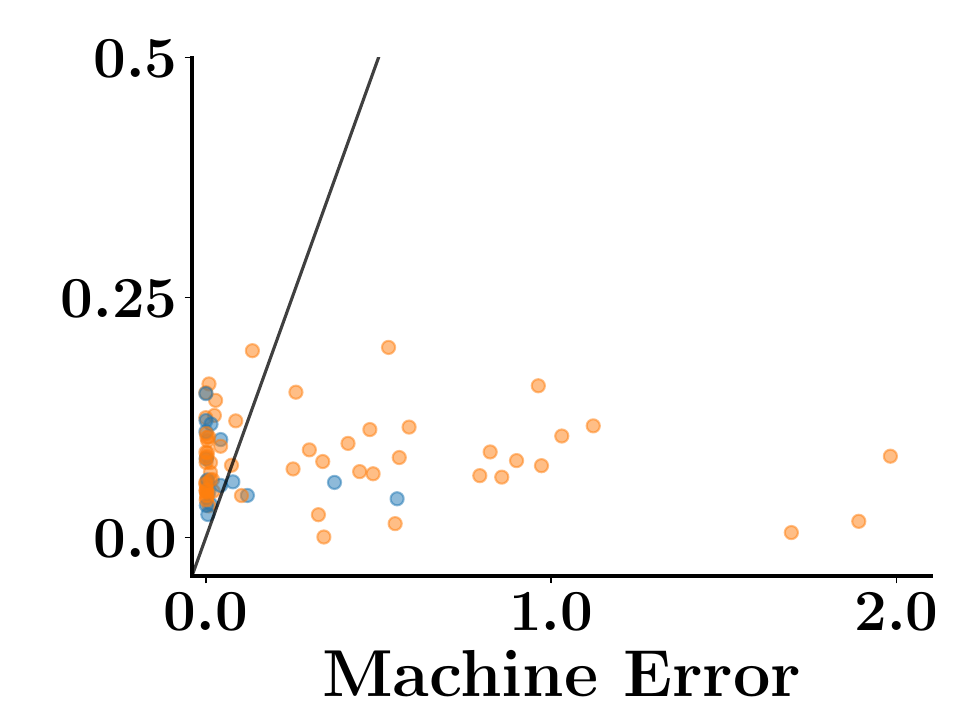}}
\caption{Human and machine error of all test samples on the Hatespeech (first row), Stare-H (second row), Stare-D (third row) and Messidor (third row) datasets 
for different automation levels.
Our algorithm outsources to humans those samples in which the machine error would have been the highest if it had to predict their response variables
}
\label{fig:scatter-human-machine-test}
\end{figure}

\end{document}